\documentclass{tlp}

\newif\ifdraft\drafttrue
\newif\ifinlineref\inlinereffalse
\newif\iffinal\finalfalse
\newif\ifextended\extendedfalse
\newif\ifdotikz\dotikzfalse
\newif\ifshowotherappendix\showotherappendixtrue
\newif\ifmakeallproofsinline\makeallproofsinlinefalse
\newif\ifrevisionmarkers\revisionmarkerstrue

\inlinereftrue
\draftfalse
\revisionmarkersfalse
\extendedtrue

\ifdraft
\fi

\usepackage{times}
\usepackage[scaled]{helvet}
\usepackage{courier}

\usepackage{enumerate}
\usepackage{amsmath}
\usepackage{amssymb}
\usepackage{url}
\usepackage{xspace}
\usepackage{marvosym}
\usepackage{xcolor}
\usepackage{paralist}
\usepackage{graphicx}
\usepackage{subfig}
\usepackage{multicol}
\usepackage{listings}

\usepackage{ifpdf}
\ifpdf
\usepackage{microtype}
\fi

	\newcommand{\qed}[0]{}
	\newcommand{\qedhere}[0]{}

\ifdraft
\usepackage{pdfsync}
\usepackage{srcltx}
\else
\fi

\usepackage[lined,algoruled]{algorithm2e}

\ifdraft
\newcommand{\comment}[1]{{\small\bf\color{blue} *** #1 ***}}
\else
\newcommand{\comment}[1]{}
\fi

\ifdotikz
\usepackage{tikz}
\pgfrealjobname{hexeval}
\usetikzlibrary{shapes,arrows,backgrounds,chains,%
matrix,patterns,arrows,decorations.pathmorphing,decorations.pathreplacing,%
positioning,fit,calc,decorations.text,shadows%
}
\else
\long\def\beginpgfgraphicnamed#1#2\endpgfgraphicnamed{\includegraphics{#1}}
\fi

\newcommand{\nop}[1]{#1}
\renewcommand{\nop}[1]{} %

\ifmakeallproofsinline
\newcommand{\myinlineproof}[1]{#1}
\newcommand{\mylocatedproof}[1]{}
\else
\newcommand{\myinlineproof}[1]{}
\newcommand{\mylocatedproof}[1]{#1}
\fi

\ifrevisionmarkers
\newcommand{\reva}[1]{{\color{blue} #1}}
\newcommand{\revam}[1]{{\color{red} #1}}
\newcommand{\revanop}[1]{}
\else
\newcommand{\reva}[1]{{#1}}
\newcommand{\revam}[1]{}
\newcommand{\revanop}[1]{}
\fi

\newtheorem{theorem}{Theorem}
\newtheorem{proposition}{Proposition}

\newtheorem{definition}{Definition}
\newtheorem{example}{Example}
\newtheorem{property}{Property}

\newenvironment{myitemize}{\begin{list}{$\bullet$}{%
\setlength{\topmargin}{0pt}
\setlength{\leftmargin}{0pt}
\setlength{\itemindent}{10pt}}
}{\end{list}}

\newcounter{myenumeratecounter}

\renewcommand{\vec}[1]{\ensuremath{\mb{#1}}}

\newcommand{\mi}[1]{\ensuremath{\mathit{#1}}}
\newcommand{\mb}[1]{\ensuremath{\mathbf{#1}}}
\def\lif{\ensuremath{\leftarrow}}

\def\naf{\ensuremath{\mathop{not}}}
\def\bbN{\ensuremath{{\mathbb{N}}}}

\def\cI{\ensuremath{{\mathcal{I}}}}
\def\cA{\ensuremath{{\mathcal{A}}}}
\def\cE{\ensuremath{{\mathcal{E}}}}

\def\cS{\ensuremath{{\mathcal{S}}}}
\def\scI{\ensuremath{{\textsc{i}}}}
\def\scO{\ensuremath{{\textsc{o}}}}

\def\cG{\ensuremath{{\cal G}}}

\def\cC{\ensuremath{{\cal C}}}
\def\cX{\ensuremath{{\cal X}}}
\def\cG{\ensuremath{{\cal G}}}
\newcommand\bi{\begin{itemize}}
\newcommand\ei{\end{itemize}}
\newcommand\quo[1]{`#1'}
\newcommand{\ins}{\,{\in}\,}

\newcommand{\les}{\,{\le}\,}

\newcommand{\eqs}{\,{=}\,}
\newcommand{\neqs}{\,{\neq}\,}
\newcommand{\setminuss}{\,{\setminus}\,}
\newcommand{\modelss}{\,{\models}\,}

\newcommand{\cups}{\,{\cup}\,}
\newcommand{\caps}{\,{\cap}\,}
\newcommand{\mids}{\,{\mid}\,}
\newcommand{\timess}{\,{\times}\,}

\newcommand{\lors}{\,{\vee}\,}

\newcommand{\subsets}{\,{\subset}\,}
\newcommand{\subseteqs}{\,{\subseteq}\,}

\newcommand\hex{{\sc hex}\xspace}
\newcommand\dlv{{\small\sffamily dlv}\xspace}
\newcommand\dlvhex{{\small\sffamily dlvhex}\xspace}
\newcommand\dlvex{{\small\sffamily dlv-ex}\xspace}

\def\HBP{\mi{H\!B}_P}

\def\fP{fP}
\newcommand{\AS}{\mathcal{A\!S}}
\newcommand{\dq}{\ifmmode\text{"}\else"\fi}
\newcommand{\lifs}{\,{\lif}\,}
\newcommand{\cons}{\ensuremath{\mi{const}}}

\newcommand\heurold{\textsl{H1}\xspace}
\newcommand\heurnew{\textsl{H2}\xspace}
\newcommand\heurtriv{\textsl{H0}\xspace}
\newcommand\heurmrg{\textsl{H3}\xspace}

\newcommand{\mycuri}[0]{\ensuremath{\mi{cur\scI}}}
\newcommand{\mycuro}[0]{\ensuremath{\mi{cur\scO}}}
\newcommand{\myrefcounto}[0]{\ensuremath{\mi{refs\scO}}}

\newcommand\clasp{{\small\sffamily clasp}\xspace}

\newcommand\gringo{{\small\sffamily gringo}\xspace}

\newcommand\clingcon{{\small\sffamily clingcon}\xspace}
\newcommand\ezcsp{{\small\textsc ezcsp}\xspace}
\newcommand\acsolver{{\small\sffamily ACsolver}\xspace}

\newcommand{\amp}[1]{\ensuremath{\text{\textsl{{\&}}}\!\mathit{#1}}}
\newcommand{\ext}[3]{\ensuremath{\amp{#1}[#2](#3)}}
\newcommand{\extfun}[1]{\ensuremath{f_{\text{\sl\&}#1}}}
\newcommand{\extFun}[1]{\ensuremath{F_{\text{\sl\&}#1}}}

\newcommand{\GroundLiberallyDomainExpansionSafeProgram}{\ensuremath{\textsc{GroundHEX}}}
\newcommand{\EvaluateGroundHEX}{\ensuremath{\textsc{EvaluateGroundHEX}}}
\newcommand{\EvaluateLDESafe}{\ensuremath{\textsc{EvaluateLDESafe}}}
\newcommand{\acthex}{\textsc{Acthex}}
\newcommand{\grnd}{\ensuremath{\mathit{grnd}}}

\newcommand{\attr}[3]{#1{\upharpoonright}_{#2}#3}
\newcommand{\ipar}{\textsc{i}}
\newcommand{\opar}{\textsc{o}}

\def\extg{\amp{g}}
\newcommand{\inpconst}{\mb{const}}

\def\depends{\rightarrow}

\def\dependsneg{\rightarrow_n}
\def\dependsext{\rightarrow^e}
\def\dependsextmon{\rightarrow^e_m}
\def\dependsextnmon{\rightarrow^e_{\mi{\reva{n}}}}
\def\dependsmon{\rightarrow_m}
\def\dependsnmon{\rightarrow_n}

\def\dependsmn{\rightarrow_{m,n}}

\newcommand{\T}{\mathbf{T}}
\newcommand{\F}{\mathbf{F}}

\def\join{{\:\bowtie\:}}

\def\ufinal{u_{\mi{final}}}

\def\myigraph{i-graph\xspace}
\def\myigraphs{i-graphs\xspace}
\def\myiint{i-inter\-pre\-ta\-tion\xspace}
\def\myiints{i-inter\-pre\-ta\-tions\xspace}
\def\myoint{o-inter\-pre\-ta\-tion\xspace}
\def\myoints{o-inter\-pre\-ta\-tions\xspace}

\def\myimodels{\mi{i\text{-}ints}}
\def\myomodels{\mi{o\text{-}ints}}
\def\myimodelsI{\mi{i\text{-}ints}_{\cI}}
\def\myomodelsI{\mi{o\text{-}ints}_{\cI}}
\def\myimodelsA{\mi{i\text{-}ints}_{\cA}}
\def\myomodelsA{\mi{o\text{-}ints}_{\cA}}
\def\myunit{\mi{unit}}
\def\mytype{\mi{type}}
\def\myint{\mi{int}}
\newcommand{\myinputsE}[1]{\mi{preds}_{#1}}
\newcommand{\myinputs}{\myinputsE{\cE}}

\def\myBuildAnswerSets{\ensuremath{\textsc{Build}\-\textsc{Answer}\-\textsc{Sets}}\xspace}

\def\myGetNextIModel{\ensuremath{\textsc{Get}\-\textsc{Next}\-\textsc{Input}\-\textsc{Model}}\xspace}
\def\myGetNextOModel{\ensuremath{\textsc{Get}\-\textsc{Next}\-\textsc{Output}\-\textsc{Model}}\xspace}
\def\myEnsureModelIncrement{\ensuremath{\textsc{Ensure}\-\textsc{Model}\-\textsc{Increment}}\xspace}
\def\myOnDemandAS{\ensuremath{\textsc{Answer}\-\textsc{Sets}\-\textsc{On}\-\textsc{Demand}}\xspace}
\def\myNextAnswerSet{\ensuremath{\mathit{NextAnswerSet}}}
\def\undef{\ensuremath{\textsc{undef}}}
\def\mycautext{first ancestor intersection unit\xspace}
\def\mycaustext{first ancestor intersection units\xspace}
\def\myCAUtext{FAI\xspace}
\def\myCAUstext{FAIs\xspace}
\def\mycau{\mi{fai}}

\newcommand\facts[0]{\ensuremath{\mi{facts}}}

\newcommand{\citeNBYYB}[2]{\citeANP{#1} \citeyear{#1,#2}}

\def\papertitle{A model building framework for Answer Set Programming with external computations}
\def\shortpapertitle{Model building for ASP with external atoms}

\submitted{27 January 2015}
\revised{24 June 2015}
\accepted{28 June 2015}

	\title[\shortpapertitle]{\vspace*{-1em} \papertitle%
  \thanks{\reva{
    This article is a significant extension of
    \protect\cite{2011_pushing_efficient_evaluation_of_hex_programs_by_modular_decomposition} and parts of \protect\cite{ps2012}.
This work has been supported by the Austrian Science Fund (FWF) Grants
P24090 and P27730, and the Scientific and Technological Research Council of Turkey (TUBITAK) Grant 114E430.}
  }}

	\author[Eiter, Fink, Ianni, Krennwallner, Redl, and Sch\"{u}ller]{%
		Thomas Eiter, Michael Fink \\
		Institut f\"ur Informationssysteme, Technische Universit\"at Wien \\
    Favoritenstra\ss e 9-11, A-1040 Vienna, Austria \\
    \email{\{eiter,fink\}@kr.tuwien.ac.at}
		\and
		Giovambattista Ianni \\
		Dipartimento di Matematica, Cubo 30B, Universit\`{a} della Calabria\\
    87036 Rende (CS), Italy \\
    \email{ianni@mat.unical.it}
		\and
		Thomas Krennwallner, Christoph Redl \\
		Institut f\"ur Informationssysteme, Technische Universit\"at Wien \\
    Favoritenstra\ss e 9-11, A-1040 Vienna, Austria \\
    \email{\{tkren,redl\}@kr.tuwien.ac.at}
		\and
		Peter Sch\"{u}ller \\
		Computer Engineering Department, Faculty of Engineering,
		Marmara University \\
    Goztepe Kampusu, Kadikoy 34722, Istanbul, Turkey\\
    \email{peter.schuller@marmara.edu.tr}}

\ifdotikz
\makeatletter
\pgfdeclareshape{document}{
	\inheritsavedanchors[from=rectangle] %
	\inheritanchorborder[from=rectangle]
	\inheritanchor[from=rectangle]{center}
	\inheritanchor[from=rectangle]{north}
	\inheritanchor[from=rectangle]{south}
	\inheritanchor[from=rectangle]{west}
	\inheritanchor[from=rectangle]{east}
	\backgroundpath{%
		\southwest \pgf@xa=\pgf@x \pgf@ya=\pgf@y
		\northeast \pgf@xb=\pgf@x \pgf@yb=\pgf@y
		\pgf@xc=\pgf@xb \advance\pgf@xc by-10pt %
		\pgf@yc=\pgf@yb \advance\pgf@yc by-10pt
		\pgfpathmoveto{\pgfpoint{\pgf@xa}{\pgf@ya}}
		\pgfpathlineto{\pgfpoint{\pgf@xa}{\pgf@yb}}
		\pgfpathlineto{\pgfpoint{\pgf@xc}{\pgf@yb}}
		\pgfpathlineto{\pgfpoint{\pgf@xb}{\pgf@yc}}
		\pgfpathlineto{\pgfpoint{\pgf@xb}{\pgf@ya}}
		\pgfpathclose
		\pgfpathmoveto{\pgfpoint{\pgf@xc}{\pgf@yb}}
		\pgfpathlineto{\pgfpoint{\pgf@xc}{\pgf@yc}}
		\pgfpathlineto{\pgfpoint{\pgf@xb}{\pgf@yc}}
		\pgfpathlineto{\pgfpoint{\pgf@xc}{\pgf@yc}}
	}
}
\makeatother
\fi%

\newcommand\myfigureSystemArchitecture{
	\beginpgfgraphicnamed{systemarchitecture}
	\begin{tikzpicture}[%
    >=latex',
		remember picture,
		text centered,
		start chain,
		node distance=1cm,
		every on chain/.style={join=by ->},
		every join/.style={line width=1.25pt},
		sources/.style={
			shadow xshift=1ex,
			shadow yshift=1ex,
			cylinder,
			copy shadow,
			fill=white,
			shape border rotate=90,
			shape aspect=.1,
			line width=1.25pt,
			text width=1.5cm,
			minimum width=1cm,
			minimum height=11mm,
      outer ysep=1mm
		}
    ]
		\tikzstyle{unit} = [line width=1.0pt, auto]
		\tikzstyle{model} = [line width=1.0pt, auto, text width=2cm]
		\tikzstyle{line} = [draw, line width=1.25pt, join=by ->]
		\matrix (m) [matrix of nodes,
		column sep=7.2mm,
		row sep=12mm,
		inner sep=3pt,
		nodes={draw, %
		  line width=0.7pt,
		  anchor=center,
		  text centered
		},
		bases/.style={
			minimum width=1.6cm,
			text width=1.6cm,
			tape
		},
		doc/.style={
			tape,
			text width=1.5cm,
			minimum width=1.5cm,
			minimum height=9mm
		},
		docs/.style={
			tape,
			copy shadow,
			fill=white,
			text width=1.5cm,
			minimum width=2cm,
			minimum height=11mm
		},
		sets/.style={
			copy shadow,
			fill=white,
			text width=1.25cm,
			minimum width=1.5cm
		},
		subsystem/.style={
			line width=1.25pt,
			text width=1.9cm,
			minimum width=1cm,
			minimum height=9mm
		},
		system/.style={
			line width=1.25pt,
			text width=1.9cm,
			minimum width=2.1cm,
			minimum height=9mm
		},
		source/.style={
			shadow xshift=1ex,
			shadow yshift=1ex,
			cylinder,
			fill=white,
			shape border rotate=90,
			shape aspect=.1,
			line width=1.25pt,
			text width=1.5cm,
			minimum width=1cm,
			minimum height=11mm
		},
		hexprog/.style={
			document,
			fill=white,
			text width=1.6cm,
			minimum width=1.2cm,
			minimum height=12mm
		},
		]
		{
			|[hexprog]| \hex-Program	\pgfmatrixnextcell |[subsystem]| Evaluation Framework	\pgfmatrixnextcell \pgfmatrixnextcell |[docs]| Answer Sets \\
						\pgfmatrixnextcell |[subsystem]| Model Generators \pgfmatrixnextcell \pgfmatrixnextcell |[system]| ASP Solver   \\
			|[system]| ASP Grounder	\pgfmatrixnextcell |[subsystem]| \hex{}-Grounder		\pgfmatrixnextcell |[subsystem]| Post \mbox{Propagator} \\
		};
    \path ($(m-3-2)!0.5!(m-3-3)$) ++(0,-20mm)
      node[draw,line width=0.7pt,anchor=center,text centered,sources]
      (plugins) {Plugins};
		\draw[dashed] ($(m-1-2)+(-1.5cm,+0.7cm)$) rectangle
      node [yshift=3.35cm] {\dlvhex{} core}
      ($(m-3-3)+(1.5cm,-0.7cm)$);
		\path[draw,line width=1pt,->] (m-1-1) -- node [scale=0.5,shape=circle,draw,fill=white] {1} (m-1-2);
		\path[draw,line width=1pt,<->] (m-1-2) -- node [scale=0.5,shape=circle,draw,fill=white] {2} (m-2-2);
		\path[draw,line width=1pt,<->] (m-2-2) -- node [scale=0.5,shape=circle,draw,fill=white] {3} (m-3-2);
		\path[draw,line width=1pt,<->] (m-3-2) -- node [scale=0.5,shape=circle,draw,fill=white] {4} (m-3-1);
		\path[draw,line width=1pt,<->] (m-3-2) -- node [scale=0.5,shape=circle,draw,fill=white] {5} (plugins);
		\path[draw,line width=1pt,<->] (m-2-2) -- node [scale=0.5,shape=circle,draw,fill=white] {6} (m-2-4);
		\path[draw,line width=1pt] (m-2-4.south) edge[bend left=45,<->] node [scale=0.5,shape=circle,draw,fill=white] {7} (m-3-3.east);
		\path[draw,line width=1pt,<->] (m-3-3) -- node [scale=0.5,shape=circle,draw,fill=white] {8} (plugins);
		\path[draw,line width=1pt,->] (m-1-2) -- node [scale=0.5,shape=circle,draw,fill=white] {9} (m-1-4);
	\end{tikzpicture}
	\endpgfgraphicnamed
}
\newcommand\myfigureTikzAtomDepGraph{%
    \beginpgfgraphicnamed{atomdepgraph}%
    \small%
    \begin{tikzpicture}[line width=0.7pt,>=latex']
    \begin{scope}[xscale=3,yscale=-1.7]
      \node (swimin) at (1,1) {$\mygoinout(\myindoor)$};
      \node (swimout) at (2,1) {$\mygoinout(\myoutdoor)$};
      \node (rqswim) at (1,2) {$\ext{\mycost}{\mygoinout}{C}$};
      \node (swimp) at (2,2) {$\mygoinout(P)$};
      \node (needinout) at (0.5,3) {$\myneed(\myinout,C)$};
      \node (gotox) at (1.5,3) {$\mygolocation(X)$};
      \node (ngotox) at (2.5,3) {$\myngolocation(X)$};
      \node (gotoy) at (2.5,4.5) {$\mygolocation(Y)$};
      \node (go) at (2,4.5) {$\mygosomewhere$};
      \node (rqgoto) at (1.25,3.75) {$\ext{\mycost}{\mygolocation}{C}$};
      \node (needloc) at (1.25,4.5) {$\myneed(\reva{\myloc},C)$};
      \node (needmoney) at (0.5,5.0) {$\myneed(X,\mymoney)$};
    \end{scope}
    \begin{scope}
      \draw (swimin.south east) edge[bend right,->] node[midway,above] {$_m$} (swimout.south west);
      \draw (swimout.north west) edge[bend right,->] node[midway,above] {$_m$} (swimin.north east);
      \draw (rqswim.north) edge[->] node[midway,left] {$^e_m\!$} (swimin);
      \draw (rqswim.north east) edge[->] node[near end,below] {$^e_m$} (swimout);
      \draw (swimp.north) edge[->] node[midway,right] {$_m$} (swimout);
      \draw (swimp.north west) edge[->] node[near end,below] {$_m$} (swimin);
      \draw (needinout.north) edge[->] node[midway,left] {$_m$} (rqswim);
      \draw (gotox.north) edge[->] node[midway,left] {$_m$} (swimp);
      \draw (ngotox.north) edge[->] node[midway,left] {$_m$} (swimp);
      \draw (gotox.south east) edge[bend right,->] node[midway,above] {$_m$} (ngotox.south west);
      \draw (ngotox.north west) edge[bend right,->] node[midway,above] {$_m$} (gotox.north east);
      \draw (gotox.north west) edge[bend left=-115,distance=10mm,->] node[midway,right] {$_m$} (gotox.south west);

      \draw (rqgoto.north) edge[->] node[midway,left] {$^e_m$} (gotox);

      \draw (needloc.north) edge[->] node[midway,left] {$_m$} (rqgoto);
      \draw (gotoy.north) edge[->] node[midway,right] {$_m$} (gotox);
      \draw (go.north) edge[->] node[midway,right] {$_m$} (gotox);

      \draw (needmoney) edge[->] node[midway,left] {$_m$} (needinout);
      \draw (needmoney) edge[->] node[near end,below] {~$_m$} (needloc);
    \end{scope}
    \end{tikzpicture}%
    \endpgfgraphicnamed%
}%
\newcommand\myfigureTikzRuleDepGraph{%
    \beginpgfgraphicnamed{ruledepgraph}%
    \small%
    \begin{tikzpicture}[line width=0.7pt,>=latex']
    \begin{scope}[xscale=3,yscale=-1.5]
      \node (r1) at (2,1)
        {$r_1{:}\ \mygoinout(\myindoor) \lors \mygoinout(\myoutdoor) \lif$};
      \node (r2) at (1,2)
        {$\begin{array}{@{}l@{}}
            r_2{:}\ \myneed(\myinout,C) \lifs \\
            \qquad\quad\ext{\mycost}{\mygoinout}{C}
          \end{array}$};
      \node (r3) at (3,2)
        {$\begin{array}{@{}l@{}}
            r_3{:}\ \mygolocation(X) \lors \myngolocation(X) \lifs\\
            \qquad\quad\mygoinout(P), \mylocation(P,X)
          \end{array}$};
      \node (r4) at (3.75,3)
        {$r_4{:}\ \mygosomewhere \lifs \mygolocation(X)$};
      \node (r5) at (1.75,3)
        {$\begin{array}{@{}l@{}}
            r_5{:}\ \myneed(\myloc,C) \lifs \\
            \qquad\quad\ext{\mycost}{\mygolocation}{C}
          \end{array}$};
      \node (c6) at (2.5,4)
        {$c_6{:}\ \lifs \mygolocation(X), \mygolocation(Y), X\,{\neq}\,Y$};
      \node (c7) at (3.75,4)
        {$c_7{:}\ \lifs \naf \mygosomewhere$};
      \node (c8) at (1,4)
        {$c_8{:}\ \lifs \myneed(X,\mymoney)$};
    \end{scope}
    \begin{scope}
      \draw (r2) edge[->] node[midway,above] {$_m$} (r1);
      \draw (r3) edge[->] node[midway,above] {$_m$} (r1);

      \draw (r5) edge[->] node[midway,above] {$_m$} (r3);
      \draw (r4) edge[->] node[midway,right] {$\ _m$} (r3);

      \draw (c6) edge[->] node[midway,left] {$_m$} (r3);
      \draw (c7) edge[->] node[midway,right] {$_{\reva{n}}$} (r4);
      \draw (c8) edge[->] node[midway,left] {$_m$} (r2);
      \draw (c8) edge[->] node[midway,right] {$\ _m$} (r5);
    \end{scope}
    \end{tikzpicture}%
    \endpgfgraphicnamed%
}%
\newcommand\myfigureExOldEvalStrat{%
    \beginpgfgraphicnamed{exOldEvalStrat}%
    \small%
    \begin{tikzpicture}[inner ysep=0.15em,line width=0.7pt,>=latex']%
    \node[rectangle,draw,inner xsep=0.20em] (comp1) at (0,4)
      {$\begin{array}{@{}l@{}}
      r_1{:}\;\mygoinout(\myindoor) \lors \mygoinout(\myoutdoor) \lifs. \\
      r_3{:}\;\mygolocation(X) \lors \myngolocation(X) \lifs
      \mygoinout(P), \mylocation(P,X). \\
      r_4{:}\;\mygosomewhere \lifs \mygolocation(X). \\
      c_6{:}\;{\lif}\, \mygolocation(X), \mygolocation(Y), X \neq Y. \\
      c_7{:}\;{\lif}\, \naf \mygosomewhere. \\
      \text{derives: } \mygoinout(X),\;
                       \mygolocation(X),\;
                       \myngolocation(X),\;
                       \mygosomewhere
      \end{array}$};
    \node[rectangle,draw,below=4mm of comp1,inner xsep=0.2em] (comp2)
      {$\begin{array}{@{}l@{}}
      r_2{:}\;\myneed(\myinout,C) \lifs \ext{\mycost}{\mygoinout}{C}. \\
      r_5{:}\;\myneed(\myloc,C) \lifs \ext{\mycost}{\mygolocation}{C}. \\
      \text{derives: } \myneed(A,B)
      \end{array}$};
    \node[rectangle,draw,below=4mm of comp2] (comp3)
      {$\begin{array}{@{}l@{}}
      c_8{:}\;{\lif}\, \myneed(X,\mymoney). \\
      \text{derives nothing}
      \end{array}$};
    \node[anchor=east,inner sep=0pt,outer sep=0pt] at (comp1.west) {$u_1\,$};
    \node[anchor=east,inner sep=0pt,outer sep=0pt] at (comp2.west) {$u_2\,$};
    \node[anchor=east,inner sep=0pt,outer sep=0pt] at (comp3.west) {$u_3\,$};
    \draw[->] (comp2.north) -- (comp1.south);
    \draw[->] (comp3.north) -- (comp2.south);
    \end{tikzpicture}%
    \endpgfgraphicnamed%
}%
\newcommand\myfigureExBetterEvalStrat{%
    \beginpgfgraphicnamed{exBetterEvalStrat}%
    \small%
    \begin{tikzpicture}[inner ysep=0.25em,xscale=1.5,yscale=1.2,line width=0.7pt,>=latex']%
      \node[rectangle,draw] (comp1) at (0.5,4.8)
        {$\begin{array}{@{}l@{}}
          r_1{:}\;\mygoinout(\myindoor) \lors \mygoinout(\myoutdoor) \,{\lif}. \\
          \text{derives: } \mygoinout(X)
        \end{array}$};
      \node[rectangle,draw,inner xsep=0.2em,anchor=east] (comp2) at (0.2,3)
        {$\begin{array}{@{}l@{}}
          r_2{:}\;\myneed(\myinout,C) \,{\lif} \ext{\mycost}{\mygoinout}{C}. \\
          c_8{:}\;{\lif}\, \myneed(X,\mymoney). \\
          \text{derives: } \myneed(\myinout,C)
        \end{array}$};
      \node[rectangle,draw,inner xsep=0.2em,anchor=west] (comp3) at (0.6,3)
        {$\begin{array}{@{}l@{}}
          r_3{:}\;\mygolocation(X) \lors \myngolocation(X) \lifs \\
          \qquad\qquad\mygoinout(P), \mylocation(P,X). \\
          r_4{:}\;\mygosomewhere \lifs \mygolocation(X). \\
          c_6{:}\;{\lif}\, \mygolocation(X), \mygolocation(Y), X \neqs Y. \\
          c_7{:}\;{\lif}\, \naf \mygosomewhere. \\
          \text{derives: } \mygolocation(X),\;
                           \myngolocation(X),\;
                           \mygosomewhere
        \end{array}$};
      \node[rectangle,draw] (comp4) at (0.5,1.0)
        {$\begin{array}{@{}l@{}}
          r_5{:}\;\myneed(\myloc,C) \lifs \ext{\mycost}{\mygolocation}{C}. \\
          c_8{:}\;{\lif}\, \myneed(X,\mymoney). \\
          \text{derives: } \myneed(\myloc,C)
        \end{array}$};
      \draw[->] (comp2) -- (comp1);
      \draw[->] (comp3) -- (comp1);
      \draw[->] (comp4) -- (comp2);
      \draw[->] (comp4) -- (comp3);
      \node[anchor=east,inner sep=0pt,outer sep=0pt] at (comp1.west) {$u_1\,$};
      \node[anchor=east,inner sep=0pt,outer sep=0pt] at (comp2.west) {$u_2\,$};
      \node[anchor=west,inner sep=0pt,outer sep=0pt] at (comp3.east) {$\,u_3$};
      \node[anchor=east,inner sep=0pt,outer sep=0pt] at (comp4.west) {$u_4\,$};
    \end{tikzpicture}%
    \endpgfgraphicnamed%
}%
\newcommand\myfigureExBetterEvalModel{%
    \beginpgfgraphicnamed{exBetterEvalModel}%
    \small%
    \begin{tikzpicture}[inner ysep=0.25em,inner xsep=0.2em,line width=0.7pt,>=latex',
      unit/.style={rectangle,draw,dashed}]%
    \begin{scope}
    \begin{scope}[yshift=3mm,xscale=1.5,yscale=-1.0]
    \node[rectangle,draw] (m1) at (0,1) {$\emptyset$};
      \node[anchor=south] (m1label) at (m1.north) {$m_1/\scI$};
    \node[rectangle,draw] (m2) at (-1,2) {
      $\begin{array}{@{}r@{}l@{}}
        \{ & \mygoinout(\myindoor) \} \end{array}$};
      \node[anchor=south east] at (m2.north) {$m_2/\scO$};
    \node[rectangle,draw] (m3) at (1,2) {
      $\begin{array}{@{}r@{}l@{}}
        \{ & \mygoinout(\myoutdoor) \} \end{array}$};
      \node[anchor=south west] at (m3.north) {$m_3/\scO$};
    \end{scope}
    \begin{scope}[xshift=-35mm,yshift=-22.5mm,xscale=0.75,yscale=-1.4]
    \node[rectangle,draw] (m4) at (-1,1) {$\myint(m_2)$};
      \node[anchor=south] (m4label) at (m4.north) {$m_4/\scI$};
    \node[rectangle,draw] (m5) at (1,1) {$\myint(m_3)$};
      \node[anchor=south] (m5label) at (m5.north) {$m_5/\scI$};
    \node[] (m4no) at (-1,2) {\Lightning};
    \node[rectangle,draw,minimum width=10mm] (m6) at (1,2) {$\emptyset$};
      \node[anchor=south east] at (m6.north) {$m_6/\scO$};
    \end{scope}
    \begin{scope}[xshift=36mm,yshift=-21.5mm,xscale=1.35,yscale=-1.3]
    \node[rectangle,draw] (m7) at (-1.9,1) {$\myint(m_2)$};
      \node[anchor=south] (m7label) at (m7.north) {$m_7/\scI$};
    \node[rectangle,draw] (m8) at (2,1) {$\myint(m_3)$};
      \node[anchor=south] at (m8.north) {$m_8/\scI$};
    \node[rectangle,draw] (m9) at (-2.9,2) {
      $\begin{array}{@{}r@{}l@{}}
        \{ & \mygosomewhere, \\
           & \myngolocation(\mypoolm), \\
           & \mi{goto}(\mypoola) \} \end{array}$};
      \node[anchor=south west] at (m9.north west) {$m_9/\scO$};
    \node[rectangle,draw] (m10) at (-1,2) {
      $\begin{array}{@{}r@{}l@{}}
        \{ & \mygosomewhere, \\
           & \myngolocation(\mypoola), \\
           & \mi{goto}(\mypoolm) \} \end{array}$};
      \node[anchor=south east] at (m10.north east) {$m_{10}/\scO$};
    \node[rectangle,draw] (m11) at (1,2) {
      $\begin{array}{@{}r@{}l@{}}
        \{ & \mygosomewhere, \\
           & \myngolocation(\mypoolg), \\
           & \mi{goto}(\mypooln) \} \end{array}$};
      \node[anchor=south west] at (m11.north west) {$m_{11}/\scO$};
    \node[rectangle,draw] (m12) at (3.2,2) {
      $\begin{array}{@{}r@{}l@{}}
        \{ & \mygosomewhere, \\
           & \myngolocation(\mypooln), \\
           & \mi{goto}(\mypoolg) \} \end{array}$};
      \node[anchor=south east] at (m12.north east) {$m_{12}/\scO$};
    \end{scope}
    \begin{scope}[xshift=18mm,yshift=-59.0mm,xscale=3.0,yscale=-1.1]
    \node[rectangle,draw] (m13) at (-1,1) {
      $\begin{array}{@{}r@{}l@{}}
        \{ & \mygosomewhere, \mygolocation(\mypooln),
             \myngolocation(\mypoolg) \} \end{array}$};
      \node[anchor=south west] (m13label) at (m13.north west) {$m_{13}/\scI$};
    \node[rectangle,draw] (m14) at (1,1) {
      $\begin{array}{@{}r@{}l@{}}
        \{ & \mygosomewhere, \mygolocation(\mypoolg),
             \myngolocation(\mypooln) \} \end{array}$};
      \node[anchor=south] (m14label) at (m14.north) {$m_{14}/\scI$};
    \node[rectangle,draw] (m15) at (-1,2) {
      $\begin{array}{@{}r@{}l@{}}
        \{ & \myneed(\myloc,\myyogamat) \} \end{array}$};
      \node[anchor=south west] at (m15.north west) {$m_{15}/\scO$};
    \node (m14no) at (1,2) {\Lightning};
    \end{scope}
    \end{scope}
    \begin{scope}[->]
    \draw (m2) -- (m1);
    \draw (m3) -- (m1);
    \draw (m4) -- (m2);
    \draw (m5) -- (m3);
    \draw (m4no) -- (m4);
    \draw (m6) -- (m5);
    \draw (m7) -- (m2);
    \draw (m8) -- (m3);
    \draw (m9) -- (m7);
    \draw (m10) -- (m7);
    \draw (m11) -- (m8);
    \draw (m12) -- (m8);
    \draw (m13) -- (m6.south);
    \draw (m13) -- (m11.south);
    \draw (m14) -- (m6.south);
    \draw (m14) -- (m12.south);
    \draw (m15) -- (m13);
    \draw (m14no) -- (m14);
    \end{scope}
    \node[unit,inner sep=0.4em,
      fit=(m1label) (m2) (m3)] (u1) {};
      \node[left=0mm of u1] {\rotatebox{90}{at unit $u_1$}};
    \node[unit,inner sep=0.4em,
      fit=(m4) (m4label) (m5) (m6)] (u2) {};
      \node[anchor=south west] at (u2.north west) {at unit $u_2$};
    \node[unit,inner sep=0.4em,
      fit=(m7) (m7label) (m9) (m12)] (u3) {};
      \node[anchor=south east] at (u3.north east) {at unit $u_3$};
    \node[unit,inner sep=0.4em,
      fit=(m13) (m13label) (m14) (m15)] (u4) {};
      \node[left=0mm of u4] {\rotatebox{90}{at unit $u_4$}};
    \node[unit,right=25mm of u1,
      minimum height=4em,minimum width=8em] (key) {};
    \node[anchor=south] at (key.center) {\reva{$m$ / $\mytype(m)$}};
    \node[anchor=north,rectangle,draw,minimum width=6em]
      at (key.center) {\reva{$\myint(m)$}};
    \node[above=0mm of key] (keyat) {\reva{at unit $\myunit(m)$}};
    \node[anchor=south east] at (key.west) {\reva{Key:}~~~~};
    \end{tikzpicture}%
    \endpgfgraphicnamed%
}%
\newcommand\myfigureCAUexample[1]{
\begin{figure}[#1]
  \centering
  \beginpgfgraphicnamed{myCAUexample}%
  \small%
  \begin{tikzpicture}[line width=0.7pt,>=latex']%
    \begin{scope}[vtx/.style={circle,draw}]
      \node[vtx] (a) at (2,0) {a};
      \node[vtx] (b) at (1,1) {b};
      \node[vtx] (c) at (0,2) {c};
      \node[vtx] (d) at (2,2) {d};
      \node[vtx] (e) at (1,3) {e};
      \node[vtx] (f) at (0,4) {f};
      \node[vtx] (g) at (2,4) {g};
      \begin{scope}[-latex']
      \draw (a) -- (b);
      \draw (a) -- (d);
      \draw (b) -- (c);
      \draw (b) -- (d);
      \draw (c) -- (e);
      \draw (d) -- (e);
      \draw (e) -- (f);
      \draw (e) -- (g);
      \end{scope}
    \end{scope}
    \node at (4,3) {$\mycau(b) = \{ e \}$};
    \node at (4,1) {$\mycau(a) = \{ d,e \}$};
  \end{tikzpicture}%
  \endpgfgraphicnamed%
  \caption{First Ancestor Intersection units (\myCAUstext) in an evaluation graph.}
  \label{fig:cauexamples}
\end{figure}
}%
\newcommand\myfiguregoodbadcau[1]{
\begin{figure}[#1]
  \centering
  \beginpgfgraphicnamed{goodbadcaukey}%
  \small%
  \begin{tikzpicture}
    \begin{scope}[xscale=0.3,yscale=0.3]
    \draw (0,-0.75) rectangle (2,0.75);
    \node[anchor=west] at (2,0) {unit};

    \draw[-open triangle 45] (0,-2) -- (2,-2);
    \node[anchor=west] at (2,-2) {unit dependency};

    \node[circle,draw,anchor=east] at (2,-4) {};
    \node[anchor=west] at (2,-4) {\myiint};

    \node[circle,fill,anchor=east] at (2,-6) {};
    \node[anchor=west] at (2,-6) {\myoint};

    \draw[-triangle 45] (0,-8) -- (2,-8);
    \node[anchor=west] at (2,-8) {dependency};

    \node at (2,-10) {~};
    \end{scope}
  \end{tikzpicture}%
  \endpgfgraphicnamed%
  \quad
  \beginpgfgraphicnamed{badcau}%
  \small%
  \begin{tikzpicture}%
    \begin{scope}[scale=0.5]

      \coordinate (top) at (3,6);
      \coordinate (left) at (-1,3);
      \coordinate (right) at (7,3);
      \coordinate (bot) at (3,-1);

      \node[starburst,draw,minimum width=2.3cm,minimum height=1.3cm,
        starburst points=17,starburst point height=0.3cm,random starburst=344] at ($(top)+(0,-2)$) {};
      \node at ($(top)+(0,-3.5)$) {Violation!};

      \draw (top) ++(-1.5,0) rectangle +(3,1.5);

      \draw (left) rectangle +(+1,-2.5);

      \draw (right) rectangle +(-1,-2.5);

      \draw (bot) ++(1.5,0) rectangle +(-3,-1.5);

      \node[circle,draw] (topi) at ($(top)+(0,-0.5)$) {};
      \node[circle,fill] (topol) at ($(top)+(-0.5,-2.0)$) {};
      \node[circle,fill] (topor) at ($(top)+(+0.5,-2.0)$) {};

      \node[circle,draw] (lefti) at ($(left)+(1.5,-0.5)$) {};
      \node[circle,fill] (lefto) at ($(left)+(1.5,-2.0)$) {};

      \node[circle,draw] (righti) at ($(right)+(-1.5,-0.5)$) {};
      \node[circle,fill] (righto) at ($(right)+(-1.5,-2.0)$) {};

      \node[circle,draw] (boti) at ($(bot)+(0,0.5)$) {};

      \begin{scope}[-triangle 45]
      \draw (topol) -- (topi);
      \draw (topor) -- (topi);
      \draw (lefto) -- (lefti);
      \draw (righto) -- (righti);

      \draw (lefti) -- (topol);
      \draw (righti) -- (topor);
      \draw (boti) -- (lefto);
      \draw (boti) -- (righto);
      \end{scope}
      \begin{scope}[-open triangle 45]
      \draw ($(left)+(+0.5,0)$) -- ($(top)+(-1.5,0.75)$);
      \draw ($(right)+(-0.5,0)$) -- ($(top)+(+1.5,0.75)$);
      \draw ($(bot)+(-1.5,-0.75)$) -- ($(left)+(+0.5,-2.5)$);
      \draw ($(bot)+(+1.5,-0.75)$) -- ($(right)+(-0.5,-2.5)$);
      \end{scope}
    \end{scope}
  \end{tikzpicture}%
  \endpgfgraphicnamed%
  \quad
  \beginpgfgraphicnamed{goodcau}%
  \small%
  \begin{tikzpicture}%
    \begin{scope}[scale=0.5]

      \coordinate (top) at (3,6);
      \coordinate (left) at (-1,3);
      \coordinate (right) at (7,3);
      \coordinate (bot) at (3,-1);

      \node at ($(top)+(0,-3.5)$) {OK!};

      \draw (top) ++(-1.5,0) rectangle +(3,1.5);

      \draw (left) rectangle +(1,-2.5);

      \draw (right) rectangle +(-1,-2.5);

      \draw (bot) ++(1.5,0) rectangle +(-3,-1.5);

      \node[circle,draw] (topi) at ($(top)+(0,-0.5)$) {};
      \node[circle,fill] (topo) at ($(top)+(0,-2.0)$) {};

      \node[circle,draw] (lefti) at ($(left)+(1.5,-0.5)$) {};
      \node[circle,fill] (lefto) at ($(left)+(1.5,-2.0)$) {};

      \node[circle,draw] (righti) at ($(right)+(-1.5,-0.5)$) {};
      \node[circle,fill] (righto) at ($(right)+(-1.5,-2.0)$) {};

      \node[circle,draw] (boti) at ($(bot)+(0,0.5)$) {};

      \begin{scope}[-triangle 45]
      \draw (topo) -- (topi);
      \draw (lefto) -- (lefti);
      \draw (righto) -- (righti);

      \draw (lefti) -- (topo);
      \draw (righti) -- (topo);
      \draw (boti) -- (lefto);
      \draw (boti) -- (righto);
      \end{scope}
      \begin{scope}[-open triangle 45]
      \draw ($(left)+(0.5,0)$) -- ($(top)+(-1.5,0.75)$);
      \draw ($(right)+(-0.5,0)$) -- ($(top)+(+1.5,0.75)$);
      \draw ($(bot)+(-1.5,-0.75)$) -- ($(left)+(0.5,-2.5)$);
      \draw ($(bot)+(+1.5,-0.75)$) -- ($(right)+(-0.5,-2.5)$);
      \end{scope}
    \end{scope}
  \end{tikzpicture}%
  \endpgfgraphicnamed%
  \caption{Interpretation Graphs:
    violation of the \myCAUtext condition
    on the left,
    correct situation on the right.}
  \label{fig:goodbadcau}
\end{figure}
}%
\newcommand\myfigureAtomDepGraph[1]{
  \begin{figure}[#1]%
    \centering%
    \myfigureTikzAtomDepGraph%
    \caption{Atom dependency graph of running example $\myPswim$.}
    \label{fig:swimmingatomdepgraph}
  \end{figure}
}%
\newcommand\myfigureRuleDepGraph[1]{
  \begin{figure}[#1]%
    \centering%
    \myfigureTikzRuleDepGraph%
    \caption{Rule dependency graph of running example $\myPswim$.}
    \label{fig:swimmingruledepgraph}
  \end{figure}
}%
\newcommand\myfigureOldEvalGraph[1]{
  \begin{figure}[#1]%
    \centering%
    \myfigureExOldEvalStrat%
    \caption{Evaluation graph $\cE_1$ for running example \hex program $\myPswim$.}
    \label{fig:exOldEvalStrat}
  \end{figure}
}%
\newcommand\myfigureBetterEvalGraph[1]{
  \begin{figure}[#1]%
    \centering%
    \myfigureExBetterEvalStrat%
    \caption{Evaluation graph $\cE_2$ for running example \hex program $\myPswim$.}
    \label{fig:exBetterEvalStrat}
  \end{figure}
}%
\newcommand\myfigureModelGraphBetterStrat[1]{
  \begin{figure}[#1]%
    \centering%
	\resizebox{\textwidth}{!}{
    \myfigureExBetterEvalModel%
	}
    \caption{Interpretation graph $\cI_2$ for $\cE_2$: \reva{dashed areas group interpretations according to their $\myunit(\cdot)$ value.}}
    \label{fig:exModelGraphBetterStrat}%
  \end{figure}
}%
\newcommand\myfigureAlgoBuildAnswerSets[1]{%
  \begin{algorithm}[#1]
    \caption{\myBuildAnswerSets}%
    \label{alg:buildAnswerSets}
    \DontPrintSemicolon
    \SetAlgoVlined
    \SetVlineSkip{1mm}
    \SetCommentSty{footnotesize}

    \KwIn{$\cE=(V,E)$: evaluation graph for \hex program $P$,
          which contains a unit $\ufinal$ that depends on all other units in $V$}
    \KwOut{a set of all answer sets of $P$}
    $M := \emptyset$,
    $F := \emptyset$,
    $\myunit := \emptyset$,
    $\mytype := \emptyset$,
    $\myint := \emptyset$,
    $U := V$ \;
\smallskip

    \nlset{(a)}\label{step:whileloop}%
    \While{$U \neq \emptyset$}{
      choose $u \in U$ s.t.\ $\myinputs(u) \cap U = \emptyset$ \;
      let $\{ u_1,\ldots,u_k \} = \myinputs(u)$\;
      \eIf{$k = 0$}{%
        \nlset{(b)}\label{step:emptyinputmodels}%
        $m := \mi{max}(M) + 1$ \;
        $M := M \cup \{m\}$ \;
        $\myunit(m) := u$,
        $\mytype(m) := \scI$,
        $\myint(m) := \emptyset$\;
      }{%
        \nlset{(c)}\label{step:firstforloop}%
        \For{$m_1 \in \myomodels(u_1),\dotsc, m_k \in \myomodels(u_k)$}{%
          \If{$J = m_1 \join \dotsb \join m_k$ is defined}{%
            $m := \mi{max}(M) + 1$ \;
            $M := M \cup \{m\}$,
            $F := F \cup \{ (m,m_i) \mid 1\leq i\leq k \}$ \;
            $\myunit(m) := u$,
            $\mytype(m) := \scI$,
            $\myint(m) := J$ \;
          }%
        }
      }
      \nlset{(d)}\label{step:return}%
      \If{$u = \ufinal$}{%
        \Return{$\myimodels(\ufinal)$} \;
      }
      \nlset{(e)}\label{step:secondforloop}%
      \For{$m' \in \myimodels(u)$}{%
        $O := \EvaluateLDESafe(u, \myint(m'))$ \;
        \For{$o \in O$}{%
          $m := \mi{max}(M) + 1$ \;
          $M := M \cup \{m\}$,
          $F := F \cup \{ (m,m') \}$ \;
          $\myunit(m) := u$,
          $\mytype(m) := \scO$,
          $\myint(m) := o$ \;
        }%
      }%
      \nlset{(f)}\label{step:removeu}%
      $U := U \setminus \{ u \}$ \;
    }
  \end{algorithm}
}%
\newcommand\myfigureAlgoEvaluateLDESafe[1]{%
  \begin{algorithm}[#1]
  \caption{\EvaluateLDESafe}
  \label{alg:EvaluateLDESafe}
  \DontPrintSemicolon

  \KwIn{A liberally de-safe \hex{}-program $P$, an input interpretation $I$}
  \KwOut{All answer sets of $P \cup \mathit{facts}(I)$ without $I$}

  \tcp*[h]{add input facts and ground, cf.~\cite{eite-etal-14a}}\;
  $P' \leftarrow \GroundLiberallyDomainExpansionSafeProgram(P \cup \mathit{facts}(I))$\;

  \tcp*[h]{evaluate the ground program, cf.~\cite{efkrs2014-jair},}\;
  \tcp*[h]{and perform output projection}\;
  \Return $\big\{ I' \setminus I \mid I' \in \EvaluateGroundHEX(P') \big\}$\;
  \end{algorithm}%
}%
\newcommand\myfigureAlgoGetNextIModel[1]{%
  \begin{algorithm}[#1]
    \caption{$\myGetNextIModel(u)$}
    \label{alg:getNextIModel}
    \DontPrintSemicolon
    \SetAlgoVlined
    \SetVlineSkip{1mm}
    \SetCommentSty{small}
    \KwIn{$u$: unit}
    \KwOut{$m_{out}$: imodel at $u$ or $\undef$}

\smallskip

    \nlset{(a)}\label{step:gnimLeaf}%
    \If{$\myinputs(u) = \emptyset$}{%
      \uIf{$\mycuri(u) = \undef$}{%
        \nlset{($+$)}%
        add imodel $\emptyset$ at $u$ to $\cA$\;
        \Return{$\emptyset$}\;
      }
      \lElse{%
        \Return{$\undef$}\;
      }%
    }
    let $\{ u_1,\ldots,u_k \} = \myinputs(u)$
    \tcc*{assume this order is fixed for each unit $u$}
    \uIf{$\mycuri(u) \neq \undef$}{%
      $at := \myEnsureModelIncrement(u,1)$\;
      \lIf{$at = \undef$}{\Return{$\undef$}}\;
      $at := at - 1$\;
    }%
    \lElse{$at := k$}\;
    \nlset{(b)}\label{step:gnimWhile}%
    \While{$at \neq 0$}{
      \uIf{$\mycuro(u_{at}) \neq \undef$}{%
        $\myrefcounto(u_{at}) := \myrefcounto(u_{at}) + 1$\;
        $at := at - 1$\;
      }
      \Else{%
        $m := \myGetNextOModel(u_{at})$\;
        \uIf{$m = \undef$}{%
          \lIf{$at \eqs k$}{\Return{$\undef$}}\;
          $at := \myEnsureModelIncrement(u,at+1)$\;
          \lIf{$at = \undef$}{\Return{$\undef$}}\;
        }%
        \Else{%
          $\myrefcounto(u_{at}) := \myrefcounto(u_{at}) + 1$\;
          $at := at - 1$\;
        }
      }
    }
    let $m = \mycuro(u_1) \join \cdots \join \mycuro(u_k)$\;
    \nlset{($+$)}%
    add imodel $m$ to $\cA$ with dependencies to $\mycuro(u_1),\ldots,\mycuro(u_k)$\;
    \Return{$m$}\;
  \end{algorithm}
}%
\newcommand\myfigureAlgoEnsureModelIncrement[1]{%
  \begin{algorithm}[#1]
    \caption{$\myEnsureModelIncrement(u,at)$}
    \label{alg:ensureModelIncrement}
    \DontPrintSemicolon
    \SetAlgoVlined
    \SetVlineSkip{1mm}
    \SetCommentSty{small}
    \KwIn{$u$: unit with $\{u_1,\ldots,u_k\} = \myinputs(u)$, $at$: index $1 \leq at \leq k$}
    \KwOut{$at'$: index $at \leq at' \leq k$ or $\undef$}

\smallskip

    \Repeat{$at = k+1$}{%
      $\myrefcounto(u_{at}) := \myrefcounto(u_{at}) - 1$\;
      $m := \myGetNextOModel(u_{at})$\;
      \lIf{$m = \undef$}{%
        $at := at + 1$\;
      }%
      \Else{%
        $\myrefcounto(u_{at}) := \myrefcounto(u_{at}) + 1$\;
        \Return{$at$}\;
      }
    }
    \Return{$\undef$}\;
  \end{algorithm}
}%
\newcommand\myfigureAlgoGetNextOModel[1]{%
  \begin{algorithm}[#1]
    \caption{$\myGetNextOModel(u)$}
    \label{alg:getNextOModel}
    \DontPrintSemicolon
    \SetAlgoVlined
    \SetVlineSkip{1mm}
    \SetCommentSty{small}
    \KwIn{$u$: unit}
    \KwOut{$m_{out}$: next omodel at $u$ or $\undef$}

\smallskip

    \lIf{$\myrefcounto(u) > 0$}{\Return{\undef}}\;
    \lIf{$\mycuri(u) = \undef$}{$\mycuri(u) := \myGetNextIModel(u)$}\;
    \While{$\mycuri(u) \neq \undef$}{%
      $\mycuro(u) := \myNextAnswerSet(u \cups \facts(\mycuri(u)),\mycuro(u))$\;
      \If{$\mycuro(u) \neq \undef$}{%
        \nlset{($+$)}%
        add omodel $\mycuro(u)$ to $\cA$ with dependency to $\mycuri(u)$\;
        \Return{$\mycuro(u)$}\;
      }
      $\mycuri(u) := \myGetNextIModel(u)$\;
    }
    \Return{$\undef$}\;
  \end{algorithm}
}%
\newcommand\myfigureAlgoOnDemandAS[1]{%
  \begin{algorithm}[#1]
    \caption{$\myOnDemandAS$}
    \label{alg:OnDemandAS}
    \DontPrintSemicolon
    \SetAlgoVlined
    \SetVlineSkip{1mm}
    \SetCommentSty{small}
    \KwIn{evaluation graph $\cE$ for program $P$, with final unit $\ufinal=\emptyset$}
    \KwOut{the answer sets of $P$}
\smallskip

    initialize global storage ${\cal S}$\;
    \Repeat{$m_{out} = \undef$}{%
      $m_{out}$ := \myGetNextOModel($\ufinal$)\;
      \lIf{$m_{out} \neq \undef$}{output $m_{out}$}
      }
\label{MyAlgOnDemandAS}
  \end{algorithm}
}%

\begin{document}
	\maketitle

\vspace*{-1ex}
\noindent
{\bfseries Note:} This article has been accepted for publication in \emph{Theory and Practice of Logic Programming}, \copyright\ Cambridge University Press.

\smallskip
\begin{abstract}
  As software systems are getting increasingly connected, there is a
  need for equipping nonmonotonic logic programs with access to
  external sources that are possibly remote and may contain
  information in heterogeneous formats. To cater for this need, \hex{}
  programs were designed as a generalization of answer set
  programs with an API style interface that allows to access arbitrary
  external sources, providing great flexibility. Efficient evaluation of
  such programs however is challenging, and it requires to interleave external
  computation and model building;
  to decide when to switch between these tasks is difficult, and
  existing approaches have limited scalability in many real-world application scenarios.
  We present a new approach for the evaluation of logic
  programs with external source access, which is based on a configurable
  framework for dividing the non-ground program into
  possibly overlapping
  smaller parts called evaluation units. The latter will be
  processed by interleaving
  external evaluation and model building
  using an evaluation graph and a model graph, respectively, and by
  combining intermediate results.
  Experiments with our prototype implementation show a significant
  improvement
  compared to previous approaches.  While
  designed for \hex-programs, the new evaluation approach may be
  deployed to related rule-based formalisms as well.
\end{abstract}


\abovedisplayshortskip=2pt
\belowdisplayshortskip=2pt
\abovedisplayskip=4pt
\belowdisplayskip=4pt
\maketitle

\section{Introduction}

Motivated by a need for knowledge bases to access external sources,
extensions of declarative KR formalisms have been conceived that
provide this capability, which is often realized via an API-style
interface. In particular, \hex programs~\cite{eist2005-ijcai}
extend nonmonotonic logic programs under the stable model
semantics with the possibility to
bidirectionally access external sources of knowledge and/or computation.
E.g., a rule

\smallskip
\centerline{%
$\mi{pointsTo}(X,Y) \leftarrow \ext{hasHyperlink}{X}{Y}, \mi{url}(X)$
}
\smallskip

\noindent might be used for obtaining
pairs of URLs $(X,Y)$, where $X$ actually links $Y$ on the Web, and $\amp{hasHyperlink}$ is an {\em external
predicate} construct.
Besides
constants (i.e., values) as above, also
relational knowledge (predicate extensions) can flow from external sources
to the logic program
and vice versa, and recursion involving external predicates is allowed under
safety conditions.
This facilitates a variety of applications that require
logic programs to interact with external environments, such as querying
RDF sources using SPARQL~\cite{polleres2007-www},
default rules on ontologies~\cite{hlkh2007,dek2009},
complaint management in e-government~\cite{zy2008-obi},
material culture analysis~\cite{Mosca:Bernini:08},
user interface adaptation~\cite{DBLP:conf/icchp/ZakraouiZ12},
multi-context reasoning~\cite{be2007}, or
robotics and planning~\cite{Schuller2013loitamp,Havur2014},
to mention a few.

Despite the absence of function symbols, an unrestricted use of external
atoms leads to undecidability, as new constants may be introduced from
the sources; in iteration, this can lead to an infinite Herbrand
universe for the program.  However, even under suitable restrictions like liberal
domain-expansion safety~\cite{eite-etal-14a} that avoid this problem,
the efficient evaluation of \hex-programs is challenging, due to aspects
\reva{such as} nonmonotonic atoms and recursive access (e.g., in transitive
closure computations).

Advanced in this regard
was the work by~\citeN{2012_conflict_driven_asp_solving_with_external_sources}, which
fostered an evaluation approach using a traditional LP system. Roughly,
the values of ground external atoms are guessed, model candidates are
computed as answer sets of a rewritten program, and then those discarded
which violate the guess.  Compared to previous approaches such
as the one by~\citeN{eiter-etal-06}, it further exploits conflict-driven techniques
which were extended to external sources.  A generalized notion of
Splitting Set~\cite{lifs-turn-94} was introduced by~\citeN{eiter-etal-06}
for non-ground \hex-programs, which were then split into subprograms
with and without external access, where the former are as large and the
latter as small as possible.
The subprograms
are evaluated with various specific techniques, depending on their
structure~\cite{eiter-etal-06,rs2006}. However, for
real-world applications this approach has severe scalability
limitations, as the number of ground
external atoms may be large, and their combination causes a huge number of model
candidates and memory outage without any answer set output.

To remedy this problem, we reconsider model computation and
make several contributions, which are summarized as follows.
\begin{myitemize}
\item We present a modularity property of \hex-programs based on a novel
generalization of the Global Splitting Theorem~\cite{eiter-etal-06},
which lifted the Splitting Set Theorem~\cite{lifs-turn-94} to \hex-programs.
In contrast to previous
 results, the new result is formulated on a {\em rule splitting set}
 comprising rules that may be non-ground,
 moreover it is based on rule dependencies rather than atom
 dependencies.  This theorem allows for defining answer sets of the
 overall program in terms of the answer sets of program components that may be non-ground.

\item Moreover, we present a generalized version of the new splitting theorem which
allows for sharing constraints
across the split;  %
this helps to prune irrelevant partial models and candidates earlier than in previous
approaches.
As a consequence --- and different from other decomposition approaches---
subprograms for evaluation may overlap and also be non-maximal
(resp.{} non-minimal).

\item Based on \reva{the generalized splitting theorem}, we present an evaluation framework that allows for flexible
evaluation of \hex-programs. It consists of an {\em evaluation graph} and a  {\em model
    graph}; the former captures a modular decomposition and partial
  evaluation order of the program, while the latter comprises for each node collections of sets of input models (which
  need to be combined) and output models to be passed on between components.  This
  structure allows us to realize customized divide-and-conquer
  evaluation strategies.  As the method works on non-ground programs,
  introducing new values by external calculations is feasible, as well as
  applying optimization based on domain splitting~\cite{efk2009-ijcai}.

\item A generic prototype of the evaluation framework has been implemented
 which can be instantiated with different  solvers for Answer Set Programming (ASP)
(in our suite, with \dlv and \clasp). It also features {\em model
 streaming}, i.e., enumeration of the models one by one. In
  combination with early model pruning, this can considerably reduce
  memory consumption and avoid termination without solution output in a
  larger number of settings.

Applying it to ordinary programs (without external functions) allows us to do
parallel solving with a solver software that does not have parallel computing capabilities itself
(\quo{parallelize from outside}).

\end{myitemize}

\reva{This paper,
which significantly extends work in
\cite{2011_pushing_efficient_evaluation_of_hex_programs_by_modular_decomposition} and parts of \cite{ps2012},
}
is organized as follows. In Section~\ref{sec:languageoverview}
we present the \hex-language and consider an example to demonstrate it
in an intuitive way; we will use it as a running example throughout the paper.
In Section~\ref{sec:prelims}
we then introduce necessary restrictions and preliminary concepts that
form dependency-based program evaluation.  After that, we develop in
Section~\ref{sec:ruledepsgeneralizedsplitting} our generalized splitting
theorem, which is applied in Section~\ref{sec:decomposition} to build a
new decomposition framework.  \reva{D}etails about the implementation and
experimental results are given in Section~\ref{sec:heximpl}.
After a discussion including related work in Section~\ref{sec:relatedanddiscussion},
the paper concludes in Section~\ref{sec:conclusion}.
The proofs of all technical results \reva{are given in} \ref{sec:proofs}.

\section{Language Overview}
\label{sec:languageoverview}

In this section, we introduce the syntax and semantics of \hex-programs
as far as this is necessary to explain use cases and basic modeling in the language.
\subsection{\hex Syntax}
\label{sec:prelims_hex_syntax}

Let $\cC$, $\cX$, and $\cG$ be mutually disjoint sets whose elements are called
\emph{constant names}, \emph{variable names}, and \emph{external predicate
names}, respectively.  Unless explicitly specified, elements from $\cX$
(resp., $\cC$) are denoted with first letter in upper case (resp., lower
case), while elements from $\cG$ are prefixed with \quo{\,\&\,}.
Note that constant names serve both as individual and predicate names.

Elements from $\cC\cup\cX$ are called \emph{terms}.
An {\em atom} is a tuple $(Y_0,Y_1,\dots,Y_n)$, where $Y_0,\dots, Y_n$ are
terms; $n\geq 0$ is the \emph{arity} of the atom.
Intuitively, $Y_0$ is the
predicate name, and we thus also use the more familiar notation
$Y_0(Y_1,\dots,Y_n)$.
The atom is {\em ordinary} (resp. \emph{higher-order}), if $Y_0$ is a
constant (resp. a variable).  An atom is {\em ground}, if all its terms
are constants.
Using an auxiliary predicate $\mi{aux}_n$ for each arity $n$,
we can easily eliminate higher-order atoms by rewriting them
to ordinary atoms $\mi{aux}_n(Y_0,\ldots,Y_n)$.
We therefore assume in the rest of this article that programs have no
higher-order atoms.

An {\em external atom}\/ is of the form
\begin{equation}
  \extg[Y_1,\dots,Y_n](X_1,\dots,X_m),
\end{equation}
where $Y_1,\dots,Y_n$ and $X_1,\dots,X_m$ are two lists of terms
(called {\em input} and {\em output} lists, respectively),
and $\extg \in \cG$ is an external predicate name.
We assume that $\extg$ has fixed lengths
$\mi{in}(\extg)=n$ and $\mi{out}(\extg)=m$ for input and output lists, respectively.

Intuitively, an external atom provides a way for deciding the truth value of
an output tuple depending on the input tuple and a given interpretation.

\begin{example}
  \label{ex:eatomintuition}
  $(a,b,c)$, $a(b,c)$, $\mi{node}(X)$, and $D(a,b)$ are atoms;
  the first three are ordinary, where the second atom is a syntactic variant of the first,
  while the last atom is higher-order.

  The external atom
  $\amp{reach}[\mi{edge},a](X)$
  may be devised for computing the nodes which are reachable
  in a graph \reva{represented by atoms of form $\mi{edge}(u,v)$}
  from node $a$.
  We have for the input arity
  $\mi{in}(\amp{reach}) = 2$ and
  for the output arity $\mi{out}(\amp{reach}) = 1$.
  Intuitively, \reva{given an interpretation $I$,}
  $\amp{reach}[\mi{edge},a](X)$ will be true
  for all ground substitutions $X \mapsto b$
  such that $b$ is a node in the graph given by
  \reva{edge list $\{ (u,v) \mids \mi{edge}(u,v) \ins I \}$},
  and there is a path from $a$ to $b$ in that graph.
  \qed
\end{example}

\begin{definition}[rules and \hex programs]
A  {\em rule $r$} is of the form
\begin{equation}
\label{eq:rule}
  \alpha_1\lor\cdots\lor\alpha_k
  \leftarrow
  \beta_1, \dots, \beta_n,
  \naf\, \beta_{n+1},\dots,\naf\,\beta_{m}, \qquad m,k\geq 0,
\end{equation}
where all $\alpha_i$ are atoms
and all $\beta_j$ are either atoms or external atoms.
We let $H(r) = \{ \alpha_1,\ldots,\alpha_k\}$
and $B(r) = B^+(r) \cup B^-(r)$,
where $B^+(r) = \{ \beta_1, \dots, \beta_n\}$ and
$B^-(r) = \{\beta_{n+1}, \dots, \beta_{m} \}$.
Furthermore, a {\em (\hex) program} is a finite set $P$ of rules.
\end{definition}

We denote by $\cons(P)$ the set of constant symbols occurring in a
program $P$.

A rule $r$ is a {\em constraint}, if $H(r) = \emptyset$ and
$B(r)\neq\emptyset$;
a {\em fact}, if $B(r) =\emptyset$ and $H(r)\neq\emptyset$; and
{\em nondisjunctive}, if $|H(r)|\leq 1$.
We call $r$ {\em ordinary}, if it contains only ordinary atoms.
We call a program $P$ \emph{ordinary} (resp., \emph{nondisjunctive}),
if all its rules are ordinary (resp., nondisjunctive).
\reva{Note that facts can be disjunctive, i.e., contain multiple head atoms.}

\def\myneed{\mi{need}}
\def\mytime{\mi{time}}
\def\mymoney{\mi{money}}
\def\myplan{\mi{plan}}
\def\myuse{\mi{use}}
\def\mycost{\mi{rq}}
\def\mychoose{\mi{choose}}
\def\mylocation{\mi{location}}
\def\mygoinout{\mi{swim}}
\def\mygolocation{\mi{goto}}
\def\myngolocation{\mi{ngoto}}
\def\mygosomewhere{\mi{go}}
\def\mypoolm{\mi{margB}}
\def\mypoola{\mi{amalB}}
\def\mypoolg{\mi{gansD}}
\def\mypooln{\mi{altD}}
\def\myindoor{\reva{\mi{ind}}}
\def\myoutdoor{\reva{\mi{outd}}}
\def\myinout{\mi{inoutd}}
\def\myloc{\mi{loc}}
\def\mygoggles{\mi{goggles}}
\def\myyogamat{\mi{yogamat}}
\def\myPswim{\ensuremath{P_\mi{swim}}}
\begin{example}[Swimming Example]
  \label{ex:swimming}
  Imagine Alice wants to go for a swim in Vienna.
  She knows two indoor pools called Margarethenbad and Amalienbad
  (represented by $\mypoolm$ and $\mypoola$, respectively),
  and she knows that outdoor swimming is possible in the river Danube at two locations called
  G\"anseh\"aufel and Alte Donau
  (denoted $\mypoolg$ and $\mypooln$, respectively).%
  \footnote{To keep the example simple, we assume Alice knows
no other possibilities to go swimming in Vienna.}
  She looks up on the Web whether she needs to pay an entrance fee,
  and what additional equipment she will need.
  Finally she has the constraint that she does not want to pay for swimming.

\begin{figure}[tb]
\qquad{$\myPswim^\mi{EDB} =
\left\{ \begin{array}{@{}l@{}}
    \mylocation(\myindoor,\mypoolm),
    \mylocation(\myindoor,\mypoola),\\ \mylocation(\myoutdoor,\mypoolg),
    \mylocation(\myoutdoor,\mypooln)
    \end{array}\right\}$}

\medskip
\qquad{$\myPswim^\mi{IDB} =
\left\{ \begin{array}{r@{:~}r@{~}l@{}}
      r_1 & \mygoinout(\myindoor) \lor \mygoinout(\myoutdoor) \lif &. \\
      r_2 & \myneed(\myinout,C) \lif & \ext{\mycost}{\mygoinout}{C}. \\
      r_3 & \mygolocation(X) \lor \myngolocation(X) \lif & \mygoinout(P), \mylocation(P,X). \\
      r_4 & \mygosomewhere \lif & \mygolocation(X). \\
      r_5 & \myneed(\myloc,C) \lif & \ext{\mycost}{\mygolocation}{C}. \\
      c_6 & \lif & \mygolocation(X), \mygolocation(Y), X \neq Y. \\
      c_7 & \lif & \naf \mygosomewhere. \\
      c_8 & \lif & \myneed(X,\mymoney).
    \end{array}\right\}$}
\caption{Program $\myPswim = \myPswim^\mi{EDB} \cup \myPswim^\mi{IDB}$ to decide swimming location}
\label{fig:Pswim}

\vspace*{-\baselineskip}

\end{figure}

  The \hex program
  $\myPswim = \myPswim^\mi{EDB} \cup \myPswim^\mi{IDB}$ shown in Figure~\ref{fig:Pswim}
  represents Alice's reasoning problem.
  The extensional part $\myPswim^\mi{EDB}$
  contains a set of facts about possible swimming locations
  (where $\myindoor$ and $\myoutdoor$ are short for $\mi{indoor}$ and $\mi{outdoor}$, respectively).
  The intensional part $\myPswim^\mi{IDB}$
  incorporates the web research of Alice
  in an external computation, i.e., using an external atom of the form
  $\ext{\mycost}{\reva{\mi{location\text{-}choice}}}{\reva{\mi{required\text{-}resource}}}$,
  \reva{which intuitively evaluates to true iff
  a given
  $\mi{location\text{-}choice}$
  requires a certain $\mi{required\text{-}resource}$}
  and represents such resources and their origin
  ($\myinout$, or $\myloc$) using predicate $\myneed$.
  Assume Alice finds out that indoor pools in general have an admission fee, and
  that one also has to pay at G\"anseh\"aufel, but not at Alte Donau.
  Furthermore Alice reads some reviews about swimming locations
  and finds out that she will need her Yoga mat for Alte Donau because the ground is so hard,
  and she will need goggles for Amalienbad because there is so much
  chlorine in the water.

  We next explain the intuition behind the rules in $\myPswim$:
  $r_1$ chooses indoor vs.\ outdoor swimming locations,
  and $r_2$ collects requirements that are caused by this choice.
  Rule $r_3$ chooses one of the indoor vs.\ outdoor locations,
  depending on the choice in $r_1$,
  and $r_5$ collects requirements caused by this choice.
  By $r_4$ and $c_7$ we ensure that some location is chosen,
  and by $c_6$ that only a single location is chosen.
  Finally $c_8$ rules out all choices that require money.
  \reva{Note that there is no apparent requirement
  for the first argument of predicate $\myneed$,
  however this argument ensures,
  that $r_2$ and $r_5$ have different heads,
  which becomes important in Example~\ref{ex:swimmingruledepgraph}.}

  The external predicate $\amp{\mycost}$ has \reva{input and output arity}
  $\mi{in}(\amp{\mycost}) \eqs \mi{out}(\amp{\mycost}) \eqs 1$\reva{.
  Intuitively} $\ext{\mycost}{\alpha}{\beta}$
  is true if a resource $\beta$ is required when swimming in a place
  in the extension of predicate $\alpha$.
  For example, $\ext{\mycost}{\mygoinout}{\mymoney}$
  is true if $\mygoinout(\myindoor)$ is true,
  because indoor swimming pool charge money for swimming.
  Note that this only gives an intuitive account of the semantics of $\amp{\mycost}$
  which will formally be defined in \reva{Example~\ref{ex:swimmingrqsemantics}}.
  \qed
\end{example}

\subsection{\hex Semantics}

The semantics of \hex-programs~\cite{eiter-etal-06,rs2006}
generalizes the answer-set semantics~\cite{1991_classical_negation_in_logic_programs_and_disjunctive_databases}.
Let $P$ be a \hex-program.
Then the {\em Herbrand base} of~$P$, denoted~$\HBP$,
is the set of all possible ground
versions of atoms and external atoms occurring in $P$ obtained by
replacing variables with constants from $\cC$.
The grounding of a rule $r$,
$\grnd(r)$,
is defined accordingly,
and the grounding of $P$ is given by
$\grnd(P)=\bigcup_{r\in P}\grnd(r)$.
Unless specified otherwise, %
$\cX$
and $\cG$ are implicitly given by $P$.
Different from the \quo{usual} ASP
setting,
the set of constants $\cC$ used for grounding a program
is only partially given by the program itself;
in \hex, external computations may introduce new constants that are relevant
for semantics of the program.

\begin{example}[ctd.]
  In $\myPswim$ the external atom $\amp{\mycost}$ can introduce constants
  $\myyogamat$ and $\mygoggles$ which are not contained in $\myPswim$,
  but they are relevant for computing answer sets of $\myPswim$.
  \qed
\end{example}

An {\em interpretation relative to} $P$ is any subset
$I \subseteq \HBP$ containing \reva{no external atoms}.
We say that~$I$ is a {\em model} of atom $a\in\HBP$,
denoted $I\,{\models}\, a$, if $a\in I$.

With every external predicate name $\extg \in \cG$,
we associate an $(n{+}m{+}1)$-ary Boolean function \reva{(called \emph{oracle function})}
$f_{\extg}$ assigning each tuple
$(I,y_1\ldots, y_n, x_1, \ldots,x_m)$
either $0$ or $1$,
where $n=\mi{in}(\extg)$, $m=\mi{out}(\extg)$, $I\subseteq \HBP$, and $x_i,y_j\in\cC$.
We say that~$I\subseteq \HBP$ is a {\em model} of a ground
external atom $a$ = $\mi{\&g}[y_1,\dots,y_n](x_1,\dots,x_m)$,
denoted~$I \models a$, if $f_{\extg} (I,y_1\ldots$, $y_n$, $x_1, \ldots,x_m) \,{=}\, 1$.\reva{\footnote{In the implementation,
Boolean functions for defining external sources are realized as plugins to the reasoner which exploit a provided interface and can be written either in Python or C++.}}

Note that this definition of external atom semantics is very general;
indeed an external atom may depend on every part of the interpretation.
Therefore we will later (Section~\ref{sec:extensionalextatomsemantics})
formally restrict external computations
such that they depend only on the extension of those predicates in $I$
which are given in the input list.
All examples and encodings in this work obey this restriction.

\begin{example}[ctd.]
  \label{ex:swimmingrqsemantics}
  The external predicate $\amp{\mycost}$ in $\myPswim$ represents Alice's knowledge
  about %
  swimming locations as follows:
  for any interpretation $I$ and some predicate (i.e., constant) $\alpha$,
  \begin{equation*}
    \begin{array}{l@{~}l@{~}l}
    \reva{I \modelss} \ext{\mycost}{\alpha}{\mymoney} &
      \text{iff } \extfun{\mycost}(I,\alpha,\mymoney) = 1 &
      \text{iff } \alpha(\myindoor) \in I \text{ or } \alpha(\mypoolg) \in I, \\
    \reva{I \modelss} \ext{\mycost}{\alpha}{\myyogamat} &
      \text{iff } \extfun{\mycost}(I,\alpha,\myyogamat) = 1 &
      \text{iff } \alpha(\mypooln) \in I\text{, and} \\
    \reva{I \modelss} \ext{\mycost}{\alpha}{\mygoggles} &
      \text{iff } \extfun{\mycost}(I,\alpha,\mygoggles) = 1 &
      \text{iff } \alpha(\mypoola) \in I.
    \end{array}
  \end{equation*}
  Due to this definition of $\extfun{\mycost}$,
  it holds, e.g., that
  $\{ \mygoinout(\myindoor) \} \models \ext{\mycost}{\mygoinout}{\mymoney}$.
  This matches the intuition about $\amp{\mycost}$
  indicated in the previous example.
  \qed
\end{example}

Let $r$ be a ground rule. Then we say that
\begin{enumerate}[(i)]
\itemsep=0pt
\item
$I$ satisfies the head of $r$, denoted   $I\models H(r)$,
  if  $I \models a$ for some $a \in H(r)$;
\item
$I$ satisfies the body of $r$ ($I\,{\models}\, B(r)$),
  if $I\,{\models}\, a$ for all $a\in B^+(r)$
  and $I\,{\not\models}\, a$ for all $a\in B^-(r)$; and
\item
$I$ satisfies $r$ ($I \models r$),
 if $I\,{\models}H(r)$ whenever $I\,{\models}\, B(r)$.
\end{enumerate}
We say that $I$ is a {\em model} of a \hex-program $P$,
denoted $I\models P$, if $I\,{\models}\, r$ for all $r\in \grnd(P)$.
We call $P$ {\em satisfiable}, if it has some model.

\begin{definition}[answer set]
Given a \hex-program $P$,
the {\em FLP-reduct} of $P$ with respect to $I\subseteq \HBP$,
denoted~$\fP^I$,
is the set of all $r\in \grnd(P)$ such that $I\models B(r)$.
Then $I\subseteq \HBP$ is an \emph{answer set of $P$}
if, $I$ is a minimal model of $\fP^I$.
\reva{We denote by $\AS(P)$ the set of all answer sets of $P$.}
\end{definition}

\begin{example}[ctd.]
  \label{ex:swimanswerset}
  The \hex program $\myPswim$
  with external semantics as given in the previous example
  has a single answer set
  \begin{align*}
  I = \{
    \mygoinout(\myoutdoor),
    \mygolocation(\mypooln),
    \myngolocation(\mypoolg),
    \mygosomewhere,
    \myneed(\myloc,\myyogamat)\}.
  \end{align*}
  (Here, and in following examples,
  we omit $\myPswim^\mi{EDB}$ from all interpretations and answer sets.)
  Under $I$, the external atom $\ext{\mycost}{\mygolocation}{\myyogamat}$ is true
  \reva{and}
  all others
  ($\ext{\mycost}{\mygoinout}{\mymoney}$,
  $\ext{\mycost}{\mygolocation}{\mymoney}$,
  $\ext{\mycost}{\mygoinout}{\myyogamat}$, \ldots)
  are false.
  Intuitively, answer set $I$ tells Alice
  to take her Yoga mat and go for a swim to Alte Donau.
  \qed
\end{example}

\hex programs \cite{eist2005-ijcai} are a conservative extension of disjunctive
(resp., normal)
logic programs under the answer set semantics:
answer sets of \emph{ordinary nondisjunctive \hex programs}
coincide with stable models of logic programs
as
proposed by \citeN{gelf-lifs-88},
and answer sets of \emph{ordinary \hex programs} coincide with
stable models of disjunctive logic programs~\cite{Przymusinski91stablesemantics,%
1991_classical_negation_in_logic_programs_and_disjunctive_databases}.

\revam{
The FLP-reduct as used in the \hex-semantics is equivalent to the GL-reduct,
which removes the default-negated part from the remaining rules and is used for ordinary ASP programs, but the former is superior
for programs with aggregates as it eliminates unintuitive answer sets.
To this end, consider the following example.

\begin{example}
\label{ex:flpsuperior}
Let $P$ be the \hex-program
\begin{align*}
   p(a) &\leftarrow \naf \ext{\mathit{not}}{p}{a} \\
   f &\leftarrow \naf p(a), \naf f
\end{align*}
where $f_{\amp{\mathit{not}}}(I,p,a) = 1$ if $p(a) \not\in I$
and $f_{\amp{\mathit{not}}}(I,p,a) = 0$ otherwise.

The program has the answer set candidates $I_1 = \{ p(a) \}$, $I_2 = \{ p(a), f \}$, $I_3 = \emptyset$ and $I_4 = \{ f \}$.
Under the GL-reduct, we have $P^{I_1} = P^{I_2} = \{ p(a) \leftarrow
\}$, $P^{I_3} = \{ f \leftarrow \}$ and $P^{I_4} = \emptyset$.
As $I_1$ is a minimal model of $P^{I_1}$,
it is a GL-answer set of $P$; no other candidate
is a GL-answer set. However, it is not intuitive that $I_1$ is an answer set as $p(a)$ supports itself.
Using the FLP-reduct, we get $f P^{I_1} = \{\, p(a)
\leftarrow \naf \ext{\mathit{not}}{p}{a} \,\}$.  But now $I_1$ is not a
minimal model of $f P^{I_1}$, as $I_3$ is also a model of $f
P^{I_1}$ and $I_3 \subsetneq I_1$. Similarly, one can check that
$I$  is not a minimal model of $f P^{I}$ for each other candidate $I$;
thus under the FLP-reduct, every interpretation fails to be an answer set.
\end{example}

In the previous example, all answer sets of a \hex program $P$ under the
FLP-reduct are in fact minimal models of $P$; this is not a coincidence
but holds in general.  For a study of properties of \hex-programs, we refer to
\cite{eist2005-ijcai,rs2006,DBLP:journals/tcs/WangYYSZ12} and~\cite{DBLP:journals/ai/ShenWEFRKD14},
where also
variants and refinements of the FLP-semantics are considered, as well as
a particular class of \hex-programs called description logic programs
(see Section~\ref{sec:languageoverview:using}).
 } %

\subsection{Using \hex-Programs for Knowledge Representation and Reasoning}
\label{sec:languageoverview:using}

While ASP is well-suited for many problems in artificial intelligence
and was successfully applied to a range of applications (cf.\ e.g.\
\reva{\cite{DBLP:journals/cacm/BrewkaET11}}),
modern trends computing, for instance in distributed systems and the World Wide Web, require accessing
other sources of computation as well.
\hex-programs cater for this need by its external atoms which provide a bidirectional
interface between the logic program and other sources.

One can roughly distinguish between two main usages of external sources,
which we will call
\emph{computation outsourcing}, \emph{knowledge outsourcing},
and combinations thereof.
However, we emphasize that this distinction concerns the usage in an application
but both are based on the same syntactic and semantic language constructs.
For each of these groups we will describe some typical use cases
which serve as usage patterns for external atoms when writing \hex-programs.

\subsubsection{Computation Outsourcing}

Computation outsourcing means to send the definition of a subproblem to an external source
and retrieve its result.
The input to the external source uses predicate extensions and constants to define the problem at hand
and the output terms are used to retrieve the result, which can in simple cases also be a \reva{B}oolean decision.

\paragraph{On-demand Constraints}

A special case of the latter case are on-demand constraints of type $\leftarrow \amp{forbidden}[p_1, \ldots, p_n]()$
which eliminate certain extensions of predicates $p_1, \ldots, p_n$.
Note that the external evaluation of such a constraint can also return reasons for conflicts to the reasoner in order to
restrict the search space and avoid reconstruction of the same conflict~\cite{2012_conflict_driven_asp_solving_with_external_sources}.
This is similar to the CEGAR approach in model checking~\cite{cgjlv2003}
\reva{and can be helpful for reducing the size of the ground program:
constraints do not need to be grounded
but they are outsourced into an external atom of the above form,
which then returns violated constraints as nogoods to the solver.
This technique has been used for efficient planning in robotics
where external atoms verify the feasibility of a 3D motion
\cite{Schuller2013loitamp}}.

\paragraph{Computations which cannot (easily) be Expressed by Rules}

Outsourcing computations also allows for including algorithms which
cannot easily or efficiently be expressed \reva{as a} logic program, e.g.,
because they involve floating-point numbers.  As a concrete example, an
artificial intelligence agent for the skills and tactics game
\emph{AngryBirds} needs to perform physics
simulations~\cite{cfgirw2013-pai}. As this requires floating point
computations which can practically not be done by rules as this would
either come at the costs of very limited precision or a blow-up of the
grounding, \hex-programs with access to an external source for physics
simulations are used.

\paragraph{Complexity Lifting}

External atoms can realize computations with a complexity higher than
the complexity of ordinary ASP programs.  The external atom serves than
as an \quo{oracle} for deciding subprograms.  While for the purpose of
complexity analysis of the formalism it is often assumed that external
atoms can be evaluated in polynomial
time~\cite{Faber04recursiveaggregates}\footnote{Under this assumption,
  deciding the existence of an answer set of a propositional
  \hex-program is $\Sigma^P_2$-complete.}, as long as external sources
are decidable there is no practical reason for limiting their complexity
(but of course a computation with greater complexity than polynomial
time lifts the complexity results of the overall formalism as well).  In
fact, external sources can be other ASP- or \hex-programs.
This allows for encoding other formalisms of higher complexity in
\hex-programs, \reva{e.g.,}~\emph{abstract argumentation
  frameworks}~\cite{dung1995-aij}.

\subsubsection{Knowledge Outsourcing}
\label{secKnowledgeOutsourcing}

In contrast, knowledge outsourcing refers to external sources which store information which
needs to be imported, while reasoning itself is done in the logic program.

A typical example can be found in Web resources which provide
information for import, \reva{e.g.,}~\emph{RDF triple stores}~\cite{rdf1999}
or \emph{geographic data}~\cite{Mosca:Bernini:08}.  More advanced use
cases are \emph{multi-context systems}, which are systems of
knowledge-bases ({\em contexts}) that are abstracted to acceptable
belief sets (roughly speaking, sets of atoms) and interlinked by
{\em bridge rules}\/ that range across knowledge
bases~\cite{be2007}; access to individual contexts has been
provided  through external atoms~\cite{befs2010}.  Also sensor data, as often
used when planning and executing actions in an environment, is a form of
knowledge outsourcing (cf.~\acthex{}~\cite{befi2010}).

\subsubsection{Combinations}

It is also possible to combine the outsourcing of computations and of knowledge. A typical example
are logic programs with access to description logic knowledge bases (DL KBs), called \emph{DL-programs}~\cite{2008_combining_answer_set_programming_with_description_logics_for_the_semantic_web}.
A DL KB does not only store information,
but also provides a reasoning mechanism.
This allows the logic program for formalizing queries which initiate external computations based on external knowledge
and importing the results.

\section{Extensional Semantics and Atom Dependencies}
\label{sec:prelims}
We now introduce additional important notions related to \hex-programs.
Some of the following concepts are needed to make the formalism decidable,
others prepare the basic evaluation techniques presented in later sections.

\subsection{Restriction to Extensional Semantics for \hex External Atoms}
\label{sec:extensionalextatomsemantics}
To make \hex programs computable in practice,
it is useful to restrict external atoms,
such that their semantics depends only on extensions of predicates given in the input tuple~\cite{eiter-etal-06}.
This restriction is relevant for all subsequent considerations.

\paragraph{Syntax}
Each $\extg$ is associated with an
input type signature $t_1,\ldots,t_n$ such that
every $t_i$ is the type of input $Y_i$ at position $i$ in the input list of $\extg$.
A {\em type} is either $\inpconst$ or a non-negative integer.

Consider $\extg$, its type signature $t_1,\ldots,t_n$,
and a ground external atom $\extg[y_1,\dots,$ $y_n](x_1,$ $\dots,x_m)$.
Then, in this setting,
the signature of $\extg$ enforces certain constraints on
$f_{\extg}(I,y_1,\dots,$ $y_n,x_1,\dots,x_m)$ such that
its truth value depends only on
\begin{enumerate}[(a)]
\itemsep=0pt
\item
  the constant value of $y_i$ whenever $t_i = \inpconst$, and
\item
  the extension of predicate $y_i$, of arity $t_i$, in $I$ whenever $t_i \in \bbN$.
\end{enumerate}

\reva{Note that parameters of type {\bf const} are different from
parameters of type $0$. In the former case, a parameter is interpreted
as a constant that is passed to the external source (essentially as string $``p"$), while a parameter $p$ with a non-negative integer as type is interpreted as predicate whose extension is passed;
in the special case of type $0$, the extension reduces to the truth value of the propositional atom $p$.}

\begin{example}[ctd.]
  Continuing Example~\ref{ex:eatomintuition}, for $\amp{reach}[\mi{edge},a](x)$,
  we have $t_1 = 2$ and $t_2 = \inpconst$.
  Therefore the truth value of $\amp{reach}[\mi{edge},a](x)$
  depends on the extension of binary predicate $\mi{edge}$,
  on the constant $a$,
  and on $x$.

  Continuing Example~\ref{ex:swimmingrqsemantics},
  the external predicate $\amp{\mycost}$ has $t_1 = 1$,
  therefore the truth value of $\ext{\mycost}{\mygoinout}{x}$ for various $x$
  wrt.\ an interpretation $I$
  depends on the extension of the unary predicate $\mygoinout$ in the input list.
  \qed
\end{example}

Note that the truth value of an external atom with only constant input
terms, i.e., $t_i = \inpconst$, $1 \leq i \leq n$, is independent of
$I$.

Semantic constraints enforced by signatures are formalized next.

\paragraph{Semantics}
Let $a$ be a type, $I$ be an interpretation and $p \in \cC$.
The {\em projection function} $\Pi_a(I,p)$ is the binary
function such that $\Pi_{\inpconst}(I,p) = p$ for $a = \inpconst$, and
$\Pi_{a}(I,p) = \{ (\reva{p,}x_1,\ldots,x_a ) \mid p\,(x_1,\ldots,x_a) \in I \}$  for $a \in
\bbN$. Recall that atoms $p(x_1,\ldots,x_a)$ are tuples
$(p,x_1,\ldots,x_a)$\reva{.
The codomain $D_a$ of $\Pi_a(I,p)$
is }$D_a := \cC^{a+1}$ for $a \ins \bbN$,
i.e., the $a{+}1$-fold cartesian product of $\cC$,
which contains all syntactically
possible atoms with $a$ arguments;
furthermore we let $D_{\inpconst}:=\cC$.

\begin{definition}[extensional evaluation function]
\label{def:extensionalevaluationfunction}
Let $\extg$ be an external predicate with oracle function $f_{\extg}$,
$\mi{in}(\extg) = n$, $\mi{out}(\extg) = m$, and type signature $t_1,\ldots,t_n$.
Then the \emph{extensional evaluation function}
$\extFun{g} : D_{t_1} \times \cdots \times D_{t_n} \to 2^{\cC^m}$
of $\extg$ is defined such that for every $\vec{a}= (a_1,\ldots,a_m)$
\begin{align*}
\vec{a} \in \extFun{g}(\Pi_{t_1}(I,p_1),\ldots,\Pi_{t_n}(I,p_n))
  \text{ iff }
  \extfun{g}(I,p_1,\ldots,p_n,a_1,\ldots,a_m) = 1.
\end{align*}
\end{definition}
Note that $\extFun{g}$ makes the possibility of new constants
in external atoms more explicit:
tuples returned by $\extFun{g}$ may contain constants
that are not contained in $P$.
Furthermore, $\extFun{g}$ is well-defined only under the assertion
at the beginning of this section.

\begin{example}[ctd.]
  \label{ex:swimmingextsem}
For $I$ from Example~\ref{ex:swimanswerset},
we have
$\Pi_1(I,\mygoinout) = \{ (\mygoinout,\myoutdoor) \}$
\text{ and}
\\
$\Pi_1(I,\mygolocation) = \{ (\mygolocation,\mypooln) \}.$
The extensional evaluation function of $\amp{\mycost}$
is
  \begin{align*}
  \begin{array}{@{}r@{}l@{}l@{}}
  \extFun{\mycost}(U) \eqs
    & \{ (\mymoney)   & \mids (X,\myindoor) \in U \text{ or } (X,\mypoolg) \in U \} \cups \\
    & \{ (\myyogamat) & \mids (X,\mypooln) \in U \} \cups \{ (\mygoggles) \mids (X,\mypoola) \in U \}
  \end{array}
  \end{align*}
  Observe that none of the constants $\myyogamat$ and $\mygoggles$
  occurs in $P$
  (we have that $\cons(P) = \{
    \mygoinout,$ $\mygolocation,$ $\myngolocation,$ $\myneed,$ $\mygosomewhere,$ $\myinout,
    \myloc,$ $\myindoor,$ $\myoutdoor,
    \mypoola,$ $\mypoolg,$ $\mypooln,$ $\mypoolm,
    \mymoney,$ $\mylocation \}$).
  \reva{These constants} are introduced by the external atom semantics.
  Note that $(\mymoney)$ is a unary tuple,
  as $\amp{\mycost}$ has a unary output list.
  \qed
\end{example}

\subsection{Atom Dependencies}
To account for dependencies
between heads and bodies
of rules is a common
approach for realizing semantics of ordinary logic
programs, as done, e.g., by
means of the notions of
\emph{stratification} and its refinements like \emph{local
  stratification} \cite{przy-88} or
\emph{modular stratification} \cite{ross-94jacm},
or by \emph{splitting sets} \cite{lifs-turn-94}.
In \hex programs, head-body dependencies
are not the only possible source of predicate interaction.
Therefore new types of (non-ground) dependencies were considered by~\citeN{eiter-etal-06} and~\citeN{rs2006}.
In the following we recall these definitions
but slightly reformulate and extend them,
to prepare for the following sections
where we lift atom dependencies to rule dependencies.

In contrast to the traditional notion of dependency,
which in essence hinges on propositional programs,
we must consider
non-ground atoms;
such atoms $a$ and $b$ clearly depend on each other if they unify,
which we denote by $a \sim b$.

For analyzing program properties it is relevant whether a dependency
is positive or negative.
Whether the value of an external atom $a$ depends on the presence
of an atom $b$ in an interpretation $I$ depends in turn on the
oracle function $\extfun{g}$ that is associated with the external
predicate $\extg$ of $a$.
Depending on other atoms in $I$,
in some cases the \emph{presence} of $b$
might make $a$ true, in some cases its \emph{absence}.
Therefore we will
not speak of \emph{positive} and \emph{negative} dependencies, as by~\citeN{2011_pushing_efficient_evaluation_of_hex_programs_by_modular_decomposition},
but
more adequately of
\emph{monotonic} and \emph{nonmonotonic} dependencies, respectively.%
\footnote{%
Note that anti-monotonicity
(i.e., a larger input of an external atom can only make the external atom false, but never true)
could be a third useful distinction that was exploited in \cite{2012_conflict_driven_asp_solving_with_external_sources}.
We here only
distinguish
monotonic from nonmonotonic external atoms
and
classify antimonotonic external atoms as nonmonotonic.
}

\begin{definition}
\label{def:eatommonotonicity}
An external predicate $\amp{g}$ is \emph{monotonic},
if
for all interpretations $I,I'$ such that $I \subseteq I'$
and all \reva{tuples $\vec{X}$ of constants},
$\extfun{g}(I,\vec{X}) = 1$
implies
$\extfun{g}(I',\vec{X}) = 1$;
otherwise  $\amp{g}$ is
\emph{nonmonotonic}.
Furthermore,
a ground external atom $a$ is {\em monotonic}, if for all
interpretations $I, I'$ such that $I\subseteqs I'$ we have
$I \models a$ implies $I' \models a$;
a non-ground external atom is {\em monotonic}, if each of its ground
instances is monotonic.
\end{definition}

Clearly, each external atom that involves a monotonic external
predicates is monotonic, but not vice versa; thus monotonicity of
external atoms is more fine-grained.
In the \reva{following formal definitions,
for simplicity we only consider external predicate monotonicity
and disregard external atom monotonicity.
However} the extension to arbitrary monotonic
external atoms is straightforward.

\begin{example}[ctd.]
  \label{ex:swimmingmonotonic}
  Consider $\extFun{\mycost}(U)$ in Example~\ref{ex:swimmingextsem}:
  adding tuples to $U$
  cannot remove tuples from $\extFun{\mycost}(U)$,
  therefore $\amp{\mycost}$ is a monotonic external predicate.
\qed
\end{example}

Next we define relations for dependencies from external atoms to other atoms.
\begin{definition}[External Atom Dependencies]
  \label{def:eatomdependencies}
  Let $P$ be a \hex program,
  let $a= \ext{g}{X_1,\dots,X_\reva{k}}{\vec Y}$ in $P$
  be an external atom with the type signature $t_1,\dots,t_\reva{k}$
  and let $b = p(\vec{Z})$ be an atom in the head of a rule in $P$.
  Then $a$ \emph{depends external monotonically (resp., nonmonotonically) on} $b$,
denoted $a \dependsextmon b$ (resp., $a \dependsextnmon b$),
  if
$\amp{g}$ is monotonic (resp., nonmonotonic),
  \reva{and for some $i \ins \{ 1, \ldots, k \}$ we have that}
$\vec{Z}$ has arity $t_i \reva{\ins \bbN}$%
  and %
    $X_i=p$.
We define that $a \dependsext b$ if $a \dependsextmon b$ or $a \dependsextnmon b$.
\end{definition}

\begin{example}[ctd.]
  In our %
  example we have the three external dependencies
  $\ext{\mycost}{\mygoinout}{C} \dependsextmon \mygoinout(\myindoor)$,
  $\ext{\mycost}{\mygoinout}{C} \dependsextmon \mygoinout(\myoutdoor)$, and
  $\ext{\mycost}{\mygolocation}{C} \dependsextmon \mygolocation(X)$.
  \qed
\end{example}

As in ordinary ASP,
atoms in \hex programs
may depend on each other because of rules in the program.
\begin{definition}
\label{def:ruleatomdependencies}
For a \hex-program $P$ and atoms $\alpha$, $\beta$ occurring in $P$,
we say that
\begin{enumerate}[(a)]
\item
$\alpha$ depends monotonically on $\beta$
    ($\alpha \dependsmon \beta$),
    if one of the following holds:
\begin{compactenum}[(i)]
\item
some rule $r \in P$ has $\alpha \in H(r)$ and $\beta \in B^+(r)$;
\item
there are rules $r_1,r_2 \in P$ such that
$\alpha \in B(r_1)$, $\beta \in H(r_2)$, and $\alpha \sim \beta$; or
\item
some rule $r \in P$ has $\alpha \in H(r)$ and $\beta \in H(r)$.
\end{compactenum}
\item
$\alpha$ depends nonmonotonically on $\beta$ ($\alpha \dependsnmon \beta$),
if there is some rule $r \in P$ such that
$\alpha \in H(r)$ and $\beta \in B^-(r)$.
\end{enumerate}
\end{definition}
Note that combinations of
Definitions~\ref{def:eatomdependencies} and~\ref{def:ruleatomdependencies}
were already introduced by~\citeN{rs2006} and~\citeN{efk2009-ijcai};
however these \reva{papers} represent nonmonotonicity of external atoms within rule body dependencies
and use a single \quo{external dependency} relation
that does not contain information about monotonicity.
In contrast,
we represent nonmonotonicity of external atoms
where it really happens, namely in dependencies from external atoms to ordinary atoms.
We therefore obtain a simpler dependency relation between rule bodies and heads.

We say that atom $\alpha$ \emph{depends}\/ on atom $\beta$, denoted
$\alpha \depends \beta$, if either  $\alpha\dependsmon\beta$,
$\alpha\dependsnmon\beta$, or  $\alpha\dependsext\beta$; that is,
$\depends$ is the union of the relations $\dependsmon$, $\dependsnmon$, and $\dependsext$.

We next define the atom dependency graph.
\begin{definition}
  \label{def:atomdependencygraph}
  For a \hex-program $P$,
  the \emph{atom dependency graph}
  $\mi{ADG}(P) = (V_{A},E_{A})$ of $P$
  has as vertices $V_{A}$ the (\reva{possibly} non-ground) atoms occurring in non-facts of $P$
  and as edges $E_{A}$ the dependency relations
  $\dependsmon$, $\dependsnmon$, $\dependsextmon$, and $\dependsextnmon$
  between %
  them in $P$.
\end{definition}

\myfigureAtomDepGraph{tbp}

\begin{example}[ctd.]
  Figure~\ref{fig:swimmingatomdepgraph}
  shows
$\mi{ADG}(\myPswim)$.
  \reva{Recall that $c_7$ is \quo{$\lifs \naf \mygosomewhere$}.}
  Note that the nonmonotonic body literal in $c_7$
  does not show up as a nonmonotonic dependency,
  as $c_7$ has no head atoms.
  (The rule dependency graph %
  in Section~\ref{sec:ruledepsgeneralizedsplitting}
  will make this negation apparent.)
  \qed
\end{example}

Next we use the dependency notions to define safety conditions on \hex programs.

\subsection{Safety Restrictions}

To  make reasoning tasks on \hex programs decidable (or more efficiently computable),
the following potential restrictions were formulated.

\paragraph{Rule safety}
This is a restriction well-known in logic programming,
and it is required to ensure finite grounding of a non-ground program.
A rule is safe, if all its variables are safe,
and a variable is safe if it is contained in a positive body literal.
Formally a rule $r$ is safe iff
variables in $H(r) \cup B^-(r)$ are a subset of variables in $B^+(r)$.

\paragraph{Domain-expansion safety}
In an ordinary logic program $P$,
we usually assume that the set of constants $\cC$
is implicitly given by $P$.
In a \hex program, external atoms may invent new constant values
in their output tuples.
We therefore must relax this to \quo{$\cC$ is countable and partially given by $P$},
as shown by the following example.
\begin{example}

In the Swimming Example, grounding $\myPswim$ with $\cons(\myPswim)$
is not sufficient.
Further
constants
`generated' by external atoms
must be considered.
For example $\myyogamat \notin \cons(\myPswim)$
 and $I \models \ext{\mycost}{\mygolocation}{\myyogamat}$,
hence we must ground
  \begin{align*}
    \myneed(\myloc,C) \lif \ext{\mycost}{\mygolocation}{C}
  \end{align*}
with $C=\myyogamat$  to obtain the correct answer set.
  \qed
\end{example}

Therefore grounding $P$ with $\cons(P)$ can lead to incorrect results.
Hence we want to obtain new constants during evaluation of external atoms,
and we must use these constants to evaluate the remainder of a given \hex program.
However, to ensure decidability,
this process of obtaining new constants must always terminate.

Hence, we require programs to be \emph{domain-expansion safe}~%
\cite{eiter-etal-06}:
there must not be a cyclic dependency between rules and external atoms
such that an input predicate of an external atom
depends on a variable output of that same external atom,
if the variable is not guarded by a domain predicate.

With \hex we need the usual notion of rule safety,
i.e., a syntactic restriction which ensures
that each variable in a rule only has a finite set of relevant constants for grounding.

We first recall the definition of safe variables and safe rules for
\hex.
\begin{definition}[Def.~5 by~\citeN{eiter-etal-06}]
  \label{def:hexsafety}
The \emph{safe variables} of a rule $r$
is the smallest set of variables $X$
that occur either
\begin{inparaenum}[(i)]
  \item
in some ordinary atom $\beta \in B^+(r)$, or
  \item
in the output list $\vec{X}$
of an external atom $\ext{g}{Y_1,\ldots,Y_n}{\vec{X}}$
in $B^+(r)$ where all $Y_1,\ldots,Y_n$ are safe.
\end{inparaenum}
  A rule $r$ \emph{is safe}, if each variable in $r$ is safe.%
  \footnote{This is stated by~\citeN{eiter-etal-06}
  as \quo{if each variable appearing in a negated atom and in any input list is safe,
  and variables appearing in $H(r)$ are safe}, which is equivalent.
}
\end{definition}

However, safety alone does not guarantee finite grounding of \hex programs,
because an external atom might create new constants,
i.e., constants not part of the program itself,
in its output list (see Example~\ref{ex:swimmingextsem}).
These constants can become part of the extension of an atom in the rule head,
and by grounding and evaluation of %
other rules become part of the extension of a predicate
which is an input to the very same external atom.
\begin{example}[adapted from~\citeN{rs2006}]
  The following \hex program is safe according to Definition~\ref{def:hexsafety}
  and nevertheless cannot be finitely grounded:
  \begin{align*}
    \begin{array}{@{}r@{}l@{}}
    \mi{source}(``\mathtt{http://some\_url}") &\lifs. \\
    \mi{url}(X) &\lifs
      \ext{rdf}{\mi{source}}{X,\mi{``rdf{:}subClassOf"},C}. \\
    \mi{source}(X) &\lifs \mi{url}(X).
    \end{array}
  \end{align*}
  Suppose the $\ext{rdf}{\mi{source}}{S,P,O}$ atom retrieves all triples $(S,P,O)$
  from all RDF triplestores specified in the extension of $\mi{source}$,
  and suppose that each triplestore contains a triple
  with a URL $S$ that does not show up in another triplestore.
  As a result, all these URLs are collected in the extension of $\mi{source}$
  which leads to even more URLs being retrieved and a potentially infinite grounding.

  However, we could change the rule with the external atom to
  \begin{align}
    \label{eqn:savedcycle}
    \mi{url}(X) \lif
      \ext{rdf}{\mi{source}}{X,\mi{``rdf{:}subClassOf"},C},
        \mi{limit}(X) %
  \end{align}
  and add an appropriate set of $\mi{limit}$ facts.
  This addition of a range predicate $\mi{limit}(X)$
  which does not depend on the external atom output
  ensures a finite grounding.
  \qed
\end{example}

To obtain a syntactic restriction that ensures finite grounding for \hex,
so called \emph{strong safety} has been
introduced for the \hex programs~\cite{eiter-etal-06}.
Intuitively, this concept requires all output variables of cyclic external atoms (using
the dependency notion from Definition~\ref{def:atomdependencygraph})
to be bounded by ordinary body atoms of the same rule which are not part of the cycle.
However,
this condition is unnecessarily restrictive, and
\reva{therefore},
the extensible notion of \emph{liberal
  domain-expansion safety (lde-safety)} was introduced by~\citeN{eite-etal-14a},
which we will use in the following.
For the purpose of this article, we may omit
the formal details of lde-safety (see \citeN{eite-etal-14a} and
\ref{sec:liberalsafety} for an outline);
it is sufficient to know that every lde-safe program has
a finite grounding that has the same answer sets as the original
program.

\section{Rule Dependencies and Generalized Rule Splitting Theorem}
\label{sec:ruledepsgeneralizedsplitting}

In this section, we \reva{first} introduce a new notion of dependencies in \hex-programs,
namely between non-ground \emph{rules} in a program
\reva{(Section~\ref{sec:ruledependencies})}.
Based on this notion, we then present a modularity property of \hex-programs
that allows us to obtain answer sets of a program from the answer sets
of its components \reva{(Section~\ref{sec:splittings})}.
The property is formulated as a splitting theorem based on dependencies among
rules and lifts a similar result for dependencies among atoms,
viz.\ the Global Splitting Theorem~\cite{eiter-etal-06}, to this
setting, and it generalizes and improves it.
This result is
exploited in a more efficient \hex-program evaluation algorithm, which we show
in Section~\ref{sec:decomposition}.

\subsection{Rule Dependencies}
\label{sec:ruledependencies}

We define rule dependencies as follows.
\begin{definition}[Rule dependencies]
\label{def:dependencies}
Let $P$ be a program and $a,b$ atoms occurring in distinct
rules $r,s\in P$.
Then $r$ depends on $s$ according to the following cases:
    \begin{compactenum}[(i)]
    \item\label{itm:pos}
      if $a \sim b$, $a \in B^+(r)$, and $b \in H(s)$, then
      $r \dependsmon s$;
    \item\label{itm:neg}
      if $a \sim b$, $a \in B^-(r)$, and  $b \in H(s)$, then
      $r \dependsnmon s$;
    \item\label{itm:sim}
      if  $a \sim b$, $a \in H(r)$, and $b \in H(s)$, then
      both $r \dependsmon s$ and $s \dependsmon r$;
    \item\label{itm:ext}
      if $a \dependsext b$, $a \in B(r)$ is an external atom, and $b \in H(s)$,
      then
      \begin{compactitem}
      \item
        $r \dependsmon s$
        if $a \in B^+(r)$ and $a \dependsextmon b$,
        and
      \item
        $r \dependsnmon s$
        otherwise.
      \end{compactitem}
    \end{compactenum}
\end{definition}
Intuitively,
conditions~\eqref{itm:pos} and~\eqref{itm:neg}
reflect the fact that the applicability of a rule $r$
depends on the applicability of a rule $s$ with a head
that unifies with a literal in the body of rule $r$;
condition~\eqref{itm:sim} exists because $r$ and $s$
cannot be evaluated independently
if they share a common head atom
(e.g., $u \lor v \leftarrow$ cannot be evaluated independently from $v \lor w \leftarrow$);
and~\eqref{itm:ext} defines dependencies due to predicate inputs of external atoms.

In the sequel, we let $\dependsmn \eqs \dependsmon \cups \dependsnmon$
be the union of monotonic and non\-mono\-tonic rule dependencies.
We next define graphs of rule dependencies.
\begin{definition}
  \label{def:ruledepgraph}
  Given a \hex-program $P$,
  the \emph{rule dependency graph} $DG(P) = (V_D,E_D)$ of $P$ is the labeled graph
  with vertex set $V_D = P$
  and edge set $E_D \eqs \dependsmn$.
\end{definition}

\myfigureRuleDepGraph{tbp}
\begin{example}[ctd.]
  \label{ex:swimmingruledepgraph}
  Figure~\ref{fig:swimmingruledepgraph}
  depicts the rule dependency graph of our running example.
  According to Definition~\ref{def:dependencies},
  we have the following rule dependencies in $\myPswim^\mi{IDB}$:
  \begin{compactitem}
  \item
    due to (i) we have
    $r_3 \dependsmon r_1$,
    $r_4 \dependsmon r_3$,
    $c_6 \dependsmon r_3$,
    $c_8 \dependsmon r_2$, and
    $c_8 \dependsmon r_5$;
  \item
    due to (ii) we have
    $c_7 \dependsnmon r_4$;
  \item
    due to (iii) we have no dependencies; and
  \item
    due to (iv) we have
    $r_2 \dependsmon r_1$ and
    $r_5 \dependsmon r_3$.
  \end{compactitem}
  \reva{Note that if we would omit the first argument of predicate $\myneed$,
  we would have in addition
  $r_2 \dependsmon r_5$ and $r_5 \dependsmon r_2$ due to (iii).}
  Also note that $\amp{\mycost}$ is monotonic
  (see Example~\ref{ex:swimmingmonotonic}).
  \qed
\end{example}

\subsection{Splitting Sets and Theorems}
\label{sec:splittings}

Splitting sets are a notion that allows for
describing how a program can be decomposed into parts
and how semantics of the overall program
can be obtained from semantics of these parts
in a divide-and-conquer manner.

We lift the original \hex splitting theorem~\cite[Theorem 2]{eiter-etal-06}
and the according definitions of global splitting set, global bottom,
and global residual~\cite[Definitions 8 and 9]{eiter-etal-06}
to our new definition of dependencies among rules.

A \emph{rule splitting set} is a part of a (non-ground) program
that does not depend on the rest of the program.
This corresponds in a sense with global splitting sets by~\citeN{eiter-etal-06}.
\begin{definition}[Rule Splitting Set]
  A \emph{rule splitting set} $R$ for a \hex-program $P$
  is a set $R \subseteq P$ of rules such that
  whenever $r \in R$, $s \in P$, and $r \dependsmn s$,
  then $s \in R$ holds.
\end{definition}
\begin{example}[ctd.]
  \label{ex:swimmingrulesplittingsets}
  The following are some rule splitting sets of $\myPswim$:
  $\{r_1\}$, $\{r_1,r_2\}$, $\{r_1,r_3\}$, $\{r_1,r_2,r_3\}$, $\{r_1,r_2,r_3,r_5,c_8\}$.
  The set $R=\{r_1,r_2,c_8\}$ is not a rule splitting set,
  because $c_8 \dependsmon r_5$ but $r_5\notin R$.
  \qed
\end{example}

Because of possible constraint duplication,
we no longer partition the input program,
and the customary notion of splitting set, bottom, and residual,
is not appropriate for sharing constraints between bottom and residual.
Instead, we next define a \emph{generalized bottom} of a program,
which splits a non-ground program into two parts
which may share certain constraints.
\begin{definition}[Generalized Bottom]
  \label{def:genbottom}
  Given a rule splitting set $R$ of a \hex-program $P$,
  a \emph{generalized bottom} $B$ of $P$ wrt.\ $R$
  is a set $B$ with $R \subseteq B \subseteq P$
  such that all rules in $B \setminus R$
  are constraints that do not depend nonmonotonically on any rule in $P \setminus B$.
\end{definition}
\begin{example}[ctd.]
  A rule splitting set $R$ of $\myPswim$
  (e.g., those given in Example~\ref{ex:swimmingrulesplittingsets})
  is also a generalized bottom of $\myPswim$ wrt.\ $R$.
  The set $\{r_1,r_2,c_8\}$ is not a rule splitting set,
  but it is a generalized bottom of $\myPswim$
  wrt.\ the rule splitting set $\{r_1,r_2\}$,
  as $c_8$ is a constraint that depends only monotonically
  on rules in $\myPswim \setminus \{r_1,r_2,c_8\}$.
  \qed
\end{example}

Next, we describe how interpretations of a generalized bottom $B$ of a program $P$
lead to interpretations of $P$ without re-evaluating rules in $B$.
Intuitively, this is a relaxation of the previous non-ground \hex
splitting theorem:
a constraint may be put both
in the bottom and in the residual if it has no nonmonotonic
dependencies to the residual.
The benefit of such constraint sharing is a smaller number of answer
sets of the bottom, and hence of fewer evaluations of the residual program.

\newcommand{\gh}[1]{{\mi{gh}(#1)}}

\smallskip

\noindent{\bf Notation.} For any set $I$ of ground ordinary atoms,
we denote by $\facts(I)$ the corresponding set of ground facts;
furthermore, for any set  $P$ of rules,
we denote by $\gh{P}$ \label{pos:ghdefinition}
the set of ground head atoms occurring in $\grnd(P)$.
\begin{theorem}[Splitting Theorem]
  \label{thm:hexsplitting}
  Given a \hex-program $P$ and a rule splitting set $R$ of $P$,
  $M \in \AS(P)$
  iff
  $M \in \AS(P \setminus R \cup \facts(X))$
  with
  $X \in \AS(R)$.
\end{theorem}
\newcommand\myproofhexsplitting{
 \def\MghR{\ensuremath{{M|_R}}}
 Given a set of ground atoms $M$ and a set of rules $R$,
 we denote by $M|_R = M \cap \gh{R}$
 the projection of $M$ to ground heads of rules in $R$.

  ($\Rightarrow$)
  Let $M \in \AS(P)$.
  We show that (1) $\MghR \in \AS(R)$
  and that (2) $M \in \AS(P\setminus R \cup \facts(\MghR))$.

As for (1), we first show that
  $\MghR$ satisfies the reduct $fR^\MghR$,
  and then that  $\MghR$ is indeed a minimal model of $fR^\MghR$.
  $M$ satisfies $fP^M$ and $R \subseteq P$.
  Observe that, by definition of FLP reduct,
  $fR^M \subseteq fP^M$.
  By definition of rule splitting set,
  satisfiability of rules in $R$ does not depend on heads of rules
  in $P \setminus R$
  (due to the restriction of external atoms to extensional semantics,
  this is in particular true for external atoms in %
  $R$).
  Therefore $fR^\MghR = fR^M$,
  $M$ satisfies $fR^\MghR$,
  and $\MghR$ satisfies $fR^\MghR$.
  For showing $\MghR \in \AS(R)$, it remains to show that
  $\MghR$ is a minimal model of $fR^\MghR$.

  Assume towards a contradiction that some $S \subset \MghR$
  is a model of $fR^\MghR$.
  Then there is a nonempty set $A  = \MghR \setminus S$ of atoms
  with $A \subseteq \gh{R}$.
  Let $M^\star = M \setminus A$.
  We next show that
  $M^\star$ is a model of $fP^M$,
  which implies that $M \notin \AS(P)$.
  Assume on the contrary that $M^\star$ is not a model of $fP^M$.
  Hence there exists some rule $r \in fP^M$
  such that $H(r) \cap M^\star = \emptyset$,
  $B^+(r) \subseteq M^\star$,
  $B^-(r) \cap M^\star = \emptyset$
  and external atoms in $B^+(r)$ (resp., $B^-(r)$)
  evaluate to true (resp., false)
  wrt.\ $M^\star$.
  $S$ agrees with $M^\star$ on atoms from $\gh{R}$,
  and $S$ satisfies $fR^\MghR$.
  The truth values of external atoms in bodies of rules in $R$
  depends only on atoms from $\gh{R}$,
  therefore external atoms in $R$ evaluate to the same truth value wrt.\ $S$ and $M^\star$.
  Therefore $r \notin fR^\MghR$
  and $r \in f(P \setminus R)^M$.
  Since $r \in P \setminus R$,
  $H(r) \subseteq \gh{P \setminus R}$,
  and because $M$ and $M^\star$ agree on atoms from $\gh{P \setminus R}$,
  $H(r) \cap M^\star = \emptyset$ from above implies that
  $H(r) \cap M = \emptyset$.
  Because $r \in fP^M$,
  its body is satisfied in $M$,
  and since its head has no intersection with $M$,
  we get that $fP^M$ is not satisfied by $M$,
  which is a contradiction.
  Therefore $M^\star$ is a model of $fP^M$.
  As $M^\star \subsets M$,
  this contradicts our assumption that $M \in \AS(P)$.
  Therefore $S = \MghR = X$ is a minimal model of $fR^M$.

  We next show that
  $M$ satisfies the reduct $f(P \setminus R \cup \facts(\MghR))^M$,
  and then that it is indeed a minimal model of the reduct.
  By the definition of reduct,
  $f(P \setminus R \cup \facts(\MghR))^M =
   f(P \setminus R)^M \cup \facts(\MghR)$.
  $M$ satisfies $\facts(\MghR)$ because $\MghR \subseteq M$.
  Furthermore $f(P \setminus R)^M \subseteq fP^M$,
  hence $M$ satisfies $f(P \setminus R)^M$.
  Therefore $M$ satisfies $f(P \setminus R \cup \facts(\MghR))^M$.

  To show that $M$ is a minimal model of
  $f(P \setminus R \cup \facts(\MghR))^M$,
  assume towards a contradiction
  that some $S \subsets M$ is a model of
  $f(P \setminus R \cup \facts(\MghR))^M$.
  Since $\facts(\MghR)$ is part of the reduct,
  $\MghR \subseteq S$,
  therefore
  $S|_\gh{R} = \MghR$.
  By definition of rule splitting set,
  satisfiability of rules in $R$ does not depend on heads of rules
  in $P \setminus R$,
  hence $S$ satisfies $fR^M$.
  Because $S$ satisfies
  $f(P \setminus R \cup \facts(\MghR))^M =
   f(P \setminus R)^M \cup \facts(\MghR)$,
  it also satisfies $f(P \setminus R)^M$.
  Since $S$ satisfies both $fR^M$,
  $S$ satisfies $fP^M = f(P \setminus R)^M \cup fR^M$.
  This is a contradiction to $M \in \AS(P)$.
  Therefore $S = M$ is a minimal model of
  $f(P \setminus R \cup \facts(\MghR))^M$.

  ($\Leftarrow$)
  Let $M \in \AS(P \setminus R \cup \facts(X))$ and let $X \in \AS(R)$.
  We first show that $M$ satisfies $fP^M$,
  and then that it is a minimal model of $fP^M$.

  As facts $X$ are part of the program $P \setminus R \cup \facts(X)$,
  and by definition of rule splitting set,
  $P \setminus R$ contains no rule heads unifying with $\gh{R}$,
  hence we have $X = \MghR$.
  Furthermore
  $f(P \setminus R \cup \facts(X))^M \setminus \facts(X) \cup fR^M = fP^M$,
  and as $M$ satisfies the left side, it satisfies the right side.
  To show that $M$ is a minimal model of $fP^M$,
  assume $S \subsets M$ is a smaller model of $fP^M$.
  By definition of reduct, $S$ also satisfies
  $f(P \setminus R)^M$ and $fR^M$.
  Since $R$ is a splitting set,
  satisfiability of rules in $R$ does not depend on heads of rules
  in $P \setminus R$,
  therefore $fR^M = fR^\MghR = fR^X$
  and $S|_{\gh{R}}$ satisfies $fR^X$.
  Since $S \subset M$, we have $S|_{\gh{R}} \subseteq X$.
  Because $X$ is a minimal model of $fR^X$,
  $S|_{\gh{R}} \subset X$ is impossible
  and $S|_{\gh{R}} = X$.
  Therefore $S|_{\gh{P \setminus R}} \subset M|_{\gh{P \setminus R}}$.
  Because $S$ satisfies $f(P \setminus R)^M$
  and $S|_{\gh{R}} = X$,
  $S$ also satisfies $f(P \setminus R \cup \facts(X))^M$.
  Since $S \subset M$, this contradicts the fact that
  $M$ is a minimal model of $P \setminus R \cup \facts(X)$.
  Therefore $S = M$ is a minimal model of $fP^M$.
}%
\myinlineproof{%
  \begin{proof}
    \myproofhexsplitting
  \end{proof}
}
Using the definition of generalized bottom,
we generalize the above theorem.
\begin{theorem}[Generalized Splitting Theorem]
  \label{thm:hexgensplitting}
  Let $P$ be a \hex-program,
  let $R$ be a rule splitting set of $P$,
  and let $B$ be a generalized bottom of $P$ wrt.\ $R$.
  Then
  \begin{equation*}
  \text{
  $M \ins \AS(P)$
  iff
  $M \ins \AS(P \setminuss R \cups \facts(X))$
  where
  $X \ins \AS(B)$.
  }
  \end{equation*}
\end{theorem}
\newcommand\myproofhexgensplitting{
  \def\MghR{\ensuremath{{M|_{\gh{R}}}}}
  \def\MghB{\ensuremath{{M|_{\gh{B}}}}}
  By definition of generalized bottom,
  the set $C = B \setminus R$ contains only constraints,
  therefore $\gh{B}=\gh{R}$ and $\MghB = \MghR$.
  As $R \subseteq B$ and $B \setminus R$ contains only constraints,
  $\AS(B) \subseteq \AS(R)$.
  The only difference between Theorem~\ref{thm:hexsplitting}
  and Theorem~\ref{thm:hexgensplitting} is,
  that for obtaining $X$,
  the latter takes additional constraints into account.

  ($\Rightarrow$)
  It is sufficient to show that
  $\MghB$ does not satisfy the body of any constraint in $C \subseteq P$
  if $M$ does not satisfy the body of any constraint in $P$.
  Since $B$ is a generalized bottom,
  no negative dependencies of constraints $C$ to rules in $P \setminus B$
  exist;
  therefore if the body of a constraint $c \in C$ is not satisfied by $M$,
  the body of $c$ is not satisfied by $\MghB$.
  As $M$ satisfies $P$,
  it does not satisfy any constraint body in $P$,
  hence the projection $\MghB$
  does not satisfy any constraint body in $B \setminus R$.

  ($\Leftarrow$)
  It is sufficient to show that
  an answer set of $R$
  that satisfies a constraint body in $C$
  also satisfies that constraint body in $P$, which raises a contradiction.
  As constraints in $C$
  have no negative dependencies to rules in $P \setminus B$,
  a constraint with a satisfied body in $\MghR$
  also has a satisfied body in $M$,
  therefore the result follows.
}%
\myinlineproof{%
  \begin{proof}
    \myproofhexgensplitting
  \end{proof}
}
Note that $B \setminus R$ contains shareable constraints
that are used twice in the Generalized Splitting Theorem, viz.\ in
computing $X$ and in computing $M$.

The Generalized Splitting Theorem is useful
for early elimination of answer sets of the bottom
thanks to constraints
which depend on it
but also on rule heads outside the bottom.
Such constraints can be shared
between the bottom and the remaining program.
\begin{example}[ctd.]
  We apply Theorems~\ref{thm:hexsplitting} and~\ref{thm:hexgensplitting}
  to $\myPswim$ and compare them.
  Using the rule splitting set $\{r_1,r_2\}$,
  we can obtain $\AS(\myPswim)$ by
  first computing $\AS(\{r_1,r_2\}) = \{ I_1, I_2\}$ where $I_1 = \{
  \mygoinout(\myindoor), \myneed(\myinout,\mymoney) \},$ $I_2 = \{ \mygoinout(\myoutdoor)\}$,
  and by then using Theorem~\ref{thm:hexsplitting}:
  $X\in \AS(\myPswim)$
  iff it holds that
  $X \in \AS(\{r_3,r_4,r_5,c_6,c_7,c_8\} \cup \facts(I_1))$
  or
  $X \in \AS(\{r_3,r_4,r_5,c_6,c_7,c_8\} \cup \facts(I_2))$.
  Note that the computation with $I_1$ yields no answer set, as
 $\myneed(\myinout,\mymoney)\in I_1$ satisfies the body of $c_8$ and \quo{kills} any model candidate.
  In contrast, if we use the generalized bottom $\{r_1,r_2,c_8\}$,
  we have $\AS(\{r_1,r_2,c_8\}) = \big\{ \{ \mygoinout(\myoutdoor) \} \big\}$
  and can use Theorem~\ref{thm:hexgensplitting}
  to obtain $\AS(\myPswim)$ with only one further answer set computation:
  $X\in \AS(\myPswim)$
  iff
  $X \in \AS(\{r_3,r_4,r_5,c_6,$ $c_7,c_8\} \cup \{ \mygoinout(\myoutdoor) \lif \})$.
  Note that we use $c_8$ in both computations,
  i.e., $c_8$ is shared between the generalized bottom and the remaining computation.
  \qed
\end{example}

Armed with the results of this section, we proceed to program evaluation
in the next section. A discussion of the new splitting theorems that
compares them to previous related theorems and argues for their
advantage is given in Section~\ref{sec:related-work}.

\section{Decomposition and Evaluation Techniques}
\label{sec:decomposition}

We now introduce our new \hex evaluation framework, which is
based on selections of sets of rules of a program
that we call
\emph{evaluation units} (or briefly {\em units}).

The traditional \hex evaluation algorithm
\cite{eiter-etal-06}  uses a
dependency graph over (non-ground) atoms, and gradually evaluates
sets of rules (the \quo{bottoms} of
a program) that are chosen based on this graph.  In contrast
our new evaluation algorithm
exploits the rule-based modularity results for \hex-programs
in Section~\ref{sec:ruledepsgeneralizedsplitting}.

While %
previously a constraint can only kill models once all its dependencies on rules
are fulfilled, the new algorithm increases evaluation efficiency by
sharing non-ground constraints, such that they may kill models earlier;
this is safe if all their nonmonotonic dependencies are fulfilled.
Moreover, %
units no longer must be maximal.
Instead, we require that partial models of units,
i.e., atoms in heads of their rules,
do not interfere with those of other units.
This allows for independence, efficient storage,
and easy composition of partial models of distinct units.

\reva{
In the following,
we first define a decomposition of a \hex-program
into evaluation units
that are organized in an evaluation graph
(Section~\ref{sec:evalgraph}).
Then we define an interpretation graph which contains
input and output interpretations of each evaluation unit
(Section~\ref{sec:interpretationgraph}).
We next extend this definition to answer set graphs
which are related with answer sets of the program
(Section~\ref{sec:answersetgraph}).
Finally Section~\ref{sec:evaluation}
uses these definitions in an algorithm
for enumerating answer sets of the \hex-program.
}

\subsection{Evaluation Graph}
\label{sec:evalgraph}
Using
rule dependencies, we next define the notion of evaluation graph
on evaluation units.
We then relate evaluation graphs to
splitting sets \cite{lifs-turn-94}
and show how
to use them
to evaluate \hex-programs
by evaluating
units and combining the results.

We define evaluation units as follows.

\begin{definition}
  \label{def:evalunit}
An \emph{evaluation unit} (in short \quo{unit})
is any lde-safe \hex-program. %
\end{definition}
The formal definition of lde-safety (see \ref{sec:liberalsafety} and
\citeN{eite-etal-14a}) is not crucial here, merely the property that a
unit has a finite grounding with the same answer sets as the original
unit which can be effectively computed; lde-safe \hex-programs are
the most general class of \hex-programs with this property and
computational support.

An important point of the
notion of evaluation graph
is that rule dependencies $r \rightarrow_x s$ lead to different
edges, i.e., unit dependencies,
depending on the dependency type $x \in \{n,m\}$ and
whether $r$ resp.\ $s$ is a constraint;
constraints cannot (directly) make atoms true,
hence they can be shared between units in certain cases,
while sharing non-constraints could violate modularity.

Given a rule $r \in P$ and a set $U$ of units,
we denote by $U|_r= \{ u \in U \mid r \in u \}$
the set of units that contain rule $r$.
\begin{definition}[Evaluation graph]
  \label{def:evalgraph}
An \emph{evaluation graph} $\cE = (U,E)$ of a program~$P$
is a directed acyclic graph whose vertices $U$ are evaluation units and
which fulfills the following properties:
  \begin{enumerate}[(a)]
  \item
    \label{itm:coverprog}
    $P = \bigcup_{u \ins U} u$,
    i.e., every rule $r \in P$ is contained in at least one unit;
  \item
    \label{itm:ruleunique}
    every non-constraint $r \in P$ is contained in exactly one unit,
    i.e., $\big| U|_r \big| = 1$;
  \item
    \label{itm:covernonmon}
    for each nonmonotonic dependency $r \dependsnmon s$ between rules $r$, $s \in P$
    and for all $u \in U|_r$, $v \in U|_s$, $u \neq v$,
    there exists an edge $(u,v) \in E$
    (intuitively,
    nonmonotonic dependencies between rules
    have corresponding edges everywhere in $\cE$); and
  \item
    \label{itm:covermon}
    for each monotonic dependency $r \dependsmon s$ between rules $r$, $s \in P$,
    there exists some
$u \in U|_r$ such that $E$
    contains all edges $(u,v)$ with $v \in U|_s$ for $v \neq u$
    (intuitively,
    for each rule $r$ there is (at least) one unit in $\cE$
    where all monotonic dependencies from $r$
    to other rules have corresponding outgoing edges in $\cE$).
  \end{enumerate}
\end{definition}

We remark that
\citeN{2011_pushing_efficient_evaluation_of_hex_programs_by_modular_decomposition}
and \citeN{ps2012} defined evaluation units as \emph{extended
  pre-groundable} \hex-programs; later, \citeN{r2014-phd} and
\citeN{eite-etal-14a} defined {\em generalized evaluation units} as
lde-safe \hex-programs, which subsume \emph{extended pre-groundable}
\hex-programs, and \emph{generalized evaluation graphs} on top as in
Definition~\ref{def:evalgraph}. As more the grounding properties of
units matter than the precise fragment, we dropped here \quo{generalized} to
avoid complex terminology.

As a non-constraint can occur only in a single unit,
the above definition implies that
all dependencies of non-constraints have corresponding edges in $\cE$,
which is formally expressed in the following proposition.
\ifrevisionmarkers %
\else
\newpage
\fi
\begin{proposition}
  \label{thm:evalgraphruledeps}
Let $\cE=(U,E)$ be an evaluation graph of a program $P$, and
assume $r \dependsmn s$ is a dependency between a non-constraint $r \in P$ and a rule $s \in P$.
Then $\{ (u,v) \mid u \in U|_r, v \in U|_s\} \subseteq E$ holds.
\end{proposition}
\newcommand\myproofEvalGraphRuleDeps{
  Assume towards a contradiction that there exist a non-constraint $r \in P$,
  a rule $s \in P$ with $r \dependsmn s$,
  and $u' \in U|_r$,
  $v' \in U|_s$ such that $(u',v') \notin E$.
  Due to Definition~\ref{def:dependencies},
  $r \dependsmn s$ implies that $s$ has
  $H(s) \neq \emptyset$ and therefore that $s$ is a non-constraint.
  Definition~\ref{def:evalgraph}~\eqref{itm:ruleunique} then implies that
  $U|_r = \{ u' \}$ and $U|_s = \{ v' \}$
  (non-constraints are present in exactly one unit).

  Case (i):
  for $r \dependsnmon s$,
  Definition~\ref{def:evalgraph}~\eqref{itm:covernonmon} specifies that
  for all $u \in U|_r$ and $v \in U|_s$ there exists an edge $(u,v) \in E$,
  therefore also $(u',v') \in E$, which is a contradiction.

  Case (ii):
  for $r \dependsmon s$,
  Definition~\ref{def:evalgraph}~\eqref{itm:covermon} specifies that
  some $u \in U|_r$ exists such that
  for every $v \in U|_s$ there exists an edge $(u,v) \in E$;
  since $U|_r = \{ u'\}$ and $U|_s = \{ v' \}$,
  it must hold that $(u',v') \in E$,
  which is a contradiction.
}%
\myinlineproof{%
  \begin{proof}[Proof of Proposition~\ref{thm:evalgraphruledeps}]
    \myproofEvalGraphRuleDeps
  \end{proof}
}

\myfigureOldEvalGraph{tb}
\myfigureBetterEvalGraph{tb}

\begin{example}[ctd.]
  Figures~\ref{fig:exOldEvalStrat} and~\ref{fig:exBetterEvalStrat}
  show two possible evaluation graphs for our running example.
The evaluation graph  $\cE_1$
contains every rule of $\myPswim$ in exactly one unit.
  In contrast,
  $\cE_2$ contains $c_8$ both in $u_2$ and in $u_4$.
  Condition~\eqref{itm:covermon}  of Definition~\ref{def:evalgraph}
  is particularly interesting for these two graphs;
  it is fulfilled as follows.
  Graph $\cE_1$ can be obtained by contracting rules in the rule
  dependency graph $DG(\myPswim)$ into units, i.e., $\cE_1$ is a (graph)
  minor of $DG(\myPswim)$ and therefore all rule dependencies are
  realized as unit dependencies and Conditions~\eqref{itm:covernonmon}
  and~\eqref{itm:covermon} are satisfied.  In contrast, $\cE_2$ is not a
  minor of $DG(\myPswim)$ because dependency $c_8 \dependsmon r_5$ is
  not realized as a dependency from $u_2$ to $u_4$. Nonetheless, all
  dependencies from $c_8$ are realized at $u_4$ and thus
  $\cE_2$ conforms with condition~\eqref{itm:covermon}, which
 merely requires that rule dependencies have edges corresponding to all
 monotonic rule dependencies at \emph{some} unit of the evaluation  graph.  \qed
\end{example}

Evaluation graphs have the important property
that partial models of evaluation units do not intersect,
i.e., evaluation units do not mutually depend on each other.
This is achieved by
acyclicity and
because rule dependencies are covered in the graph.

In fact, due to acyclicity,
mutually dependent rules of a program are contained in the same unit;
thus each strongly connected component of the program's dependency graph
is fully contained in a single unit.
Furthermore,
a unit can have in its rule heads only
atoms that do not unify with atoms in the rule heads of other units,
as rules which have unifiable heads mutually depend on one another.
This ensures that under any grounding,
the following property holds.
\begin{proposition}[Disjoint unit outputs]
  \label{thm:disjointunitoutputs}
Let $\cE \eqs (U,E)$ be an evaluation graph of a program $P$. Then
for each distinct units $u_1,u_2 \ins U$, it holds that $\gh{u_1} \caps \gh{u_2} \eqs \emptyset$.%
  \footnote{See page~\pageref{pos:ghdefinition} for the definition of notation $\gh{P}$.}
\end{proposition}
\newcommand\myproofDisjointOutputModels{
  Given two distinct units $u_1, u_2 \in U$,
  assume towards a contradiction that some
  $\gamma \in \gh{u_1} \cap \gh{u_2}$  exists.
  Then there exists some $r \in u_1$
  with $\alpha \in H(r)$ and $\alpha \sim \gamma$,
  and there exists some $s \in u_2$
  with $\beta \in H(s)$ and $\beta \sim \gamma$.
  As $\alpha \sim \gamma$ and $\beta \sim \gamma$ and $\gamma$ is ground,
  we obtain $\alpha \sim \beta$;
  hence, by Definition~\ref{def:dependencies}~\eqref{itm:sim}
  we have $r \dependsmon s$ and $s \dependsmon r$.
  As $r$ and $s$ have nonempty heads, they are non-constraints.
  Thus by Proposition~\ref{thm:evalgraphruledeps},
  there exist edges $(u_1,u_2), (u_2,u_1) \in E$.
  As an evaluation graph is acyclic, it follows $u_1 = u_2$; this is a contradiction.
}%
\myinlineproof{%
  \begin{proof}
    \myproofDisjointOutputModels
  \end{proof}
}
\begin{example}[ctd.]
  Figures~\ref{fig:exOldEvalStrat} and~\ref{fig:exBetterEvalStrat}
  show for each unit which atoms can become true
  due to rule heads in them, denoted as \quo{derived} atoms.
  Observe
  that both graphs have
  strictly non-intersecting atoms in rule heads of distinct units.
  \qed
\end{example}
As units of evaluation graphs can be arbitrary lde-safe programs, we
clearly have the following property.
\begin{proposition}
  \label{thm:evalgraphcoversdomainexpansionsafe}
  For every lde-safe \hex program $P$,
  some evaluation graph $\cE$ exists.
\end{proposition}
\newcommand\myproofEvalGraphCoversDExpansionSafe{
  For an lde-safe program $P$,
  the graph $\cE = (\{P\},\emptyset)$ is a valid evaluation graph.
}%
\myinlineproof{%
  \begin{proof}[Proof of Proposition~\ref{thm:evalgraphcoversdomainexpansionsafe}]
    \myproofEvalGraphCoversDExpansionSafe
  \end{proof}
}
Indeed, we can simply put $P$ into a single unit to obtain a valid
evaluation graph. Thus the \hex evaluation approach based on evaluation
graphs is applicable to all domain-expansion safe \hex programs.

\subsubsection{Evaluation Graph Splitting}

We next show that units and their predecessors in an evaluation graph
correspond to generalized bottoms.
We then use this property to formulate an algorithm
for unit-based, efficient evaluation of \hex-programs. %

Given an evaluation graph $\cE = (U,E)$, we write
$u < w$, if a path from $u$ to $w$ exists in $\cE$, and
$u \leq w$ if either $u < w$ or $u = w$.

For a unit $u \in U$, we denote by $\myinputs(u) = \{ v \in U \mid (u,v)\in E \}$
the set of units on which $u$ (directly) depends and by
$u^< = \bigcup_{w \in U, u < w} w$
the set of rules in all units on which $u$ transitively depends; furthermore,
we let $u^\leq = u^< \cup u$.
Note that for a leaf unit $u$ (i.e., $u$ has no predecessors)
we have $\myinputs(u)=u^< = \emptyset$ and $u^\leq = u$.
\begin{theorem}
  \label{thm:evalgraphprecbottom}
For every evaluation graph $\cE=(U,E)$ of a \hex-program $Q$
and unit $u \in U$, it holds that $u^<$
is a generalized bottom of $u^\leq$
wrt.\
$R = \{ r \in u^<
\mid  \reva{H}(r)\neq \emptyset \}$. %
\end{theorem}
\newcommand\myproofevalgraphprecbottom{
 For any set of rules, let $\mi{constr}(S)= \{ r\in S \mid
 \reva{H}(r)=\emptyset \}$ denote the set of constraints in $S$.
  We say that the
  \emph{dependencies of $r \in Q$ are covered at unit $u \in U$},
  if
  for every rule $s \in Q$ such that $r \dependsmn s$ and $s \notin u$,
  it holds that $(u,u') \in E$ for all $u' \in U|_s$, i.e.,
  $u$ has an edge to all units containing $s$.

  To prove that $B = u^<$ is a generalized bottom of $P = u^\leq$
  wrt.\ the rule splitting set $R = u^< \setminus \mi{constr}(u^<)$
  as by Definition~\ref{def:genbottom}, we prove that
  \begin{inparaenum}[(a)]
  \item
    $R \subseteq B \subseteq P$,
  \item
    $B \setminus R$ contains only constraints,
  \item
    no constraint in $B \setminus R$
    has nonmonotonic dependencies to rules in $P \setminus B$, and
  \item
    $R$ is a rule splitting set of $P$.
  \end{inparaenum}

  Statement (a) corresponds to
  $u^< \setminus \mi{constr}(u^<) \subseteq u^< \subseteq u^\leq$
  and $u^\leq$ is defined as $u^\leq = u^< \cup u$,
  therefore the relations all hold.
  For (b), $B \setminus R = u^< \setminus (u^< \setminus \mi{constr}(u^<))$,
  and as $A \setminus (A \setminus B) = A \cap B$,
  it is easy to see that $B \setminus R = u^< \cap \mi{constr}(u^<)$
  and thus $B \setminus R$ only contains constraints.
  For (c),
  we show a stronger property,
  namely that no rule (constraint or non-constraint)
  in $B$ has nonmonotonic dependencies to rules in $P \setminus B$.
  $B = u^<$ is the union of evaluation units $V = \{ v \in U \mid v < u \}$.
  By Definition~\ref{def:evalgraph}~\eqref{itm:covernonmon}
  all nonmonotonic dependencies $r \dependsneg s$ are covered at every
  unit $w$ such that $w\in U_r$.
  Hence if $r\in w$ and $w \in V$, then either
  $s \in w$ or $s \in w^<$ holds, and hence $s\in w^\leq \subseteq u^<$.
  As $P \setminus B = u^\leq \setminus u^<$,
  no nonmonotonic dependencies
  from $B=u^<$ to $P \setminus B$ exist and (c) holds.
  For (d) we know that $R = u^< \setminus \mi{constr}(u^<)$
  contains no constraints,
  and by Proposition~\ref{thm:evalgraphruledeps}
  all dependencies of non-constraints in $R$ are covered by $\cE$.
  Therefore $r \in R$, $r \dependsmn s$, and $s \in P$
  implies that $s \in R$.
  Consequently, (d) holds which proves the theorem.
}%
\myinlineproof{%
  \begin{proof}%
    \myproofevalgraphprecbottom
  \end{proof}
}

\begin{example}[ctd.]
  \label{ex:evalgraphsthmprecbottom}
In $\cE_1$, $u_2^< = u_1$ and $u_2^\leq = u_1 \cups
  u_2$ and $u_2^<$ is a generalized bottom of $u_2^\leq$ wrt.\ $R=\{r_1,
  r_{\reva{3}}, r_4\}$.
  In $\cE_2$,
  we have $u_4^< = u_1 \cups u_2 \cups u_3$
  and $u_4^\leq = \myPswim$
  and %
  $u_4^<$ is a generalized bottom of $\myPswim$
  wrt.\
 $R = \{ r_1,r_2,r_3,r_4 \}$.
  We can verify this on Definition~\ref{def:genbottom}:
  we have $P = \myPswim$,
  $B = u_4^< = \{ r_1,r_2,r_3,r_4,c_6,c_7,c_8 \}$,
  and $R$ as above.
  Then $R \subseteqs B \subseteqs P$, and
  furthermore $B \setminus R = \{ c_6,c_7,c_8 \}$
  consists of constraints none of which
  depends nonmonotonically on a rule in
  $P \setminus B = \{ r_5 \}$.
  \qed
\end{example}
\begin{theorem}
  \label{thm:evalgraphpredbottom}
  Let $\cE=(U,E)$ be an evaluation graph of a \hex-program $Q$ and
  $u \in U$. Then for every unit $u' \in \myinputs (u)$,
  it holds that $u'^\leq$
  is a generalized bottom of the subprogram $u^<$
  wrt.\ the rule splitting set $R = \{ r \in u'^\leq \mid \reva{H}(r)\neq\emptyset \}$.
\end{theorem}
\newcommand\myproofevalgraphpredbottom{
  Similar to the proof of Theorem~\ref{thm:evalgraphprecbottom},
  we show this in four steps;
  given $P = u^<$,
  $R = u'^{\leq} \setminus \mi{constr}(u'^{\leq})$,
  and
  $B = u'^{\leq} = u' \cup u'^<$,
  we show that
  \begin{inparaenum}[(a)]
  \item
    $R \subseteq B \subseteq P$,
  \item
    $B \setminus R$ contains only constraints,
  \item
    no constraint in $B \setminus R$
    has nonmonotonic dependencies to rules in $P \setminus B$, and
  \item
    $R$ is a rule splitting set of $P$.
  \end{inparaenum}
  Let $\myinputs(u) = \{u_1,\ldots,u_k\}$ and
  Let $V = \{ v \in U \mid v < u' \}$ be the set of units
  on which $u'$ transitively depends.
  (Note that $V \subset \myinputs(u)$ and $u \notin V$.)
  As $u'^<$ contains all units $u'$ transitively depends on,
  we have $B = u' \cup \bigcup_{w \in V} w$.

  For (a), $R \subseteq B$ holds trivially,
  and $B \subseteq P$ holds by definition of $u^<$ and $u'^\leq$
  and because $u' \in \myinputs(u)$.
  Statement (b) holds, because $B \setminus R$ removes $R$ from $B$,
  i.e., it removes everything that is not a constraint in $B$ from $B$,
  therefore only constraints remain.
  For (c) we show that no rule in $B$
  has a nonmonotonic dependency to rules in $P \setminus B$.
  By Definition~\ref{def:evalgraph}~\eqref{itm:covernonmon},
  all nonmonotonic dependencies are covered at all units.
  Therefore a rule $r \in w$, $w \in \{ u' \} \cup V$
  with $r \dependsnmon s$, $s \in U$
  implies that either
  $s \in w$, or that
  $s$ is contained in a predecessor unit of $w$
  and therefore in $u'$ or in $V$.
  Hence there are no nonmonotonic dependencies from rules in $B$
  to any rules not in $B$,
  and hence also not to rules in $P \setminus B$ and (c) holds.
  For (d) we know that $R$ contains no constraints
  and by Proposition~\ref{thm:evalgraphruledeps}
  all dependencies of non-constraints in $R$ are covered by $\cE$.
  Therefore $r \in R$, $r \dependsmn s$, $s \in P$ implies that $s \in R$
  and the theorem holds.
}%
\myinlineproof{%
  \begin{proof}%
    \myproofevalgraphpredbottom
  \end{proof}
}
\begin{example}[ctd.]
  \label{ex:evalgraphspredbottom}
In $\cE_1$, we have $u_1 \in \myinputsE{\cE_1}(u_2)$;
hence  $u_1^\leq = u_1$ is by Theorem~\ref{thm:evalgraphpredbottom}
a generalized bottom of $u_2^< = u_1$  wrt.\   %
$R = \{ r_1,r_3,r_4\}$. Furthermore, $u_2 \in \myinputsE{\cE_1}(u_3)$ and  hence  $u_2^\leq = u_1 \cups u_2$
  is a generalized bottom of  $u_3^< = u_1 \cups u_2$  wrt.\
  $R = \{ r_1,r_2,r_3,r_4,r_5\}$.
The case of $\cE_2$ and $u_4$ is less clear.
  We have $u_2 \in \myinputsE{\cE_2}(u_4)$,
  thus by Theorem~\ref{thm:evalgraphpredbottom}
  $u_2^\leq = u_1 \cups u_2 = \{ r_1,r_2,c_8 \}$
  is a generalized bottom of
  $u_4^< = u_1 \cups u_2 \cups u_3$
  wrt.\
  $R = \{r_1,r_2\}$.
 Comparing against Definition~\ref{def:genbottom},
 we have
  $P = u_1 \cups u_2 \cups u_3$
  and
  $B = u_1 \cups u_2$;
 thus indeed $R \subseteqs B \subseteq P$
 and no constraint in $B \setminus R = \{c_8\}$
depends nonmonotonically on any rule in
  $P \setminuss B = \{ r_3,r_4,c_6,c_7 \}$.
  \qed
\end{example}

\subsubsection{First Ancestor Intersection Units}
\label{sec:fais}

\myfigureCAUexample{tb}

We will use the evaluation graph for model building; as syntactic
dependencies reflect semantic dependencies between units, multiple
paths between units \reva{require} attention. Of particular importance are
\emph{\mycaustext,} which are units where distinct
paths starting at some unit meet first. More formally,

\ifrevisionmarkers %
\newpage
\fi
\begin{definition}
  Given an evaluation graph $\cE=(U,E)$
  and units $v\neq w \in U$,
  we say that unit \emph{$w$ is a \mycautext (\myCAUtext) of $v$},
  if paths $p_1\neq p_2$ from $v$ to $w$ exist in $E$ that overlap only in $v$ and $w$.
 By $\mycau(v)$  we denote the set of all \myCAUstext of $v$.
\end{definition}
\begin{example}
  Figure~\ref{fig:cauexamples} sketches an evaluation graph with
  dependencies %
  $a \depends b \depends c \depends e \depends f$,
  $a \depends d \depends e \depends g$,
  and
  $b \depends d$.
  We have that $\mycau(a) = \{ d, e\}$,
 $\mycau(b) = \{ e \}$, and $\mycau(u) = \emptyset$ for each $u\in
 U\setminus \{a,b\}$.
  In particular,
  $f$ and $g$ are not \myCAUstext of $b$,
  because all pairs of distinct paths from $b$ to $f$ or $g$
  overlap in more than two units.
  \qed
\end{example}
Note that for tree-shaped evaluation graphs, $\mycau(v)=\emptyset$
for each unit $v$ as paths between nodes in a tree are unique.
\begin{example}[ctd.]
  The evaluation graph $\cE_1$
  of
  $\myPswim$
  is a tree  (see Fig.~\ref{fig:exOldEvalStrat}),
  thus $\mycau(u) = \emptyset$
  for $u \in \{u_1,u_2,u_3\}$.
  In contrast, the evaluation graph $\cE_2$ of $\myPswim$
  (see Fig.~\ref{fig:exBetterEvalStrat})
  is not a tree;
  we have that $\mycau(u_4) = \{u_1\}$
  and no other unit in $\cE_2$ has \myCAUstext.
  \qed
\end{example}

We can build an evaluation graph $\cE$ for a program $P$ based on
the dependency graph $DG(P)$. Initially, the units are set to the
maximal strongly connected components of $DG(P)$, and then
units are iteratively merged while preserving acyclicity and the
conditions (a)-(d) of an evaluation graph; we will
discuss some existing heuristics in Section~\ref{sec:heuristics}, while for details we refer to
\citeN{r2014-phd}.

\subsection{Interpretation Graph}
\label{sec:interpretationgraph}

We now define the Interpretation Graph (short \myigraph),
which is the foundation of our model building algorithm.
An \myigraph is a labeled directed graph
defined wrt.\ an evaluation graph, where
each vertex is associated with a specific evaluation unit,
a type (input resp.\ output interpretation) and
\reva{a} set of ground atoms.

We do not use interpretations themselves as vertices, as distinct
vertices may be associated with the same interpretation; still we call
vertices of the \myigraph interpretations.

\reva{Towards defining \myigraphs
we first define an auxiliary concept called {\em interpretation structure}.
We then define \myigraphs
as the subset of interpretation structures
that obey certain topological and uniqueness conditions.
Finally we present an example
(Example~\ref{ex:modelGraphBetterStrat} and Figure~\ref{fig:exModelGraphBetterStrat}).}
\begin{definition}[Interpretation Structure]
  \label{def:interpretationstructure}
  Let $\cE = (U,E)$ be an evaluation graph for a program $P$.
  An \emph{interpretation structure $\cI$ for $\cE$}
  is a directed acyclic graph $\cI = (M,F,\myunit,\mytype,\myint)$
  \reva{with nodes} $M \subseteq \cI_{\mi{id}}$
  from a countable set $\cI_{\mi{id}}$ of identifiers,
  \reva{edges $F \subseteqs M \timess M$,}
  and total node labeling functions
  $\myunit\colon M \to U$,
  $\mytype\colon M \to \{\scI,\, \scO\}$,
  and $\myint\colon M \to 2^{\HBP}$.
\end{definition}
The following notation will be useful.
Given unit $u \in U$ in the evaluation graph associated with an \myigraph $\cI$,
we denote by
$\myimodelsI(u) \eqs
  \{ m \ins M \mid
    \myunit(m) \eqs u$ $\text{and } \mytype(m) \eqs \scI \}$ the {input
    (i-)interpretations}\/, and by $\myomodelsI(u) =
  \{ m \ins M \mid$ $\myunit(m) \eqs u \text{ and } \mytype(m) \eqs \scO \}$
the {output (o-)interpretations}\/ of $\cI$ at unit $u$.
For every vertex $m \in M$, we denote by
$$
  \myint^+(m) = \myint(m) \cup \bigcup
    \{ \myint(m') \mid m' \in M \text{ and $m'$ is reachable from $m$ in $\cI$}\}
$$
the \emph{expanded interpretation} of $m$.

Given an interpretation structure $\cI = (M,F,\myunit,\mytype,\myint)$
for $\cE=(U,E)$ and a unit $u \in U$, we define the following properties:
\begin{enumerate}[(a)]
\item[(IG-I)]
  \emph{I-connectedness:}
  for every $m \ins \myomodelsI(u)$, \reva{it holds that}
  $|\{ m' \mids (m,m') \ins F\}| = \reva{ |\{ m' \ins \myimodelsI(u)
    \mids (m,m') \ins F\}| } \eqs  1$;
\item[(IG-O)]
  \emph{O-connectedness:}
  for every $m \ins \myimodelsI(u)$,
  \reva{$|\{ m_i \mids (m,m_i)$ $\ins F\}| \eqs |\myinputs(u)|$}
  and for every $u_i \ins \myinputs(u)$
  \reva{we have $|\{ m_i \ins \myomodelsI(u_i) \mids (m,m_i) \ins F\}| \eqs 1$};
\item[(IG-F)]
  \emph{\myCAUtext intersection:}
 let $\cE'$ be the subgraph  of $\cE$ on the units reachable from $u$%
\footnote{I.e., $\cE'$ is the subgraph of $\cE$ induced by the set of
 units reachable from $u$, including $u$; in abuse of terminology, we
 briefly say \quo{the subgraph (of $\cE$) reachable from}}
 and for every $m \in \myimodelsI(u)$, let $\cI'$ be the subgraph of $\cI$ reachable from $m$.
  Then
  $\cI'$ contains exactly one \myoint at each unit of $\cE'$.
  (Note that both $\cI$ and $\cE$ are acyclic,
 hence $\cI'$ does not include $m$ and $\cE'$ does not include $u$.)
\item[(IG-U)]
  \emph{Uniqueness:}
  for every
  $m_1\neq m_2 \ins M$ such that
$\myunit(m_1) \,{=}$ $\myunit(m_2) = u$,
we have  $\myint^+(m_1) \neq \myint^+(m_2)$ (the expanded
interpretations differ).
\end{enumerate}

\begin{definition}[Interpretation Graph]
  \label{def:interpretationgraph}
  Let $\cE = (U,E)$ be an evaluation graph for a program $P$.
  then an \emph{interpretation graph (\myigraph)}
  for $\cE$
  is an interpretation structure $\cI = (M,F,\myunit,\mytype,\myint)$  that fulfills for every unit $u \in U$
  the conditions (IG-I), (IG-O), (IG-F), and (IG-U).
\end{definition}

Intuitively, the conditions make every \myigraph \quo{live}
on its associated evaluation graph:
an \myiint must conform to all dependencies of the unit it belongs to,
by depending on exactly one \myoint at that unit's predecessor units (IG-\reva{O});
moreover an \myoint must depend on exactly one \myiint at the same unit (IG-\reva{I}).
Furthermore,
every \myiint depends directly or indirectly
on exactly one \myoint at each unit it can reach in the \myigraph (IG-F);
this ensures that no expanded interpretation $\myint^+(m)$
\quo{mixes} two or more \myiints resp.\, \myoints from the same unit.
(The effect of condition (IG-F) is visualized in Figure~\ref{fig:goodbadcau}.)
Finally, redundancies in an \myigraph are ruled out by the uniqueness condition (IG-U).

\myfiguregoodbadcau{btp}
\myfigureModelGraphBetterStrat{tb}
\begin{example}[ctd.]
  \label{ex:modelGraphBetterStrat}
  Figure~\ref{fig:exModelGraphBetterStrat} %
  shows an interpretation graph $\cI_2$ for $\cE_2$.
 The $\myunit$ label is depicted as \reva{dashed} rectangle
  labeled with the respective unit.
  The $\mytype$ label is indicated after interpretation names,
  i.e., $m_1/\scI$ denotes that interpretation $m_1$ is an input interpretation.
  For $\cI_2$ the set $\cI_\mi{id}$ of identifiers is $\{ m_1,\ldots,m_{15}\}$.
  \reva{T}he symbol \Lightning{}\
  \reva{in a unit $u$ pointing to an \myiint $m$ indicates
 that there is no \myoint wrt.\ input $m$ of unit $u$.
  Section~\ref{sec:evaluation} describes
  an algorithm for building an \myigraph given an evaluation graph.
  }

  Dependencies are shown as arrows between interpretations.
  Observe that %
  I-connec\-tedness \reva{(IG-I)} is fulfilled,
  as every \myoint depends on exactly one \myiint at the same unit.
  \reva{For example $m_9$ and $m_{10}$ depend on $m_7$.}
  O-connectedness \reva{(IG-O)} is similarly fulfilled,
  in particular consider \myiints of $u_4$ in $\cI_2$:
  $u_4$ has two predecessor units ($u_2$ and $u_3$)
  and every \myiint at $u_4$ depends on exactly one \myoint
  at $u_2$ and exactly one \myoint at $u_3$.
  The condition on \myCAUtext intersection \reva{(IG-F)}
  could only be violated by \myiints at $u_4$\reva{,
  concretely it would be violated if two different
  \myoints are reachable at $u_1$ from one \myiint at $u_4$}.
  We can verify that from both $m_{13}$ and $m_{14}$
  we can reach exactly one \myoint at each unit;
  hence the condition is fulfilled.
  \reva{An example for a violation
  would be an \myiint at $u_4$
  that depends on $m_6$ and $m_9$:
  in this case we could reach two distinct \myoints
  $m_2$ and $m_3$ at $u_1$, thereby violating (IG-F).}
  Uniqueness (IG-U) is satisfied,
  as in both graphs no unit has two output models with the same content.
  \qed
\end{example}

Note that the empty graph is an \myigraph.
This is by intent, as our model building algorithm will progress
from an empty \myigraph to one with interpretations at every unit,
precisely if the program has an answer set.

\subsubsection{Join}

We will build \myigraphs by adding one vertex at a time,
always preserving the \myigraph conditions.
Adding an \myoint requires to add a dependency to one \myiint at the same unit.
Adding an \myiint similarly requires addition of dependencies.
However this is more involved
because condition (IG-F) could be violated.
Therefore, we next define an operation that captures all necessary conditions.

We call the combination of \myoints which yields an \myiint a \quo{\emph{join}}.
Formally, the join operation \quo{$\join$} is defined as follows.
\begin{definition}
  \label{def:join}
  Let $\cI=(M,F,\myunit,\mytype,\myint)$ be an \myigraph for
  an evaluation graph $\cE=(V,E)$ of a program $P$.
  Let $u \in V$ be a unit, %
  let $\myinputs(u)= \{u_1,\ldots,u_k\}$ be the predecessor units of $u$,
  and let $m_i\in \myomodelsI(u_i)$, $1 \leq i \leq k$,
  be an \myoint at $u_i$.
  Then \emph{the join
  $m_1 \join \dotsb \join m_k = \bigcup_{1 \leq i \leq k} \myint(m_i)$
  at $u$ is defined}
  iff
  for each $u' \in \mycau(u)$
  the set of \myoints at $u'$ that are reachable (in $F$) from some \myoint
  $m_i$, $1 \leq i \leq k$, contains exactly one \myoint $m' \in \myomodelsI(u')$.
\end{definition}
Intuitively,
a set of interpretations can only be joined
if all interpretations depend on the same (and on a single) interpretation at every unit.
\begin{example}[ctd.]
  \label{ex:join}
  In $\cI_2$,
  \myiints $m_1$, $m_4$, $m_5$, $m_7$, and $m_8$
  are created by trivial join operations with none or one predecessor unit.
  For $m_{13}$ and $m_{14}$, we have a nontrivial join:
  $\myint(m_{13}) = \myint(m_6) \cup \myint(m_{11})$
  and the join is defined because $\mycau(u_4) = \{ u_1 \}$,
  and from $m_{6}$ and $m_{11}$ we can reach in $\cI_2$
  exactly one \myoint at $u_1$.
  Observe that the join $m_6 \join m_9$ is not defined,
  as we can reach in $\cI_2$ from $\{ m_6, m_9 \}$
  the  \myoints $m_2$ and $m_3$ at $u_1$, and thus
  more than exactly one \myoint at some \myCAUtext of $u_4$.
Similarly, the join $m_6 \join m_{10}$ is undefined,
as we can reach $m_2$ and $m_3$ at $u_1$.
  \qed
\end{example}
The result of a join is the union of predecessor interpretations; this
is important for answer set graphs and join operations on them,
which comes next.
Note that each leaf unit
(i.e., without predecessors)
has exactly one well-defined join result, viz.\ $\emptyset$.

If we add a new \myiint from the result of a join operation
to an \myigraph and dependencies to all participating \myoints, the
resulting graph is again an \myigraph; thus the join is sound wrt.\ to the \myigraph properties.
Moreover, each \myiint that can be added to \reva{an} \myigraph while
preserving the \myigraph conditions can be
synthesized by a join; that is, the join is complete for such additions.
This is a consequence of the following result.
\begin{proposition}
  \label{thm:joinmodelgraphsoundcomplete}
  Let $\cI \eqs (M,F,\myunit,\mytype,\myint)$
  be an \myigraph
  for an evaluation graph $\cE \eqs (V,$ $E)$ and
  $u \in V$ with $\myinputs(u) \eqs \{u_1,\ldots,u_k\}$.
  Furthermore, let
  $m_i \ins \myomodelsI(u_i)$, $1 \leq i \leq k$,
 such that no vertex $m \in \myimodelsI(u)$ exists such that $\{ (m,m_1),
  \ldots,$ $(m,m_k) \} \subseteq F$.
  Then the join $J = m_1 \join \dotsb \join m_k$ is defined at $u$
  iff
  $\cI' = (M',\allowbreak{}F',\allowbreak{}\myunit',\allowbreak{}\mytype',\allowbreak{}\myint')$
  is an \myigraph for $\cE$
  where
  \begin{inparaenum}[(a)]
  \item
    $M' = M \cup \{m'\}$ for some new vertex $m' \ins \cI_{\mi{id}} \setminuss M$,
  \item
    $F' = F \cup \{(m',m_i)\mid 1 \leq i \leq k\}$,
  \item
    $\myunit' = \myunit \cup \{ (m',u) \}$,
  \item
    $\mytype' = \mytype \cup \{ (m',\scI) \}$, and
  \item
    $\myint' = \myint \cup \{ (m',J) \}$.
  \end{inparaenum}
\end{proposition}
\newcommand\myproofJoinModelGraphSoundComplete{
  ($\Rightarrow$)
  The added vertex $m'$ is assigned to one unit and gets assigned a type.
  Furthermore, the graph stays acyclic as only outgoing edges from $m'$
  are added.
  I-connectedness is satisfied,
  as it is satisfied in $\cI$ and we add no \myoint.
  O-connected\-ness is satisfied,
  as $m'$ gets appropriate edges to \myoints at its predecessor units,
  and for other \myiints it is already satisfied in $\cI$.

  For \myCAUtext intersection,
  observe that if we add an edge $(m',m_i)$ to $\cI$ and it holds that $m_i \ins \myomodelsI(u_i)$,
  then $m'$ reaches in $\cI$ only one \myoint at $u_i$,
  and due to O-connec\-ted\-ness that \myoint is connected
  to exactly one \myiint at $u_i$,
  which is part of the original graph $\cI$
  and therefore satisfies \myCAUtext intersection.
  Therefore it remains to show that the union of subgraphs of $\cI$
  reachable in $\cI$ from $m_1$,\ldots,$m_k$,
  contains one \myoint at each unit
  in the subgraph of $\cE$ reachable from $u_1$,\ldots,$u_k$.
  We make a case distinction.

  Case (I):
  two \myoints $m_i \in \myomodelsI(u_i)$,
  $m_j \in \myomodelsI(u_j)$ in the join,
  with $1 \leq i < j \leq k$,
  have no common unit that is reachable in $\cE$ from $u_i$ and from $u_j$:
  then the condition is trivially satisfied,
  as the subgraphs of $\cI$ reachable in $\cI$
  from $m_i$ and $m_j$, respectively, do not intersect at any unit.

  Case (II):
  two \myoints $m_i \in \myomodelsI(u_i)$,
  $m_j \in \myomodelsI(u_j)$ in the join,
  with $1 \leq i < j \leq k$,
  have at least one common unit
  that is reachable from $u_i$ and from $u_j$ in $\cE$.
  Let $u^f$ be a unit reachable in $\cE$ from  both $u_i$ and $u_j$
  on two paths that do not intersect before reaching $u^f$.
  From $u_i$ to $u^f$, and from $u_j$ to $u^f$,
  exactly one \myoint is reachable in $\cI$ from $m_i$ and $m_j$, respectively,
  as these paths do not intersect.
  $u^f$ is a \myCAUtext of $u$,
  and as the join is defined,
  we reach in $\cE$ exactly one \myoint at unit $u^f$
  from $m_i$ and $m_j$.
  Due to O-connectedness,
  we also reach in $\cI$ exactly one \myiint $m''$ at $u^f$
  from $m_i$ and $m_j$.
  Now $m''$ is common to subgraphs of $\cI$
  that are reachable in $\cI$ from $m_i$ and $m_j$,
  and $m''$ satisfies \myCAUtext intersection in $\cI$.

  Consequently, \myCAUtext intersection is satisfied in $\cI'$ for all pairs
  of predecessors of $m'$ and therefore in all cases.
  As no vertex $m$ with $\{ (m,m_1), \ldots, (m,m_k) \} \subseteq F$ exists
  and and as $\cI$ satisfies Uniqueness,
  also  $\cI'$ satisfies Uniqueness.

  ($\Leftarrow$)
  Assume towards a contradiction that $\cI'$ is an \myigraph
  but that the join is not defined.
  Then there exists some \myCAUtext $u' \in \mycau(u)$
  such that either no or more than one \myoint from $\myomodelsI(u)$
  is reachable in $\cI$ from some $m_i$, $1 \leq i \leq k$.
  As $\cI$ is an \myigraph, due to I-connec\-ted\-ness and O-connec\-ted\-ness,
  if a unit $u'$ is a \myCAUtext and therefore
  $u'$ is reachable in $\cE$ from $u_i$,
  then at least one \myiint and one \myoint at $u'$
  is reachable in $\cI$ from $m_i$.
  If more than one \myoint is reachable in $\cI$
  from some $m_i$, $1 \leq i \leq k$,
  this means that more than one \myoint at $u'$ is
  reachable in $\cI'$ from
  the newly added \myiint $m$.
  However, this violates \myCAUtext intersection in $\cI'$,
  which is a contradiction. Hence the result follows.
}%
\myinlineproof{%
  \begin{proof}
    \myproofJoinModelGraphSoundComplete
  \end{proof}
}

Note that the \myigraph definition specifies
topological properties of an \myigraph wrt.\ an evaluation graph.
In the following we extend this specification to the contents of interpretations.
\subsection{Answer Set Graph}
\label{sec:answersetgraph}
We next restrict \myigraphs to  \emph{answer set graphs}
such that interpretations correspond with answer sets
of certain \hex programs that are induced by the evaluation graph.
\begin{definition}[Answer Set Graph]
  \label{def:answersetgraph}
An \emph{answer set graph $\cA = (M,F,\myunit,\mytype,\myint)$} for
an evaluation graph $\cE =(U,E)$ is an \myigraph for $\cE$ such that for each unit $u \in U$,
  it holds that
  \begin{enumerate}[(a)]
\itemsep=0pt
  \item
    \label{itm:imodelexpansion}
 $\{ \myint^+(m) \mid m \ins \myimodelsI(u) \} \subseteq \AS(u^<)$,
 i.e., every expanded \myiint at $u$ is an answer set of $u^<$;

  \item
    \label{itm:omodelexpansion}
    \reva{$\{$}$\myint^+(m) \mid  m \ins \myomodelsI(u)\} \subseteq \AS(u^{\leq})$,
    i.e., every expanded \myoint at $u$ is an answer set of $u^{\leq}$; and
  \item
    \label{itm:imodelunion}
 for each %
 $m \in \myimodelsI(u)$, it holds that
    $\myint(m) = \bigcup_{(m,m_i) \in F} \myint(m_i)$.
  \end{enumerate}
\end{definition}
Note that each leaf unit $u$,
has $u^{<} = \emptyset$,
and thus  $\emptyset$ is the only \myiint possible.
Moreover, condition~\eqref{itm:imodelunion} is necessary to ensure
that an \myiint at unit $u$ contains all atoms of answer sets of predecessor units
that are relevant for evaluating $u$.
Furthermore,
note that the empty graph is an answer set graph.

\begin{example}[ctd.]
\label{ex:answersetgraph}
  The example \myigraph $\cI_2$ is in fact an answer set graph.
  First,
  $\myint^+(m_1) = \emptyset$
  and $u_1^< = \emptyset$
  and indeed $\emptyset \in \AS(\emptyset)$
  which satisfies condition~\eqref{itm:imodelexpansion}.
  Less obvious is the case of \myoint $m_6$ in $\cI_2$:
  $\myint^+(m_6) =
    \{ \mygoinout(\myoutdoor) \}$
  and
  $u_2^\leq = \{ r_1, r_2, c_8 \}$;
  as $c_8$ kills all answer sets where money is required,
  $\AS(\{r_1,r_2,c_8\}) = \{ \{ \mygoinout(\myoutdoor) \} \}$;
 hence $\myint^+(m_6)$ is the only expanded interpretation of an \myoint
 possible at $u_2$.
 Furthermore, the condition (IG-U) on \myigraphs
implies that $m_6$ is the only possible \myoint at $u_2$.
  Consider next $m_{13}$:
  \begin{align*}
  u_4^< &= \{ r_1, r_2, r_3, r_4, c_6, c_7, c_8 \}\text{ and} \\
  \myint^+(m_{13}) &=
    \{ \mygosomewhere, \mygolocation(\mypooln),
    \myngolocation(\mypoolg),  \mygoinout(\myoutdoor) \}.
  \end{align*}
  The two answer sets of $u_4^<$ are
$\{ \mygosomewhere, \mygolocation(\mypooln), \myngolocation(\mypoolg), \mygoinout(\myoutdoor) \},$
and $\{ \mygosomewhere,$ $\mygolocation(\mypoolg),$ $\myngolocation(\mypooln), \mygoinout(\myoutdoor) \}$,
and $\myint^+(m_{13})$ is one of them;
the other one is $\myint^+(m_{14})$.
Finally
  \begin{align*}
  \myint^+(m_{15}) =
  \{\mygoinout(\myoutdoor),
    \mygolocation(\mypooln),
    \mygosomewhere,
    \myngolocation(\mypoolg),
    \myneed(\myloc,\myyogamat)\},
  \end{align*}
  which is the single answer set of
  $u_4^\leq = \myPswim$.
  \qed
\end{example}

Similarly as for i-graphs, the join is a sound and complete operation to
add \myiints to an answer set graph.
\begin{proposition}
  \label{thm:joinanswersetgraphsoundcomplete}
 Let $\cA=(M,F,\myunit,\mytype,\myint)$ be an answer set graph
  for an evaluation graph $\cE=(V,E)$ and
let $u \in V$ with $ \myinputs(u) = \{ u_1,\ldots,u_k \}$. Furthermore,
let $m_i \in \myomodelsA(u_i)$, $1 \leq i \leq k$,
such that
no $m \in \myimodelsA(u)$ with $\{ (m,m_1), \ldots, (m,m_k) \} \subseteq F$
exists.
  Then the join $J = m_1 \join \dotsb \join m_k$ is defined at $u$
  iff
  $\cA'=(M',F',\myunit',\mytype',\myint')$
  is an answer set graph for $\cE$
  where
  \begin{inparaenum}[(a)]
  \item
    $M' = M \cup \{m'\}$ for some new vertex $m' \ins \cI_{\mi{id}} \setminuss M$,
  \item
    $F' = F \cup \{ (m',m_i) \mid 1 \les i \les k\}$,
  \item
    $\myunit' = \myunit \cup \{ (m',u) \}$,
  \item
    $\mytype' = \mytype \cup \{ (m',\scI) \}$, and
  \item
    $\myint' = \myint \cup \{ (m',J) \}$.
  \end{inparaenum}
\end{proposition}
\newcommand\myproofJoinAnswerSetGraphSoundComplete{
  ($\Rightarrow$)
  Whenever the join is defined,
  $\cA'$ is an \myigraph by Proposition~\ref{thm:joinmodelgraphsoundcomplete}.
  It remains to show that $\myint(m')^+ \in \AS(u^<)$, and that
  $\cA'$ fulfills items  (a) and (c) of an answer set graph.
  By Theorem~\ref{thm:evalgraphpredbottom} we know that for each $u_i$,
  $u_i^\leq$ is a generalized bottom of $u^<$ wrt.\ the set $R_i = \{ r
  \in  u_i^\leq \mid B(r)\neq\emptyset \}$.
  For each $u_i$,
  therefore $Y \in \AS(u^<)$ iff
  $Y \in \AS(u^< \setminus R_i \cup \facts(X))$
  for some
  $X \in \AS(u_i^\leq)$.
  As $\cA$ is an answer set graph,
  for each $m_i$ we know that $\myint(m_i)^+ \in \AS(u_i^\leq)$;
  hence $Y \in \AS(u^<)$ if
  $Y \in \AS(u^< \setminus R_i \cup \myint(m_i)^+)$.
  Now from the evaluation graph properties we know that
  $u^< = u_1^\leq \cup \cdots \cup u_k^\leq$,
  and from the construction of $\myint(m')$ and its dependencies in $\cA'$
  we obtain that $\myint(m')^+ = \myint(m_1)^+ \cup \cdots \cup \myint(m_k)^+$.
  It follows that $\myint(m')^+ \in \AS(u^<)$,
  which satisfies condition~\eqref{itm:imodelexpansion}.
  Due to the definition of join,
  condition~\eqref{itm:imodelunion} is also satisfied
  and $\cA'$ is indeed an answer set graph.

  ($\Leftarrow$)
As $\cA'$ is an answer set graph, it is an \myigraph, and hence
by Proposition~\ref{thm:joinmodelgraphsoundcomplete}
$m=m_1\join\cdots\join m_k$ is defined.
}%
\myinlineproof{%
  \begin{proof}[Proof of Proposition~\ref{thm:joinanswersetgraphsoundcomplete}]
    \myproofJoinAnswerSetGraphSoundComplete
  \end{proof}
}

\begin{example}[ctd.]
Imagine that $\cI_2$ has no interpretations at $u_4$.
The following candidate pairs of \myoints exist for creating
\myiints at $u_4$:
$m_6 \join m_9$, $m_6 \join m_{10}$, $m_6 \join m_{11}$, and $m_6 \join m_{12}$.
A seen in Example~\ref{ex:join},
$m_{13} = m_6 \join m_{11}$ and $m_{14} = m_6 \join m_{12}$
are the only joins at $u_4$ that are defined.
In Example~\ref{ex:answersetgraph} we have seen that $\AS(u_4^<) = \{
\myint^+(m_{13}), \myint^+(m_{14}) \}$, and due to (IG-U), we cannot
have additional \myiints with the same content.
\qed
\end{example}

\subsubsection{Complete Answer Set Graphs}

We next introduce a notion of completeness for answer set graphs.
\begin{definition}
  \label{def:iocomplete}
  Let $\cA = (M,F,\myunit,\mytype,\myint)$ be an answer set graph
  for an evaluation graph $\cE = (U,E)$
  and let $u \in U$. %
  Then
  \begin{itemize}%
\itemsep=0pt
  \item
    \emph{$\cA$ is input-complete for $u$,}
    if
    $\{ \myint^+(m) \mid m \in \myimodelsA(u) \} = \AS(u^<)$,
    and
  \item
    \emph{$\cA$ is output-complete for $u$,}  if
    $\{ \myint^+(m) \mid m \in \myomodelsA(u) \} = \AS(u^{\leq})$.
  \end{itemize}
\end{definition}

If an answer set graph is complete for all units of its corresponding evaluation graph,
answer sets of the associated program can be obtained as follows.
\begin{theorem}
  \label{thm:answersetsfromvanilla}
 Let $\cE=(U,E)$, where $U = \{u_1,\ldots,u_n\}$, be an evaluation graph of a program $P$,
  and let $\cA=(M,F,\myunit,\mytype,\myint)$ be an answer set graph
 that is output-complete for every unit $u \in U$.
Then
  \begin{multline}
\label{eqn:answersetsfromvanilla}
\hspace*{-1ex}\textstyle  \AS(P) = \Big\{ \bigcup_{i=1}^n \myint(m_i) \mid
 m_i \,{\in}\,\myomodelsA(u_i) \text{, } 1 \leq i \leq n,
|\myomodels_{\cA'}(u_i)| = 1 \Big\}, \hspace*{-1ex}
\end{multline}
  where $\cA'$ is the subgraph of $\cA$
consisting of all interpretations that are reachable in $\cA$
from some interpretation $m_1,\ldots,m_n$.
\end{theorem}

\newcommand\myproofAnswerSetsFromVanilla{
  We prove this theorem using Proposition~\ref{thm:answersetsfromufinal}.
  We construct $\cE''=(U'',E'')$ with
  $U'' = U \cup \{ \ufinal \}$, $\ufinal = \emptyset$,
  and $E'' = E \cup \{ (\ufinal,u) \mid u \in U \}$.
  As $\ufinal$ contains no rules and as $\cE''$ is acyclic,
  no evaluation graph property of gets violated and $\cE''$
  is also an evaluation graph.
  As $\cA$ contains no interpretations at $\ufinal$
  and dependencies from units in $U$ are the same in $\cE$ and $\cE''$,
  $\cA$ is in fact an answer set graph for $\cE''$.
  We now modify $\cA$ to obtain $\cA''$ as follows.
  We add the set $M_\mi{new} = \{ m \mid
   m = m_1 \join \cdots \join m_n
    \text{ is defined at } \ufinal \text{ (wrt. } \cA)\}$
 as \myiints of $\ufinal$ and dependencies
from each $m\in M_\mi{new}$ to the respective \myoints $m_i$, $1 \leq i \leq n$.
  By Proposition~\ref{thm:joinanswersetgraphsoundcomplete},
   $\cA''$ is an answer set graph for $\cE''$,
  and moreover $\cA''$ gets input-complete for $\ufinal$ by
  construction. As $\cA''$ is input-complete for $U\cup\{\ufinal\}$ and
  output-complete for $U$, by Proposition~\ref{thm:answersetsfromufinal} we have
  that $\AS(P) = \myimodelsA(\ufinal) = M_\mi{new}$.
  As for every join  $m = m_1 \join \cdots \join m_n$, we have
  $\myint(m) = \myint(m_1)\cup$ $\cdots$ $\cup\myint(m_n)$,
  to complete the proof of the theorem, it remains to show that
  the join $m$ between $m_1$,\ldots,$m_n$ is defined at $\ufinal$
  iff
  the subgraph $\cA'$ of $\cA$ reachable from the \myoints $m_i$ in $F$
  fulfills $|\myomodelsA(u_i)|=1$, for each $u_i \in U$.
  As the join involves all units in $U$, and since $\cA''$ is an answer
  set graph and thus an \myigraph, it follows from the conditions for an
  \myigraph that at each $u_i \in U$ exactly one \myoint is reachable
  from $m$, and thus also from each $m_i$; thus the condition for $\cA'$
  holds. Conversely, if the subgraph $\cA'$ fulfills
  $|\myomodelsA(u_i)|=1$ for each $u_i \in U$, then clearly the FAI
  condition for the join $m$ being defined is fulfilled.
}%

\begin{example}[ctd.]
In $\cI_2$ we first choose $m_{15} \in \myomodels(u_4)$,
which is the only \myoint at $u_4$.
The subgraph reachable from $m_{15}$ must contain exactly one
\myoint at each unit;
we thus must choose every \myoints $m$
  such that $m_{15} \depends^+ m$.
  Hence we obtain
  \begin{align*}
&\big\{ \myint(m_3) \cup \myint(m_6) \cup \myint(m_{11}) \cup \myint(m_{15}) \big\} \\
&  =\,
\big\{ \{ \mygoinout(\myoutdoor) \} \cup \emptyset \cup
    \{ \mygolocation(\mypooln), \myngolocation(\mypoolg),
          \mygosomewhere \} \cups
    \{ \myneed(\myloc,\myyogamat) \} \big\} \\
&   =\,
\big\{ \{ \mygoinout(\myoutdoor),
      \mygolocation(\mypooln),
\myngolocation(\mypoolg),
\mygosomewhere,
\myneed(\myloc,\myyogamat)\} \big\}
  \end{align*}
which is indeed the set of answer sets of $\myPswim$.
\qed
\end{example}
The rather involved set construction
in~\eqref{eqn:answersetsfromvanilla} establishes a relationship between
answer sets of a program and complete answer set graphs that resembles
condition (IG-F) of \myigraphs.
To obtain a more convenient way to enumerate answer sets, we can extend an evaluation graph
always with a single void unit $\ufinal$ that depends on all other units in the
graph (i.e., $(\ufinal,u) \in E$  for each  $u \in
U\setminus\{\ufinal\}$), which we call a {\em final unit};
the answer sets of $P$  correspond then directly to \myiints at
$\ufinal$. Formally,
\begin{proposition}
  \label{thm:answersetsfromufinal}
  Let $\cA = (M,F,\myunit,\mytype,\myint)$ be an answer set graph for an
  evaluation graph $\cE=(U,E)$ of a program $P$, where $\cE$ contains a
  final unit $\ufinal$, and assume that $\cA$ is input-complete for $U$
  and output-complete for $U \setminus \{ \ufinal \}$.  Then
  \begin{align}
    \AS(P) = \{ \myint(m) \mid m \in \myimodelsA(\ufinal) \}.
  \end{align}
\end{proposition}
\newcommand\myproofAnswerSetsFromUFinal{
As $\ufinal$ depends on all units in $U \setminus \{ \ufinal\}$,
due to O-connec\-ted\-ness every \myiint $m \in \myimodelsA(\ufinal)$
depends on one \myoint at every unit in $U \setminus \{ \ufinal \}$.
Let $U \setminus \{ \ufinal \} = \{ u_1, \ldots, u_k\}$
and let $M_M = \{m_1, \ldots, m_k\}$ be the set of \myoints
such that $(m,m_i) \in F$ and $m_i \in \myomodelsA(u_i)$, $1\leq i \leq k$.
Then, due to \myCAUtext intersection,
$M_m$
contains each  \myoint that is reachable from $m$ in $\cA$,
and $M_m$ contains only interpretations with this property.
Hence $\myint(m)^+ = \myint(m_1) \cup \cdots \cup \myint(m_k)$,
and due to condition~\eqref{itm:imodelunion}
in Definition~\ref{def:answersetgraph},
  we have $\myint(m) = \myint(m)^+$.
  By the dependencies of $\ufinal$,
  we have $\ufinal^< = P$,
  and as $\ufinal$ is input-complete,
  we have that $\AS(P) = \AS(\ufinal^<) = \{ \myint(m)^+ \mid m \in \myimodelsA(\ufinal)\}$.
  As  $\myint(m) = \myint(m)^+$
  for every \myiint $m$ at $\ufinal$,
  we obtain the result.
}%
\myinlineproof{%
  \begin{proof}%
    \myproofAnswerSetsFromUFinal
  \end{proof}
}
\myinlineproof{%
  \begin{proof}[Proof of Theorem~\ref{thm:answersetsfromvanilla}]
    \myproofAnswerSetsFromVanilla
  \end{proof}
}
Expanding \myiints at $\ufinal$
is not necessary, as $\ufinal$ depends on all other units;
thus for every $m \in \myimodelsA(\ufinal)$ it holds that
$\myint^+(m) = \myint(m)$.

We will use the technique with $\ufinal$ for our model enumeration
algorithm; as the join condition must be checked anyways, this technique
is an efficient and simple method for obtaining all answer sets of a
program using an answer set graph
\reva{without requesting an implementation of the conditions in Theorem~\ref{thm:answersetsfromvanilla}}.

\subsection{Answer Set Building}
\label{sec:evaluation}

Thanks to the results above, we can obtain the answer sets of a
\hex-program from any answer set graph for it. To build an answer set
graph, we proceed as follows. We start with an empty graph, obtain
\myoints by evaluating a unit on an \myiint, and then gradually generate
\myiints by joining \myoints of predecessor units in an evaluation graph
at hand.

Towards an algorithm for evaluating a \hex-program based on an
evaluation graph, we use a generic grounding algorithm \GroundLiberallyDomainExpansionSafeProgram{} for lde-safe
programs, and a solving
algorithm $\EvaluateGroundHEX$ which returns for a ground
\hex-program $P$ its answer sets $\mathcal{AS}(P)$.  We assume that they
satisfy the following properties.

\begin{property}
\label{prop:groundingAlgorithm}
Given an lde-safe program $P$,
$\GroundLiberallyDomainExpansionSafeProgram(P)$ returns a finite ground
program such that
$\mathcal{AS}(P) = \mathcal{AS}(\GroundLiberallyDomainExpansionSafeProgram(P))$.
\end{property}

\begin{property}
\label{prop:groundEvaluationAlgorithm}
Given a finite ground \hex-program $P$, $\EvaluateGroundHEX(P) = \mathcal{AS}(P)$.
\end{property}

Concrete such algorithms
are given in \cite{eite-etal-14a} and \cite{efkrs2014-jair}, respectively.
\reva{Since the details of these algorithms are not relevant for the further understanding of
this paper, we give here only an informal description and refer the
interested reader to the respective papers.
The idea of the grounding algorithm
is to iteratively extend the grounding by expanding the set of constants
until it is large enough to ensure that it has the same answer sets as the original program.
To this end, the algorithm starts with the constants in the input program only, and
in each iteration of the algorithm it evaluates external atoms
a (finite) number
of relevant inputs in order to determine additional relevant constants.
Under the syntactic restrictions recapitulated in the preliminaries,
this iteration will reach a fixpoint after finitely many steps.
The solving algorithm is based on conflict-driven clause learning (CDCL)
and lifts the work of~\citeN{2012_conflict_driven_answer_set_solving_from_theory_to_practice}
from ordinary %
to \hex programs. The main idea is to learn not only conflict clauses, but also (parts of) the behavior of external sources while the search space
is traversed. The behavior is described in terms of input-output relations, i.e., certain input atoms and constants lead to a certain output
of the external atom. This information is added to the internal representation of the program such that
guesses for external atoms that violate the known behavior are eliminated in advance.}

\myfigureAlgoEvaluateLDESafe{t}

By composing the two algorithms, we obtain Algorithm~\ref{alg:EvaluateLDESafe} for evaluating a single
unit. Formally, it has the following property.

\begin{proposition}
\label{thm:EvaluateLDESafe}
Given an
lde-safe \hex-program $P$ and an input interpretation $I$,
Algorithm~\ref{alg:EvaluateLDESafe}
returns the set
$\left\{ I' \setminus I \mid I' \in
\mathcal{AS}(P \cup \mathit{facts}(I)) \right\}$,
i.e., the answer sets of $P$ augmented with facts for the input $I$,
projected to the non-input.
\end{proposition}

\newcommand\myproofEvaluateLDESafe{
The proposition follows from Property~\ref{prop:groundingAlgorithm},
which asserts that the grounding $P'$ has the same answer sets as $P$,
and from the soundness and completeness of the evaluation algorithm for ground \hex-programs as asserted by Property~\ref{prop:groundEvaluationAlgorithm}.
}
\myinlineproof{
\begin{proof}
\myproofEvaluateLDESafe
\end{proof}
}

We are now ready to formulate an algorithm
for evaluating \hex programs that have been decomposed into an evaluation graph.

\myfigureAlgoBuildAnswerSets{tb}

To this end, we build first an evaluation graph $\cE$ and
then compute gradually an answer set graph $\cA =
(M,F,\myunit,\mytype,\myint)$ based on $\cE$, proceeding along already
evaluated units towards the unit $\ufinal$. Algorithm~\ref{alg:buildAnswerSets} shows the model
building algorithm in pseudo-code, in which the positive integers $\bbN
= \{ 1,2,\ldots \}$ are used as identifiers $\cI_{\mi{id}}$ and
$\max(M)$ is maximum in any set $M\subseteq \bbN$ where, by convention,
$\max(\emptyset)=0$.
Intuitively, the algorithm works as follows.
The set $U$ contains units for which $\cA$ is not yet output-complete
(see Definition~\ref{def:iocomplete});
we start with an empty answer set graph $\cA$,
thus initially $U = V$.
In each iteration of the while loop~\ref{step:whileloop},
a unit $u$ that is not output-complete and depends only on output-complete units is selected.
The first for loop~\ref{step:firstforloop} makes $u$ input-complete;
if $u$ is the final unit, the answer sets are returned in~\ref{step:return},
otherwise the second for loop~\ref{step:secondforloop} makes $u$
output-complete, and then $u$ is removed from $U$.
Each iteration makes one unit input- and output-complete;
hence when the algorithm reaches $\ufinal$
and makes it input-complete, all answer sets can directly be returned
in~\ref{step:return}. Formally, we have
\begin{theorem}
  \label{thm:soundcomplete}
  Given an evaluation graph $\cE=(V,E)$ of a \hex program $P$,
  $\myBuildAnswerSets(\cE)$ returns $\AS(P)$.
\end{theorem}
\newcommand\myproofSoundComplete{
We show by induction on its construction that
$\cI=(M,F,\myunit,$ $\mytype,\myint)$ is an answer set graph for $\cE$,
and that at the beginning of the while-loop
$\cI$ is input- and output-complete for $V \setminus U$.

(Base)
Initially, $\cI$ is initially and $V = U$,
hence the base case trivially holds.

(Step)
Suppose that $\cI$ is an answer set graph for $\cE$
at the beginning of the while-loop, and that it is
input- and output-complete for $V \setminus U$.
As the chosen $u$ only depends on units in $V \setminus U$,
it depends only on output-complete units.
For a leaf unit $u$,
\ref{step:emptyinputmodels} creates an empty \myiint
and therefore makes $u$ input-complete.
For a non-leaf unit $u$,
  the first for-loop \ref{step:firstforloop}
  builds all possible joins of interpretations at predecessors of $u$
  and adds them as \myiints to $\cI$.
  As all predecessors of $u$ are output-complete by the hypothesis,
  this makes $u$ input-complete.
  Now suppose that Condition~\ref{step:return} is false, i.e.,
  $u\neq \ufinal$. Then the second for-loop~\ref{step:secondforloop}
  evaluates $u$ wrt.\ every \myiint at $u$
  and adds the result to $u$ as an \myoint.
  Due to Proposition~\ref{thm:EvaluateLDESafe},
  $\EvaluateLDESafe(u,\myint(m'))$ returns all interpretations $o$ such
  that $o \in \{ X \setminus \myint(m') \mid X \in \AS(u \cup
  \facts(\myint(m')) \}$.
  As $u$ depends on all units on which its rules depend,
  and as \myiints contain all atoms from \myoints of predecessor units
  (due to condition~\eqref{itm:imodelunion} of Definition~\ref{def:answersetgraph}),
  we have $\EvaluateLDESafe(u,\myint(m')) = \EvaluateLDESafe(u,\myint(m')^+)$.
  By Theorem~\ref{thm:evalgraphprecbottom},
  $u^<$ is a generalized bottom of $u^{\leq}$, and by the induction
  hypothesis $\myint(m')^+ \in \AS(u^<)$;
  hence by Theorem~\ref{thm:hexgensplitting},
  we have that
  $\myint(m')^+ \cup o \in \AS(u^{\leq})$.
  Consequently, adding a new \myoint $m$ with interpretation $\myint(m) = o$
  and dependency to $m'$ to the graph $\cI$ results in $\myint(m)^+ \in
  \AS(u^{\leq})$, and adding all of them makes $\cI$ output-complete for $u$.
  Finally, in \ref{step:removeu} $u$ is removed from $U$; hence at the
  end of the while-loop $\cI$ is an answer set graph and again input-
  and output-complete for $V \setminus U$.

  It remains to consider the case where  Condition~\ref{step:return} is true.
  Then $\ufinal$ was made input-complete,
  which means that all predecessors of $\ufinal$ are output-complete.
  As $\ufinal$ depends on all other units, we have $U = \{ \ufinal \}$
  and the algorithm returns $\myimodelsA(u)$; by
  Proposition~\ref{thm:answersetsfromufinal}, it thus returns $\AS(P)$,
  which will happen in the $|V|$-th iteration of the while loop.
}
\myinlineproof{
  \begin{proof}
    \myproofSoundComplete
  \end{proof}
}

A run of the algorithm on our running example using the evaluation
graph $\cE_2$ extended with a final unit is given in
\ref{sec:appendix-alg-run}.

\subsubsection{Model Streaming}

Algorithm $\myBuildAnswerSets$ as described above keeps all answer sets
in memory, and it evaluates each unit only once wrt.\ every possible \myiint.
This may lead to a resource bound excess, as in general an
exponential number of answer sets respectively interpretations at
evaluation units are possible. However, keeping the whole answer set
graph in memory is not necessary for computing all answer sets.

We have realized a variant of Algorithm $\myBuildAnswerSets$ that uses
the same principle of constructing an answer set graph, interpretations
are created at a unit {\em on demand}\/ when they are requested by units
that depend on it; furthermore, the algorithm keeps basically only one
interpretation at each evaluation unit in memory at a time, which means
that interpretations are provided in a {\em streaming fashion}\/ one by one,
and likewise the answer sets of the program at the
unit $\ufinal$, where the model building starts.
Such answer set streaming is particularly attractive for applications,
as one can terminate the computation after obtaining sufficiently many
answer sets. On the other hand, it comes at the cost of potential
re-evaluation of units wrt.\ the same \myiint, as we need to trade space
for time. However, in practice this algorithm works well and is the one
used in the \dlvhex prototype.
We describe this algorithm in
\ref{sec:appendixStreaming}.

\section{Implementation}
\label{sec:heximpl}

In this section we give some details on the implementation of the techniques.
Our prototype system is called \dlvhex{}; it is written in C++ and
online available as open-source software.%
\footnote{\url{http://www.kr.tuwien.ac.at/research/systems/dlvhex}}
The current version 2.4.0 was released in September 2014.

We first describe the general architecture, the major components, and
their interplay \reva{(Section~\ref{sec:architecture})}.
Then we give an overview about the existing
heuristics \reva{for  building evaluation graphs
(Section~\ref{sec:heuristics}).
Experimental results are presented and discussed in
Section~\ref{sec:experiments}.
}
For details on the usage of the system, we refer to the
website; an exhaustive description of the supported command-line
parameters is output when the system is called without parameters.

\subsection{System Architecture}
\label{sec:architecture}

The \dlvhex{} system architecture is shown in Figure~\ref{fig:dlvhexArchitecture}.
The arcs model both control and data flow within the system.
The evaluation of a \hex{}-program works as follows.

First, the input program is
passed to the \emph{evaluation framework}~\textcircled{\tiny 1},
\reva{which} creates an \emph{evaluation graph} depending on the chosen evaluation heuristics.
This results in a number of interconnected \emph{evaluation units}. While the interplay of the units
is managed by the evaluation framework, the individual units are handled by \emph{model generators}
of different kinds.

Each instance of a model generator
\reva{%
realizes \EvaluateLDESafe\ (Algorithm~\ref{alg:EvaluateLDESafe}) for}
a single evaluation unit,
receives \emph{input interpretations} from the framework (which are either output by predecessor units
or come from the input facts for leaf units), and sends output interpretations back to the framework~\textcircled{\tiny 2},
which manages
the integration of the latter to final answer sets
\reva{and realizes \myBuildAnswerSets (Algorithm~\ref{alg:buildAnswerSets})}.

Internally, the model generators make use of a \emph{grounder} and a
\emph{solver} for ordinary ASP programs. The architecture of our system
is flexible and supports multiple concrete backends that can be plugged
in. Currently it supports \dlv{}, \gringo{} 4.4.0 and \clasp{} 3.1.0,
as well as an internal grounder and a solver that were built from scratch
(mainly for testing purposes); they use basically the same core
algorithms as \gringo{} and \clasp{}, but without
optimizations.  The reasoner backends \gringo{} and \clasp{} are
statically linked to our system; thus no interprocess communication is
necessary.  The model generator within the \dlvhex{} core sends a
non-ground \reva{evaluation unit} to the \hex{}-grounder \reva{which returns} a ground
\reva{evaluation unit}~\textcircled{\tiny 3}. The \hex{}-grounder in turn uses
\reva{one of the above mentioned }ordinary ASP grounders as \reva{backend}~\textcircled{\tiny 4} and accesses
external sources to handle
\reva{newly introduced constants that are not part of the input program (called \emph{value invention})}~\textcircled{\tiny 5}.  The
ground \reva{evaluation unit} is then sent to the \reva{ASP solver and answer sets of the ground unit} are returned~\textcircled{\tiny
 6}.

\reva{Intuitively, model generators evaluate evaluation units
by replacing external atoms by ordinary \quo{replacement} atoms,
guessing their truth value, and making sure that the guesses are correct with respect to the external oracle functions.
To achieve that, the solver backend needs to make}
callbacks to the \emph{Post Propagator} in the \dlvhex{} core \reva{during model building}.
\reva{
The Post Propagator checks guesses for external atoms against the actual semantics
and checks the minimality of the answer set.}
\reva{
It processes a complete or partial model candidate,
and returns learned nogoods} to the external solver~\textcircled{\tiny 7}
as formalized in \cite{2012_conflict_driven_asp_solving_with_external_sources}.
\reva{The \dlv{} backend calls the Post Propagator only for complete model candidates,
the internal solver and the \clasp{} backend also call it for partial model candidates of evaluation units.}
For the \clasp{} backend, we exploit its SMT interface, which was previously used
for the special case of constraint answer set solving~\cite{geossc09a}.
\reva{
Verifying guesses of replacement atoms requires calling} \emph{plugins}
that implement the external sources
\reva{(i.e., the oracle functions $\extFun{g}$ from Definition~\ref{def:extensionalevaluationfunction})}%
~\textcircled{\tiny 8}.
Moreover, the Post Propagator also
\reva{%
ensures answer set minimality by eliminating}
unfounded sets that \reva{are caused by external sources and therefore can not be}
detected by the ordinary ASP solver \reva{backend
(as shown by \citeN{efkrs2014-jair}).
}
Finally, \reva{as soon as} the evaluation framework
\reva{obtains an \myiint of the final evaluation unit $u_\mi{final}$,
this \myiint (which is an answer set according to Proposition~\ref{thm:answersetsfromufinal})
}
is returned to the user~\textcircled{\tiny 9}.

\begin{figure}[t]
\renewcommand\arraystretch{2.4}
\centering
\resizebox{\textwidth}{!}{
\myfigureSystemArchitecture
}
\caption{Architecture of \dlvhex{}}
\label{fig:dlvhexArchitecture}
\end{figure}

\subsection{Heuristics}
\label{sec:heuristics}

As for creating evaluation graphs,
several heuristics have been implemented.
A heuristics starts with the rule dependency graph as by Definition~\ref{def:ruledepgraph}
and then acyclically combines nodes into units.

Some heuristics are described in the following.

\newcommand{\heuritem}[1]{\smallskip\noindent{\textbf #1\ }}
\heuritem{\heurtriv}
is a \quo{trivial} heuristics that makes units as small as
possible. This is useful for debugging,
however it generates the largest possible number of evaluation units
and therefore incurs a large overhead.
\reva{As a consequence \heurtriv performs clearly worse than other heuristics and we do not report its performance in experimental results.}

\heuritem{\heurold}
is the evaluation heuristics of the \dlvhex prototype version~1.
\heurold makes units as large as possible and has several drawbacks as discussed above.

\heuritem{\heurnew}
is a simple evaluation heuristics which has the goal of finding
a compromise between the \heurtriv and \heurold.
It places rules into units as follows:
  \begin{enumerate}[(i)]
  \item
    it puts rules~$r_1,r_2$ into the same unit
    whenever~$r_1 \dependsmn s$ and $r_2 \dependsmn s$ for some rule $s$
and there
    is no rule $t$ such that exactly one of $r_1,r_2$ depends on $t$;
  \item
    it puts rules~$r_1,r_2$ into the same unit
    whenever $s \dependsmn r_1$ and $s \dependsmn r_2$ for some rule $s$
    and there is no rule $t$ such that $t$ depends on exactly one of $r_1,r_2$; but
  \item
    it never puts rules $r,s$ into the same unit
    if $r$ contains external atoms and $r \dependsmn s$.
  \end{enumerate}
  Intuitively,  %
  \heurnew builds an evaluation graph
  that puts all rules with external atoms and their successors into one unit,
  while separating rules creating input for distinct external atoms.
  This avoids redundant computation and joining unrelated interpretations.

\heuritem{\heurmrg}
is a heuristics %
for%
\reva{ finding a compromise between}%
~(1) minimizing the number of units, and~(2) splitting the program whenever a
de-relevant nonmonotonic external atom would receive input from
the same unit.
\reva{We mention this heuristics only as an example,
but disregard it in the experiments since it was developed in connection
with novel \quo{liberal} safety criteria~\cite{efkr2013-aaai}
that are beyond the scope of this paper.}
\heurmrg greedily gives preference to~(1) and is motivated by
the following considerations.
The grounding algorithm by~\citeN{eite-etal-14a} evaluates the external
sources under all interpretations such that the set of observed
constants is maximized.  While monotonic and antimonotonic input atoms
are not problematic (the algorithm can simply set all to true
resp.~false), nonmonotonic parameters require an exponential number of
evaluations in general.  Thus, although program decomposition is not strictly
necessary \reva{for evaluating liberally safe \hex-programs}, it is still useful in such cases as it restricts grounding to
those interpretations that are actually relevant in some answer set.
However, on the other hand it can be disadvantageous for propositional
solving algorithms such as
those
in \cite{2012_conflict_driven_asp_solving_with_external_sources}.

Program decomposition can be seen as a hybrid between traditional and
lazy grounding (cf.~e.g.~\citeN{pdpr2009-fi}), as program parts are
instantiated that are larger than single rules but smaller than the
whole program.

\subsection{Experimental Results}
\label{sec:experiments}

In this section, we evaluate the model-building framework empirically.
To this end, we compare the following configurations. In the \reva{\emph{H1}}
column, we use the previous \reva{state-of-the-art}
evaluation method~\cite{rs2006}
before the framework \reva{in Section~\ref{sec:decomposition}} was developed.
\reva{This previous} method also makes use of program
decomposition.  However, in contrast to our new framework, the
decomposition is based on atom dependencies rather than rule
dependencies, and the decomposition strategy is hard-coded and not
customizable.  This evaluation method corresponds to
heuristics \heurold{} in our new framework.

In the \emph{w/o framework} column, we present the results without application of the framework
using the \reva{\hex-program evaluation} algorithm by~\citeN{eite-etal-14a}
\reva{which allows to first instantiate and then solve the instantiated \hex-program}.
Note that before this algorithm was developed,
\reva{such a %
`two-phase'} evaluation was not possible
since program decomposition was necessary for grounding purposes.
With \reva{the algorithm in \cite{eite-etal-14a}}%
, decomposition is not necessary anymore,
but \reva{can still be} useful
as the results in the \reva{\emph{H2} column shows, which correspond to the results when applying the
heuristics \emph{H2} described above}.

The configuration of the grounding algorithm and the solving algorithm (\reva{e.g.,}~conflict-driven learning strategies)
also influence the results.
Moreover, in addition to the default heuristics of framework, other heuristics have been developed as well
and the best selection of the heuristics often depends on the configuration of the grounding and the solving algorithm.
Since they were used as black boxes in Algorithm~\ref{alg:EvaluateLDESafe},
an exhaustive experimental analysis of the system is beyond the scope of this paper and would require an in-depth description
of these algorithms.
Thus, we confine the discussion to the default settings, which suffices to show that the new framework
can speed up the evaluation significantly.
\reva{The only configuration difference
between the  result columns \emph{\heurold} and \emph{\heurnew}
is the evaluation heuristics,
all other parameters are equal.
Evaluating the \emph{w/o framework} column
requires the grounding algorithm from \cite{eite-etal-14a}
instead of evaluation via decomposition,
therefore \emph{w/o framework} does not use any heuristics.
The solver backend (\clasp) configuration is the same in
\emph{\heurold}, \emph{\heurnew}, and \emph{w/o framework}.
We use the streaming algorithm
(see \ref{sec:appendixStreaming}) in all experiments.}
For an in depth discussion, we refer to~\citeNBYYB{efkrs2014-jair}{eite-etal-14a}
and~\citeN{r2014-phd},
where the efficiency was evaluated using a variety of applications
including planning tasks (\reva{e.g.,} robots searching an unknown area for
an object, tour planning), computing extensions of abstract argumentation frameworks, inconsistency analysis in multi-context systems, and reasoning over description logic knowledge bases.

\reva{We %
discuss here two benchmark problems,
which we evaluated on a
Linux server with two 12-core AMD 6176 SE CPUs with 128GB RAM
running \dlvhex{} version 2.4.0. and an \emph{HTCondor} load distribution system%
\footnote{\url{http://research.cs.wisc.edu/htcondor}}
that ensures robust runtimes.
The HTCondor system ensures that multiple runs of the same instance have negligible deviations in the order of
fractions of a second, thus we can restrict the experiments to one run.}
The grounder and solver backends for all benchmarks are \gringo{} 4.4.0 and \clasp{} 3.1.1.
For each instance, we limited the CPU usage to two cores and 8GB RAM.
The timeout for each instance was 600 seconds.
Each line shows the average runtimes over all instances of a
certain size, where each timeout counts as 600 seconds.
While instances usually become harder with larger size,
there might be some exceptions due to the randomly generated instances;
however, the overall trend shows that runtimes increase with the instance size.
Numbers in parentheses are the numbers of instances of respective size
in the leftmost column and the numbers of timeout instances elsewhere.
The generators, instances and external sources are available at {\small\url{http://www.kr.tuwien.ac.at/research/projects/hexhex/hexframework}}.

\newcommand{\mcsbench}[0]{{\sc MCS}\xspace}
\subsubsection{Multi-Context Systems (\mcsbench)}

The \mcsbench benchmarks originate in the application scenario of
enumerating output-projected equilibria (i.e., global models) of a given
multi-context system (MCS) (cf.\ Section~\ref{secKnowledgeOutsourcing}).
Each instance comprises 7--9 contexts (propositional knowledge bases) whose local
semantics is modeled by external atoms; roughly speaking,
they single out assignments to the atoms of a context occurring in bridge rules such that
local models exist. For each context, 5--10 such atoms are
guessed and bridge rules, which are
modeled by ordinary rules, are randomly constructed on top.
The \mcsbench instances were generated using the
DMCS~\cite{Bairakdar2010dmcs}
instance generator,
with 10 randomized instances for different link structure between
contexts (diamond (d), house (h), ring (r), zig-zag (z))
and system size; they have between 4 and about 20,000 answer sets, with
an average of 400. We refer to \cite{Bairakdar2010dmcs} and \cite{ps2012}
for more details on the benchmarks and the \hex-programs.
\begin{table}[t]
        \scriptsize
        \centering
        \begin{tabular}[t]{rrrrrrr}
                \hline
                \hline
				Topology and & \multicolumn{3}{c}{First Answer Set} & \multicolumn{3}{c}{All Answer Sets} \\
        \cline{2-4}
        \cline{5-7}
				Instance Size \rule{0pt}{1.1em} %
          & \multicolumn{1}{c}{H1} & \multicolumn{1}{c}{w/o framework} & \multicolumn{1}{c}{H2}
          & \multicolumn{1}{c}{H1} & \multicolumn{1}{c}{w/o framework} & \multicolumn{1}{c}{H2} \\
		\hline
		 d-7-7-3-3 (10) & 1.23 (0) & 0.29 (0) & 0.38 (0) & 4.93 ~~(0) & 0.76 (0) & 0.79 (0) \\
		 d-7-7-4-4 (10) & 18.43 (0) & 1.09 (0) & 0.76 (0) & 50.78 ~~(0) & 3.39 (0) & 1.80 (0) \\
		 d-7-7-5-5 (10) & 94.18 (1) & 3.60 (0) & 1.52 (0) & 289.35 ~~(4) & 20.21 (0) & 4.97 (0) \\
		 h-9-9-3-3 (10) & 83.17 (1) & 3.77 (0) & 0.70 (0) & 300.96 ~~(4) & 28.67 (0) & 2.11 (0) \\
		 h-9-9-4-4 (10) & 389.74 (6) & 30.56 (0) & 2.14 (0) & 555.94 ~~(9) & 335.11 (5) & 12.56 (0) \\
		 r-7-7-4-4 (10) & 39.27 (0) & 2.82 (0) & 0.33 (0) & 366.17 ~~(5) & 57.26 (0) & 2.06 (0) \\
		 r-7-7-5-5 (10) & 389.88 (6) & 105.80 (1) & 0.93 (0) & 600.00 (10) & 377.37 (5) & 4.39 (0) \\
		 r-7-8-5-5 (10) & 226.04 (3) & 25.11 (0) & 0.57 (0) & 541.80 ~~(9) & 317.64 (5) & 3.99 (0) \\
		 r-7-9-5-5 (10) & 355.37 (5) & 145.99 (2) & 0.87 (0) & 600.00 (10) & 458.14 (7) & 5.42 (0) \\
		 r-8-7-5-5 (10) & 502.64 (8) & 329.47 (5) & 1.21 (0) & 555.26 ~~(9) & 443.15 (7) & 5.84 (0) \\
		 r-8-8-5-5 (10) & 390.81 (6) & 201.08 (3) & 1.00 (0) & 600.00 (10) & 495.41 (8) & 5.38 (0) \\
		 z-7-7-3-3 (10) & 2.34 (0) & 0.32 (0) & 0.44 (0) & 9.17 ~~(0) & 1.13 (0) & 1.00 (0) \\
		 z-7-7-4-4 (10) & 33.32 (0) & 1.58 (0) & 1.07 (0) & 182.44 ~~(2) & 9.00 (0) & 2.67 (0) \\
		 z-7-7-5-5 (10) & 164.33 (2) & 12.69 (0) & 3.52 (0) & 502.49 ~~(8) & 89.01 (1) & 6.90 (0) \\
                \hline
                \hline
        \end{tabular}
        \caption{\mcsbench experiments: variable topology (d, h, r, z) and instance size.}
        \label{tab:mcs}
\end{table}

Table~\ref{tab:mcs} shows the experimental results:
computation with the old method \reva{\heurold} often exceeds the time limit,
while the new method \heurnew manages to enumerate all solutions of all instances.
Monolithic evaluation without decomposition shows a performance between the old and new method.
These results show that our new evaluation method is essential for using
HEX to computationally realize the MCS application.

\newcommand{\revsel}[0]{{\sc RS}\xspace}
\newcommand{\revselTwo}[0]{{\sc RSTrack}\xspace}
\newcommand{\revselFour}[0]{{\sc RSPaper}\xspace}

\subsubsection{Reviewer Selection (\revsel)}

Our second benchmark is Reviewer Selection (\revsel):
we represent $c$ conference tracks,
$r$ reviewers and $p$ papers.
Papers and reviewers are assigned to conference tracks,
and there are conflicts between reviewers and papers,
some of which are
given by external atoms.
We consider two scenarios: \revselTwo and \revselFour.
They are designed to measure the effect of external atoms
on the elimination of a large number of answer set candidates;
in contrast to the \mcsbench experiments we can control this aspect in the \revsel experiments.

In \revselTwo we vary the number $c$ of conference tracks,
where each track has 20 papers and 20 reviewers.
Each paper must get two reviews, and no reviewer must get more than two papers.
Conflicts are dense such that only one valid assignment exists per track,
hence each instance has exactly one answer set,
and in each track two conflicts are external.
For each
number $c$ there is only one instance
because \revselTwo instances are not randomized.
The results of \revselTwo are shown in Table~\ref{tab:revsel2}:
runtimes of the old evaluation heuristics (\heurold)
grow fastest with
size,
without using decomposition grows slightly slower
but also reaches timeout at
size~9.
Only the new decomposition (\heurnew heuristics)
can deal with
size~20
without timeout.
Finding the first answer set and enumerating all answer sets show very similar times,
as \revselTwo instances have a single answer set and finding it seems
hard.
\begin{table}[t]
        \scriptsize
        \centering
        \begin{tabular}[t]{rrrrrrr}
                \hline
                \hline
				Instance Size & \multicolumn{3}{c}{First Answer Set} & \multicolumn{3}{c}{All Answer Sets} \\
        \cline{2-4}
        \cline{5-7}
        \rule{0pt}{1.1em} %

          & \multicolumn{1}{c}{H1} & \multicolumn{1}{c}{w/o framework} & \multicolumn{1}{c}{H2}
          & \multicolumn{1}{c}{H1} & \multicolumn{1}{c}{w/o framework} & \multicolumn{1}{c}{H2} \\
		\hline
		 1 (1) & 2.84 (0) & 3.14 (0) & 2.78 (0) & 2.73 (0) & 3.14 (0) & 2.79 (0) \\
		 2 (1) & 6.13 (0) & 7.18 (0) & 4.90 (0) & 6.05 (0) & 7.17 (0) & 4.88 (0) \\
		 3 (1) & 10.18 (0) & 12.30 (0) & 8.32 (0) & 10.25 (0) & 12.35 (0) & 8.37 (0) \\
		 4 (1) & 15.92 (0) & 18.66 (0) & 12.12 (0) & 15.86 (0) & 18.85 (0) & 12.16 (0) \\
		 5 (1) & 26.06 (0) & 28.47 (0) & 17.17 (0) & 26.23 (0) & 28.35 (0) & 17.06 (0) \\
		 6 (1) & 47.06 (0) & 45.71 (0) & 23.39 (0) & 46.84 (0) & 45.62 (0) & 23.26 (0) \\
		 7 (1) & 92.76 (0) & 79.41 (0) & 31.19 (0) & 96.56 (0) & 79.82 (0) & 31.04 (0) \\
		 8 (1) & 198.59 (0) & 155.10 (0) & 37.85 (0) & 199.74 (0) & 155.26 (0) & 38.06 (0) \\
		 9 (1) & 600.00 (1) & 600.00 (1) & 46.61 (0) & 600.00 (1) & 600.00 (1) & 46.75 (0) \\
		 10 (1) & 600.00 (1) & 600.00 (1) & 57.48 (0) & 600.00 (1) & 600.00 (1) & 57.40 (0) \\
		 11 (1) & 600.00 (1) & 600.00 (1) & 68.98 (0) & 600.00 (1) & 600.00 (1) & 69.45 (0) \\
		 12 (1) & 600.00 (1) & 600.00 (1) & 84.41 (0) & 600.00 (1) & 600.00 (1) & 84.11 (0) \\
		 13 (1) & 600.00 (1) & 600.00 (1) & 99.55 (0) & 600.00 (1) & 600.00 (1) & 99.52 (0) \\
		 14 (1) & 600.00 (1) & 600.00 (1) & 117.39 (0) & 600.00 (1) & 600.00 (1) & 117.15 (0) \\
		 15 (1) & 600.00 (1) & 600.00 (1) & 138.45 (0) & 600.00 (1) & 600.00 (1) & 137.51 (0) \\
		 16 (1) & 600.00 (1) & 600.00 (1) & 163.12 (0) & 600.00 (1) & 600.00 (1) & 158.43 (0) \\
		 17 (1) & 600.00 (1) & 600.00 (1) & 184.99 (0) & 600.00 (1) & 600.00 (1) & 181.94 (0) \\
		 18 (1) & 600.00 (1) & 600.00 (1) & 208.83 (0) & 600.00 (1) & 600.00 (1) & 210.82 (0) \\
		 19 (1) & 600.00 (1) & 600.00 (1) & 236.98 (0) & 600.00 (1) & 600.00 (1) & 237.45 (0) \\
		 20 (1) & 600.00 (1) & 600.00 (1) & 267.54 (0) & 600.00 (1) & 600.00 (1) & 268.60 (0) \\
		 21 (1) & 600.00 (1) & 600.00 (1) & 600.00 (1) & 600.00 (1) & 600.00 (1) & 600.00 (1) \\
                \hline
        \end{tabular}
        \caption{\revselTwo experiments: variable number of conference tracks, single answer set.}
        \label{tab:revsel2}
\end{table}

In \revselFour
we fix the number of
tracks to $c \eqs 5$;
we vary the number $p$ of papers in each track
and set the number of reviewers to $r \eqs p$.
Each paper must get three reviews
and each reviewer must not get more than three papers assigned.
Conflicts are randomized and less dense than in \revselTwo:
the number of answer sets is greater than one and does not grow with the instance size.
Over all tracks and papers, $2 p$ randomly chosen conflicts are external,
and we generate 10 random instances per size and report results
averaged per instance size
in Table~\ref{tab:revsel4e6}.
As clearly seen, our new method is always faster than the other
methods, and evaluation without a decomposition framework performs
slightly better than the old method. %
Different from \revselTwo,
we can see a clear difference between
finding the first answer set and
enumerating all answer sets
as \revselFour instances have more than one answer set.

To confirm that the new method is geared towards handling many external atoms,
we conducted also experiments with instances that had few
external atoms for eliminating answer set candidates but many local
constraints.
For such highly constrained instances, the new
decomposition framework is not beneficial as it incurs an overhead
compared to the monolithic evaluation that increases runtimes.
\begin{table}[t]
        \scriptsize
        \centering
        \begin{tabular}[t]{rrrrrrr}
                \hline
                \hline
				Instance Size & \multicolumn{3}{c}{First Answer Set} & \multicolumn{3}{c}{All Answer Sets} \\
        \cline{2-4}
        \cline{5-7}
        \rule{0pt}{1.1em} %

				  & \multicolumn{1}{c}{H1} & \multicolumn{1}{c}{w/o framework} & \multicolumn{1}{c}{H2}
				  & \multicolumn{1}{c}{H1} & \multicolumn{1}{c}{w/o framework} & \multicolumn{1}{c}{H2} \\
				\hline
                5 (10) & 1.06 ~~(0) & 0.28 ~~(0) & 0.21 (0) & 2.25 ~~(0) & 0.43 ~~(0) & 0.23 ~~(0) \\
                8 (10) & 8.76 ~~(0) & 2.73 ~~(0) & 0.38 (0) & 14.73 ~~(0) & 4.54 ~~(0) & 0.44 ~~(0) \\
                11 (10) & 108.70 ~~(1) & 83.26 ~~(1) & 0.98 (0) & 171.01 ~~(2) & 104.84 ~~(1) & 1.28 ~~(0) \\
                14 (10) & 180.99 ~~(2) & 125.83 ~~(1) & 2.08 (0) & 299.22 ~~(4) & 245.62 ~~(3) & 2.67 ~~(0) \\
                17 (10) & 418.92 ~~(6) & 364.95 ~~(5) & 5.15 (0) & 549.01 ~~(9) & 513.21 ~~(8) & 8.14 ~~(0) \\
                20 (10) & 485.35 ~~(8) & 453.39 ~~(7) & 7.32 (0) & 507.66 ~~(8) & 501.74 ~~(8) & 14.45 ~~(0) \\
                23 (10) & 542.03 ~~(9) & 508.75 ~~(8) & 13.91 (0) & 600.00 (10) & 600.00 (10) & 23.16 ~~(0) \\
                26 (10) & 600.00 (10) & 600.00 (10) & 33.20 (0) & 600.00 (10) & 600.00 (10) & 154.51 ~~(2) \\
                29 (10) & 600.00 (10) & 600.00 (10) & 60.78 (0) & 600.00 (10) & 600.00 (10) & 108.03 ~~(0) \\
                32 (10) & 600.00 (10) & 600.00 (10) & 129.95 (0) & 600.00 (10) & 600.00 (10) & 315.56 ~~(4) \\
                35 (10) & 600.00 (10) & 600.00 (10) & 136.84 (0) & 600.00 (10) & 600.00 (10) & 302.90 ~~(3) \\
                38 (10) & 600.00 (10) & 600.00 (10) & 308.92 (3) & 600.00 (10) & 600.00 (10) & 441.06 ~~(6) \\
                41 (10) & 600.00 (10) & 600.00 (10) & 421.69 (6) & 600.00 (10) & 600.00 (10) & 529.80 ~~(8) \\
                44 (10) & 600.00 (10) & 600.00 (10) & 470.61 (7) & 600.00 (10) & 600.00 (10) & 553.19 ~~(9) \\
                47 (10) & 600.00 (10) & 600.00 (10) & 485.60 (7) & 600.00 (10) & 600.00 (10) & 529.00 ~~(8) \\
                50 (10) & 600.00 (10) & 600.00 (10) & 485.07 (7) & 600.00 (10) & 600.00 (10) & 526.66 ~~(8) \\
                \hline
                \hline
        \end{tabular}
        \caption{\revselFour experiments: variable number of papers/reviewers, multiple answer sets, randomized.}
        \label{tab:revsel4e6}
\end{table}

\subsubsection{Summary}

The results demonstrate a clear improvement using the new framework;
they can often be further improved
by fine-tuning the grounding and solving algorithm, and by customizing the default heuristics
of the framework, as discussed by~\citeNBYYB{efkrs2014-jair}{eite-etal-14a},
and \citeN{r2014-phd}.
However, already the default settings yield results that are significantly better than
using the previous evaluation method or using no framework at all;
note that the latter requires an advanced grounding algorithm as by~\citeN{eite-etal-14a}, which
was not available at the time the initial evaluation approach as by~\citeN{rs2006} was developed.

In conclusion, the evaluation framework in
Section~\ref{sec:decomposition} pushes \hex-programs
towards scalability for realistic instance sizes, which previous
evaluation techniques missed.

\section{Related Work and Discussion}
\label{sec:relatedanddiscussion}

We now discuss our results in the context of related work, and will
address possible optimizations.

\subsection{Related Work}
\label{sec:related-work}

\reva{
We first discuss related approaches for integrating
external reasoning into ASP formalisms,
then we discuss work related to the notion of
rule dependencies that we introduced
in Section~\ref{sec:ruledependencies},
we discuss related notions of modularity and program decomposition.
Finally we relate our splitting theorems
to other splitting theorems in the literature.
}

\subsubsection{External Sources}

The \dlvex{} system~\cite{cci2007} was a pioneering work on value
invention through external atoms in ASP. It supported VI-restricted
programs, which amount to \hex-programs under extensional semantics without
higher-order atoms and a strong safety condition that is subsumed by
lde-safety. Answer set computation followed the traditional approach
on top of \dlv{}, but used a special progressive grounding
method (thus an experimental comparison to solving, i.e., model building as
in the focus of this paper, is inappropriate).

\reva{
With respect to constraint theories and ASP, several works exist.
The \acsolver\ system \cite{Mellarkod2008},
the \ezcsp\ system \cite{Balduccini2009},
and the \clingcon{} system~\cite{os2012}
divide the program into ASP-literals and constraint-literals,
which} can be seen as a special case of
\hex-programs that focuses on a particular external source.  As for
evaluation, an important difference to general external sources is that
constraint atoms do not use value invention. The modularity techniques
from above are less relevant for this setting as grounding the overall
program in one shot is possible.
However, this also fits into our framework as disabling decomposition
in fact corresponds to a dedicated (trivial) heuristics which keeps the whole
program as a single unit.
\reva{
For a detailed comparison between \acsolver, \ezcsp, and \clingcon\
see \cite{Lierler2014}.

\citeN{Balduccini2013}
also experimentally compared \ezcsp and \clingcon
while varying the degree of integration
between the constraint solver and the ASP solver backend.
Their \quo{black-box integration} corresponds with \dlvhex's
integration of the \dlv backend:
external atom semantics are verified
by plugins callbacks only when a full
answer set candidates has been found in the backend;
moreover their \quo{clear-box integration} corresponds with \dlvhex's
integration of the \clasp solver backend:
plugin callbacks are part of the CDCL propagation
and can operate on partial answer set candidates.
Note that constraint answer set programs can be realized
as a \dlvhex-plugin
(such an effort is currently ongoing).%
\footnote{\url{http://github.com/hexhex/caspplugin}}
}

We also remark that \gringo{} and \clasp{} use a concept called
\quo{external atoms} for realizing various applications such as constraint
ASP solving as in \clingcon{} and incremental solving~\cite{gkks2014}.
However, despite their name they are different from external atoms in
\hex-programs.
In the former case, external atoms are excluded from grounding-time
optimization such that these atoms are not eliminated
even if their truth value is deterministically false during grounding.
This allows to add rules that found truth of such atoms in later incremental grounding steps.
In case of \hex{} the truth value is determined by external sources.
Moreover \gringo{} contains an interface for Lua and Python functions
that can perform computations during grounding.
\hex{} external atoms are more expressive:
they cannot always be evaluated during grounding
because their semantics is defined with respect to the answer set.

\subsubsection{Rule Dependencies}
In the context of answer set programming, dependency graphs over rules
have been used earlier, \reva{e.g.,}~by~\citeN{Linke2001} and~\citeN{2004_suitable_graphs_for_answer_set_programming}. However,
these works consider only ordinary ground programs, and furthermore the
graphs are used for characterizing and computing the answer sets of a
program from these graphs. In contrast, we consider nonground programs
with and external atoms, and we use the graph to split the program into
evaluation units with the goal of modularly computing answer sets.

\subsubsection{Modularity}
Our work is naturally related to work on program modularity under stable model
semantics, as targeted by splitting sets~\cite{lifs-turn-94}
and descendants, with
the work by~\citeN{2008_achieving_compositionality_of_the_stable_model_semantics_for_smodels_programs} and~\citeN{jotw2009-jair}
a prominent representative that
lifted
them to modular programs with choice rules and disjunctive rules, by
considering \quo{symmetric splitting}. Other works, \reva{e.g.,}\
by~\citeN{DBLP:conf/aaai/LierlerT13} go further to {\em define}\/ semantics
of systems of program modules, departing from a mere
semantics-preserving decomposition of a
larger program into smaller parts, or consider multi-language systems that combine modules in
possibly different formalisms on equal terms (cf.~e.g.\
\citeN{DBLP:conf/lpnmr/JarvisaloOJN09} and~\citeN{DBLP:conf/frocos/TasharrofiT11}).

Comparing
the works by~\citeN{2008_achieving_compositionality_of_the_stable_model_semantics_for_smodels_programs} and~\citeN{jotw2009-jair}
as, from a semantic decomposition perspective, the closest in this group  to
ours, an important difference is that our approach works for non-ground
programs and explicitly considers possible  overlaps of modules.
It is tailored to efficient evaluation of arbitrary programs, rather
than to facilitate module-style logic programming with declarative
specifications, or to provide compositional semantics for modules
beyond uni-lateral evaluation, as done
by~\citeN{DBLP:conf/lpnmr/JarvisaloOJN09}
and~\citeN{DBLP:conf/frocos/TasharrofiT11}; for them,
introducing values outside the module domain (known as {\em
value invention}) does not play a visible role.
In this regard, it is in line with previous \hex-program evaluation~\cite{eiter-etal-06}
and decomposition techniques to ground ordinary
programs efficiently \cite{cali-etal-08}.

\subsubsection{Splitting Theorems}
\label{secSplittingTheorems}
Our new splitting theorems compare to related splitting theorems as
follows.

Theorem~\ref{thm:hexsplitting} is similar to Theorem 4.6.2 by~\citeN{rs2006};
however, we do not use splitting sets on atoms, but splitting
sets on rules.  Furthermore, \citeN{rs2006} has no analog to
Theorem~\ref{thm:hexgensplitting}.

The seminal Splitting Set Theorem by~\citeN{lifs-turn-94}
divides the interpretation of $P$ into disjoint sets $X$ and $Y$,
where $X$ is an answer set of the \quo{bottom} $\mi{gb}_A(P) \subseteq P$
and $Y$ is an answer set of a \quo{residual} program
obtained from $P \setminus \mi{gb}_A(P)$ and $X$.
In the residual program, all references to atoms in $X$ are removed,
in a way that
it semantically behaves as if facts $X$
were added to $P \setminus \mi{gb}_A(P)$,
while the answer sets of the residual do not contain any atom in $X$.
This works nicely for answer set programs,
but it is problematic when applied to \hex programs,
because external atoms may depend on the bottom and on atoms in heads of the residual program;
hence, they cannot be eliminated from rule bodies.
The only way to eliminate bottom facts from the residual program would
be to \quo{split} external atoms semantically into a part depending on the
bottom and the program remainder, and by replacing external atoms in rules
with external atoms that have been partially evaluated wrt.\ a bottom
answer set.
Technically, this requires to introduce new external atoms, and
formulating a splitting theorem for \hex programs with two disjoint
interpretations $X$ and $Y$ is not straightforward. Furthermore, such
external atom splitting and partial evaluation might not be possible
in a concrete application scenario.
Different from the two splitting theorems recalled above,
the Global Splitting Theorem by~\citeN{eiter-etal-06}
does not split an interpretation of the program $P$
into disjoint interpretations $X$ and $Y$, and thus should be compared
to our Theorem~\ref{thm:hexgensplitting}.
However, the Global Splitting Theorem
does not allow constraint sharing, and it involves
a residual program which specifies how external atoms are evaluated via
\quo{replacement atoms},
which lead to extra facts $D$ in the residual program that
must be removed from its answer sets.
Both the specification of replacement atoms and the extra facts
make the Global Splitting Theorem
cumbersome to work with when proving correctness of \hex encodings.
Moreover, the replacement atoms are geared towards
a certain implementation technique which however is not mandatory
and can be avoided.

Lemma~5.1 by~\citeN{1997_disjunctive_datalog}
is structurally similar to our Theorem~\ref{thm:hexgensplitting}: answer sets of the bottom program
are evaluated together with the program depending on the bottom (here called the residual),
hence answer sets of the residual are answer sets of the original
program. However, the result was based on atom dependencies and did
neither consider negation nor external atoms.

\reva{%
In sumary our new Generalized Splitting Theorem
has the following advantages.
\begin{itemize}
\item By moving from atom to rule splitting sets, no separate definition
 of the bottom is needed, which just becomes the (rule) splitting set.

\item As regards \hex-programs, splitting is simple (and not troubled)
if all atoms that are true in an answer set of the bottom also appear
in the residual program. Typically, this is not the case in results
from the literature.

\item Finally, also the residual program itself is simpler (and easier
  to construct), by just dropping rules and adding facts. No rule
  rewriting needs to be done, and no extra facts need to be introduced in
  the residual program nor in the bottom.
\end{itemize}
The only (negligible) disadvantage of the new theorems is that the answer
sets of the bottom and the residual program may no longer be disjoint;
however, each residual answer set includes some (unique) bottom answer
set.
}

\subsection{Possible Optimizations}

Evaluation graphs naturally encode parallel
evaluation plans.  We have not yet investigated the potential benefits
of this feature in practice, but this property allows us to do parallel
solving based on solver software that does not have parallel computing
capabilities itself (\quo{parallelize from outside}).
This applies both to programs with external atoms,
as well as to ordinary ASP programs (i.w., without external atoms).
Improving reasoning performance by decomposition has been investigated by~%
\citeN{2005_partition_based_logical_reasoning_for_fo_and_propositional_theories},
however, only wrt.\ monotonic logics.

Improving \hex evaluation efficiency by using knowledge about domain restrictions
of external atoms has been discussed by~\citeN{efk2009-ijcai}.
These rewriting methods yield partially grounded sets of rules
which can easily be distributed into distinct evaluation units by an optimizer.
This directly provides efficiency gains as described in the above work.

As a last remark on possible optimizations, we observe that the data
flow
between evaluation units
can be optimized using proper notions of model projection, such as
in \cite{gks2009-cpaior}.  Model projections would tailor input data of
evaluation units to necessary parts of intermediate answer sets;
however, given that different units might need different parts of the same
intermediate input answer set, a space-saving efficient projection
technique is not straightforward.

\section{Conclusion}
\label{sec:conclusion}

\hex-programs
extend answer set programs with access to external sources through an
API-style interface, which has been fruitfully deployed to various
applications. Providing efficient evaluation methods for such programs
is a challenging but important endeavor, in order to enhance the
practicality of the approach and to make it eligible for a broader range
of applications.
In this direction, we have presented in this article a novel evaluation
method for \hex-programs based on modular decomposition.  We have
presented new results for the latter using special splitting sets, which
are more general than previous results and use rule sets as a basis for
splitting rather than sets of atoms as in previous
approaches. Furthermore, we have presented an evaluation framework which
employs besides a traditional evaluation graph that consists of program
components and reflects syntactic dependencies among them, also a model
graph whose nodes collect answer sets that are combined and passed on
between components. Using decomposition techniques, evaluation units can
be dynamically formed and evaluated in the framework using different
heuristics, Moreover, the answer sets of the overall program can be
produced in a streaming fashion.  The new approach leads in combination
with other techniques to significant improvements for a variety of
applications, as demonstrated
by \citeNBYYB{eite-etal-14a}{efkrs2014-jair} and \citeN{r2014-phd}.
Notably, while our results target \hex-programs,
the underlying concepts and techniques are not limited to them (e.g., to
separate the evaluation and the model graph) and may be fruitfully
transferred to other rule-based formalisms.

\subsection{Outlook}

The work we presented can be continued in different directions. As for
the prototype reasoner, a rather straightforward extension is to support
brave and cautious reasoning on top of \hex programs, while
incorporating constructs like aggregates or preference constraints
requires more care and efforts. Regarding program evaluation, our
general evaluation framework provides a basis for further optimizations
and evaluation strategies.
Indeed, the generic notions of evaluation unit, evaluation graph and
model graph allow to specialize and improve our framework in different
respects. First, evaluation units (which may contain duplicated
constraints), can be chosen according to a proper estimate of the number
of answer sets (the fewer, the better); second, evaluation plans can be
chosen by ad-hoc optimization modules, which may give preference to (a
combination of) time, space, or parallelization conditions. Third,
the framework is amenable to a form of coarse-grained distributed
computation at the level of evaluation units (in the style of
\citeN{perr-etal-10}).%

While modular evaluation is advantageous in many applications,
it can also be counterproductive, as currently the propagation of
knowledge learned by conflict-driven techniques
into different evaluation units is not
possible. In such cases, evaluating the program as a single evaluation
unit is often also infeasible due to the properties of the grounding
algorithm, as observed by~\citeN{eite-etal-14a}.  Thus, another starting
point for future work is a tighter integration of the solver instances
used to evaluate different units, e.g., by exchanging learned knowledge.
In this context, also the interplay of the grounder and the solver is an
important topic.

\section*{Acknowledgements}

We would like to thank the anonymous reviewers
and Michael Gelfond for their constructive feedback.

\newcommand\myhexevalappendixproofs{%
\section{Proofs}
\label{sec:proofs}
  \mylocatedproof{%
    \begin{proof}[Proof of Theorem~\ref{thm:hexsplitting} (Splitting Theorem)]
      \myproofhexsplitting
    \end{proof}
  }
  \mylocatedproof{%
    \begin{proof}[Proof of Theorem~\ref{thm:hexgensplitting} (Generalized Splitting Theorem)]
      \myproofhexgensplitting
    \end{proof}
  }
  \mylocatedproof{%
    \begin{proof}[Proof of Proposition~\ref{thm:evalgraphruledeps}]
      \myproofEvalGraphRuleDeps
    \end{proof}
  }
  \mylocatedproof{%
    \begin{proof}[Proof of Proposition~\ref{thm:disjointunitoutputs}]
      \myproofDisjointOutputModels
    \end{proof}
  }
  \mylocatedproof{%
    \begin{proof}[Proof of Proposition~\ref{thm:evalgraphcoversdomainexpansionsafe}]
      \myproofEvalGraphCoversDExpansionSafe
    \end{proof}
  }
  \mylocatedproof{%
    \begin{proof}[Proof of Theorem~\ref{thm:evalgraphprecbottom}]
      \myproofevalgraphprecbottom
    \end{proof}
  }
  \mylocatedproof{%
    \begin{proof}[Proof of Theorem~\ref{thm:evalgraphpredbottom}]
      \myproofevalgraphpredbottom
    \end{proof}
  }
  \mylocatedproof{%
    \begin{proof}[Proof of Proposition~\ref{thm:joinmodelgraphsoundcomplete}]
      \myproofJoinModelGraphSoundComplete
    \end{proof}
  }
  \mylocatedproof{%
    \begin{proof}[Proof of Proposition~\ref{thm:joinanswersetgraphsoundcomplete}]
      \myproofJoinAnswerSetGraphSoundComplete
    \end{proof}
  }
  \mylocatedproof{%
    \begin{proof}[Proof of Theorem~\ref{thm:answersetsfromvanilla}]
      \myproofAnswerSetsFromVanilla
    \end{proof}
  }
  \mylocatedproof{%
    \begin{proof}[Proof of Proposition~\ref{thm:answersetsfromufinal}]
      \myproofAnswerSetsFromUFinal
    \end{proof}
  }
  \mylocatedproof{
    \begin{proof}[Proof of Proposition~\ref{thm:EvaluateLDESafe}]
      \myproofEvaluateLDESafe
    \end{proof}
  }
  \mylocatedproof{
    \begin{proof}[Proof of Theorem~\ref{thm:soundcomplete}]
      \myproofSoundComplete
    \end{proof}
  }
} %

	\bibliographystyle{acmtrans}

\iffinal\inlinereftrue\fi
\ifinlineref
\else
\bibliography{hexeval}
\fi %

\ifextended
\appendix

\myhexevalappendixproofs

\ifshowotherappendix

\section{Example Run of Algorithm~\ref{alg:buildAnswerSets}}
\label{sec:appendix-alg-run}

We provide here an example run of Algorithm~\ref{alg:buildAnswerSets} for
our running example.

\begin{example}[ctd.]
\label{ex:buildanswersets}
Consider an evaluation graph $\cE_2'$
which is $\cE_2$ plus $\ufinal=\emptyset$, which depends on all other units.
Following Algorithm~\ref{alg:buildAnswerSets}
we first choose $u = u_1$,
and as $u_1$ has no predecessor units,
step~\ref{step:emptyinputmodels}
creates the \myiint $m_1$ with $\myint(m_1) = \emptyset$.
As $u_1 \neq \ufinal$, we continue and in loop~\ref{step:secondforloop}
obtain
$O = \AS(u_1) = \big\{
    \{ \mygoinout(\myindoor) \},
    \{ \mygoinout(\myoutdoor) \} \big\}$.
We add both answer sets as \myoints $m_2$ and $m_3$
and then finish the outer loop with $U = \{ u_2, u_3, u_4, \ufinal \}$.
In the next iteration, we could choose $u = u_2$ or $u = u_3$;
assume we choose $u_2$. Then $\myinputs(u_2) = \{ u_1 \}$ and $k = 1$,
and we enter the loop~\ref{step:firstforloop}
and build all joins that are possible with \myoints at $u_1$
(all joins are trivial and all are possible),
i.e., we copy the interpretations
and store them at $u_2$ as new \myiints $m_4$ and $m_5$.
In the loop~\ref{step:secondforloop},
we obtain $O = \EvaluateLDESafe(u_2, \{ \mygoinout(\myindoor) \}) = \emptyset$,
as indoor swimming requires money which is excluded by $c_8 \in u_2$.
Therefore \myiint $\{ \mygoinout(\myindoor) \}$ yields no \myoint,
indicated by \Lightning.
However,
we obtain $O = \EvaluateLDESafe(u_2, \{ \mygoinout(\myoutdoor) \}) = \{ \emptyset \}$:
as outdoor swimming neither requires money nor anything else,
\myiint $\{ \mygoinout(\myoutdoor) \}$ derives no additional atoms
and yields the empty answer set, which we store as \myoint $m_6$ at $u_2$;
the iteration ends with $U = \{ u_3, u_4, \ufinal \}$.
In the next iteration we choose $u = u_3$,
we add in loop~\ref{step:firstforloop} \myiints $m_7$ and $m_8$ to $u_3$,
and in loop~\ref{step:secondforloop} \myoints $m_9$, \ldots, $m_{12}$ to $u_3$;
the iteration ends with $U = \{ u_4, \ufinal \}$.
In the next iteration we choose $u = u_4$;
this time we have multiple predecessors,
and in loop~\ref{step:firstforloop}
we check join candidates $m_6 \join m_9$ and $m_6 \join m_{10}$,
which are both not defined.
The other join candidates are $m_6 \join m_{11}$ and $m_6 \join m_{12}$, which
are both defined; we thus add their results as \myiints $m_{13}$ and
$m_{14}$, respectively, to $u_4$.
The loop~\ref{step:secondforloop} computes then
one \myoint $m_{15}$ for \myiint $m_{13}$
and no \myoint for $m_{14}$.
The iteration ends with $U = \{ \ufinal \}$.
In the next iteration, we have
$\myinputs(\ufinal) = \{ u_1, u_2, u_3, u_4 \}$
and the loop~\ref{step:firstforloop}
checks all combinations of one \myoint at each unit in $\myinputs(\ufinal)$.
Only one such join candidate is defined, namely
$m = m_3 \join m_6 \join m_{11} \join m_{15}$,
whose result is stored as a new \myiint at $\ufinal$.
The check~\ref{step:return} now succeeds,
and we return all \myiints at $\ufinal$; i.e., we return
$\{ m \} =
 \big\{\{ \mygoinout(\myoutdoor),
\mygolocation(\mypooln),
\myngolocation(\mypoolg),
\mygosomewhere,
\myneed(\myloc,\myyogamat)\} \big\}$.
This is indeed the set of answer sets of $\myPswim$.
  \qed
\end{example}

\section{On Demand Model Streaming Algorithm}
\label{sec:appendixStreaming}

Algorithm~\ref{alg:buildAnswerSets} fully evaluates all other units
before computing results at the final evaluation unit $\ufinal$, and it
keeps the intermediate results in memory. If we are only interested in
one or a few answer sets, many unused results may be calculated.

Using the same evaluation graph, we can compute the answer sets with a
different, more involved algorithm $\myOnDemandAS$ (shown in
Algorithm~\ref{MyAlgOnDemandAS}) that operates demand-driven from units,
starting with $\ufinal$, rather than data-driven from completed units.
It uses in turn several building blocks that are shown in
Algorithms~\ref{alg:getNextOModel}--\ref{alg:getNextIModel}

$\myOnDemandAS$ calls Algorithm~$\myGetNextOModel$ for $\ufinal$ and
outputs its output models, i.e., the answer sets of the
input program $P$ given by the evaluation graph $\cE$, one by one
until it gets back $\undef$.  Like Algorithm~\ref{alg:buildAnswerSets},
$\myGetNextOModel$ builds in combination with the other algorithms an
answer set graph $\cA$ for $\cE$ that is input-complete at all units, if
all statements marked with '$(+)$' are included; omitting them, it
builds $\cA$ virtually and has at any time at most one input and one
output model of each unit in memory.

Roughly speaking, the models at units are determined in the same order
in which a right-to-left depth-first-traversal of the evaluation graph
$\cE$ would backtrack from edges. This is because first all models of
the subgraph reachable from a unit $u$ are determined, then models at
the unit $u$, and then the algorithm backtracks. The models of the
subgraph are retrieved with $\myGetNextIModel$ one by one,
and using
$\myNextAnswerSet$ the output models are generated and returned. The
latter function is assumed to return, given a \hex{}-program $P$ and
the $i$-th element in an arbitrary but fixed enumeration $I_1,
I_2,\ldots,I_m$ of the answer sets of $P$ (without duplicates),
the next answer set $I_{i+1}$, where by convention $I_0=\undef$ and the
return value for $I_m$ is $\undef$. This is easy to provide
on top of current solvers, and the incremental usage of
$\myNextAnswerSet$ allows for an efficient stateful realization (e.g.\ answer
set computation is suspended).

\myfigureAlgoOnDemandAS{tp}
\myfigureAlgoGetNextOModel{tp}

The trickiest part of this approach is $\myGetNextIModel$, which has to
create locally and in an incremental fashion all joins that are globally
defined, i.e., all combinations of incrementally available output models
of predecessors which share a common predecessor model at all
\myCAUstext. To generate all combinations of output models in the right
order, it uses the algorithm $\myEnsureModelIncrement$.

\myfigureAlgoEnsureModelIncrement{tp}

The algorithms operate on a global data structure
$\cS=(\cE,\cA,\mycuri,\mycuro,\myrefcounto)$ called \emph{storage},
where
\begin{compactitem}
\item $\cE=(U,E)$ is the evaluation graph
containing $\ufinal \in U$,
\item
$\cA=(M,F,\myunit,\mytype,\myint)$ is the (virtually built) answer set graph,
\item $\mycuri : U \to M \cup \{ \undef \}$ and $\mycuro : U \to M \cup
\{ \undef \}$, are functions that informally associate with a unit $u$
  the current input respectively output model considered,  and
\item $\myrefcounto : U \to \bbN \cup \{ 0 \}$ is a function that keeps
  track of how many current input models point to the current output
  model of $u$; this is used to ensure correct joins, by checking
  in $\myGetNextOModel$ that the condition (IG-F) for sharing models in the interpretation graph is not violated (for details see Section~\ref{sec:fais} and Definition~\ref{def:interpretationgraph}).
\end{compactitem}

Initially, the storage $\cS$ is empty, i.e., it contains the input
evaluation graph $\cE$, an empty answer set graph $\cA$, and the
functions are set to $\mycuri(u) = \undef$, $\mycuro(u) = \undef$, and
$\myrefcounto(u) = 0$ for all $u \in U$.  The call of $\myGetNextOModel$
for $\ufinal$ triggers the right-to-left depth-first  traversal of the
evaluation graph.

\myfigureAlgoGetNextIModel{tp}

We omit tracing Algorithm~$\myOnDemandAS$ on our running example, as
this would take quite some space; however, one can check that given the
evaluation graph $\cE_2$, it correctly outputs the single answer set
  \begin{align*}
  I = \{
    \mygoinout(\myoutdoor),
    \mygolocation(\mypooln),
    \myngolocation(\mypoolg),
    \mygosomewhere,
    \myneed(\myloc,\myyogamat)\}.
  \end{align*}
Formally, it can be shown that given an evaluation graph $\cE=(U,E)$
of a program $P$ such that $\cE$ contains a final unit $\ufinal =
\emptyset$, Algorithm~$\myOnDemandAS$ outputs one by one all answer
sets of $P$, without duplicates, and that in the version without
$(+)$-lines, it stores at most one input and one output model per unit
(hence the size of the used storage is linear in the size of the
ground program $\grnd(P)$).

\section{Overview of Liberal Domain-Expansion Safety}
\label{sec:liberalsafety}

Strong domain-expansion safety is overly restrictive,
as it also excludes programs
that
clearly \emph{are}\/ finitely restrictable.
In this section we give an overview about the notion
and refer to \cite{eite-etal-14a} for details.

\begin{example}
\label{ex:DESNotSS}
Consider the following program:
\begin{equation*}
P {=} \left\{
\begin{array}{@{\,}l@{\colon}l@{~~~}l@{\colon}l@{~}l@{}}
r_1 & p(a).             & r_3 & s(Y) &\leftarrow p(X), \ext{\mathit{concat}}{X,a}{Y}. \\[1ex]
r_2 & \mathit{q}(aa). & r_4 & p(X) &\leftarrow s(X), \mathit{q}(X).
\end{array}\right\}
\end{equation*}
It is not strongly safe because $Y$ in the cyclic external atom
$\ext{\mathit{concat}}{X,a}{Y}$ in $r_3$ does not occur
in an ordinary body atom that does not depend on $\ext{\mathit{concat}}{X,a}{Y}$.
However, $P$ is finitely restrictable as the cycle is ``broken'' by $\mathit{dom}(X)$
in $r_4$.
\qedhere
\end{example}
To overcome unnecessary restrictions
of strong safety in \cite{eiter-etal-06}, \emph{liberal domain-expansion safety} (lde-safety) has
been introduced~\cite{eite-etal-14a}, which
incorporates both syntactic and semantic properties of a program.
The details of the notion are not necessary for this paper, except that
all lde-safe programs have finite groundings with the same answer sets;
we give here a brief overview.

Unlike strong safety, liberal de-safety is not a property of
entire atoms
but of
\emph{attributes}, i.e., pairs of predicates and argument positions.
Intuitively, an attribute is lde-safe, if the number of
different terms in an answer-set preserving grounding (i.e.\ a
grounding which has the same answer sets if restricted to the positive atoms
as the original program) is finite.
A program is lde-safe, if all its attributes are lde-safe.

The notion of lde-safety is designed in an extensible fashion,
i.e., such that several
safety criteria can be easily integrated.
For this we parametrize our definition of lde-safety
by a \emph{term bounding function (TBF)}, which
identifies variables in a rule that are ensured to have only
finitely many instantiations in the answer set preserving grounding.
Finiteness of the overall grounding
follows then from the properties of TBFs.

For an ordinary predicate $p \,{\in}\, \mathcal{P}$,
let $\attr{p}{}{i}$ be the \emph{$i$-th attribute of $p$} for all $1 \le i \le \mathit{ar}(p)$.
For an external predicate $\amp{g} \in \mathcal{X}$ with input list $\vec{X}$ in rule $r$,
let $\attr{\amp{g}[\vec{X}]_r}{T}{i}$ with $T \in \{\ipar, \opar\}$ be the \emph{$i$-th input resp. output attribute of $\amp{g}[\vec{X}]$ in $r$} for all $1 \le i \le \mathit{ar}_T(\amp{g})$.
For a ground program $P$, the \emph{range} of an attribute is, intuitively, the set of ground terms which occur in the position of the attribute.
Formally, for an attribute $\attr{p}{}{i}$ we have
$\mathit{range}(\attr{p}{}{i}, P) = \{ t_i \mid p(t_1, \ldots, t_{\mathit{ar}(p)}) \in A(P) \}$;
for an attribute $\attr{\amp{g}[\vec{X}]_r}{T}{i}$ we have
$\mathit{range}(\attr{\amp{g}[\vec{X}]_r}{T}{i}, P) = \{ x^T_i \mid
\amp{g}[\vec{x}^\ipar](\vec{x}^\opar) \in \mathit{EA}(P) \}$,
where $\vec{x}^s = x^{s}_1, \ldots, x^{s}_{\mathit{ar}_{s}(\amp{g})}$.

We use the following
monotone operator to compute
by fixpoint iteration a finite subset of $\mathit{grnd}(P)$ for a program
$P$:
\begin{equation*}
G_{P}(P') = \bigcup_{r \in P} \{ r\theta \mid \exists I \subseteq \mathcal{A}(P'), I \not\models \bot, I \models B^{+}(r\theta) \},
\end{equation*}
\noindent where
$\mathcal{A}(P') = \{ \T a, \F a \mid a \in A(P') \}
\setminus \{ \F a \mid a \leftarrow . \in P \}$ and
$r\theta$ is the ground instance of $r$ under variable substitution $\theta\colon\mathcal{V} \to \mathcal{C}$.
Note that in this definition, $I$ might be partial, but by convention we assume that all atoms which are not explicitly assigned to true
are false.
That is, $G_{P}$ takes a ground program $P'$ as input
and returns all rules from $\mathit{grnd}(P)$ whose positive
body is satisfied under some assignment over the atoms of $\Pi'$.
Intuitively, the operator iteratively extends the grounding by new rules
if they are possibly relevant for the evaluation, where relevance
is in terms of satisfaction of the positive rule body under some assignment
constructable over the atoms which are possibly derivable so far.
Obviously, the least fixpoint $G_{P}^{\infty}(\emptyset)$ of this operator
is a subset of $\mathit{grnd}(P)$;
we will show that it is finite if $P$ is lde-safe
according to our new notion.
Moreover, we will show that this grounding preserves all answer sets
as all omitted rule instances have unsatisfied bodies anyway.

\begin{example}
\label{ex:groundingoperator}
Consider the following program $P$:
\begin{equation*}
\begin{array}{r@{~}r@{~}l}
r_1\colon & s(a). & \quad r_2\colon \mathit{dom}(\mathit{ax}).\quad r_3\colon\ \mathit{dom}(\mathit{axx}). \\
r_4\colon & s(Y)  & \leftarrow  s(X), \ext{\mathit{concat}}{X,x}{Y}, \mathit{dom}(Y).
\end{array}
\end{equation*}
The least fixpoint of $G_{P}$ is the following ground program:
\begin{equation*}
\begin{array}{r@{}r@{~}l}
r_1'\colon\ & s(a).      & \quad r_2'\colon\ \mathit{dom}(\mathit{ax}).\quad r_3'\colon\ \mathit{dom}(\mathit{axx}). \\
r_4'\colon\ & s(\mathit{ax}) &\leftarrow s(a), \ext{\mathit{concat}}{a,x}{ax}, \mathit{dom}(ax). \\
r_5'\colon\ & s(\mathit{axx}) &\leftarrow s(ax), \ext{\mathit{concat}}{ax,x}{axx}, \mathit{dom}(axx).
\end{array}
\end{equation*}
Rule $r_4'$ is added in the first iteration and rule $r_5'$ in the second.
\end{example}

Towards a definition of lde-safety,
we say that a term in a rule is \emph{bounded},
if the number of substitutions in $G_{P}^{\infty}(\emptyset)$ for this term is finite.
This is abstractly formalized using \emph{term bounding
          functions}.

\begin{definition}[Term Bounding Function (TBF)]
\label{def:termboundingfunction}
A \emph{term bounding function}, denoted $b(P, r, S, B)$,
maps a program $P$, a rule~$r \in P$, a set $S$ of (already
safe) attributes,
and a set $B$ of (already bounded) terms in $r$
to an enlarged set of (bounded) terms $b(P, r, S, B) \supseteq B$,
such that every~$t \in b(P, r, S, B)$
has finitely many substitutions in $G_{P}^{\infty}(\emptyset)$ if
\begin{inparaenum}[(i)]
\item the attributes $S$ have a finite range in $G_{P}^{\infty}(\emptyset)$
and
\item each term in~$\mathit{terms}(r) \cap B$ has finitely many
substitutions in~$G_{P}^{\infty}(\emptyset)$.
\end{inparaenum}
\end{definition}

Intuitively, a TBF receives a set of already bounded terms and
a set of attributes that are already known to be lde-safe. Taking the program
into account, the TBF then identifies and returns
further terms which are also bounded.

The concept yields lde-safety of attributes and programs
from the boundedness of variables according to a TBF.
We provide a mutually inductive definition that takes
the empty set of lde-safe attributes $S_0(P)$ as
its basis. Then, each iteration step~$n \ge 1$ defines first
the set of bounded terms $B_{n}(r, P, b)$ for all rules~$r$,
and then an enlarged set of lde-safe attributes $S_n(P)$.
The set of lde-safe attributes in step $n+1$ thus depends on the TBF,
which in turn depends on the domain-expansion safe attributes from step $n$.
\begin{definition}[Liberal Domain-Expansion Safety]
\label{def:domainexpansionsafety}
Let $b$ be a term bounding function. The set $B_{n}(r, P, b)$ of \emph{bounded terms} in a rule $r \in P$
in step $n \ge 1$ is
$B_{n}(r, P, b) = \bigcup_{j \ge 0} B_{n,j}(r, P, b)$
where $B_{n,0}(r, P, b) = \emptyset$ and for all $j \ge 0$, $B_{n,j+1}(r, P, b) = b(P, r, S_{n-1}(P), B_{n,j})$.

\smallskip
The set of \emph{domain-expansion safe attributes} $S_{\infty}(P) = \bigcup_{i \ge 0} S_i(P)$
of a program $P$
is iteratively constructed with
$S_0(P) = \emptyset$
and for $n \ge 0$:
\begin{itemize}
\item $\attr{p}{}{i}{\,\in}\, S_{n+1}(P)$ if for each $r \,{\in}\, P$
and atom $p(t_1, \ldots,$ $t_{\mathit{ar}(p)}) \in H(r)$,
we have that term $t_i \in B_{n+1}(r, P, b)$, i.e., $t_i$ is \emph{bounded};
\item $\attr{\amp{g}[\vec{X}]_r}{\ipar}{i} \,{\in}\, S_{n+1}(P)$ if
each $\vec{X}_i$ is a \emph{bounded} variable,
or $\vec{X}_i$ is a predicate input parameter $p$ and
$\attr{p}{}{1}, \ldots, \attr{p}{}{\mathit{ar}(p)} \in S_n(P)$;
\item $\attr{\amp{g}[\vec{X}]_r}{\opar}{i} \,{\in}\, S_{n+1}(P)$
  if and only if
$r$ contains an external atom $\amp{g}[\vec{X}](\vec{Y})$ such that $\vec{Y}_i$ is bounded,
or
$\attr{\amp{g}[\vec{X}]_r}{\ipar}{1}, \ldots, \attr{\amp{g}[\vec{X}]_r}{\ipar}{\mathit{ar}_{\ipar}(\amp{g})} \in S_n(P)$.
\end{itemize}

\smallskip
A program $P$ is \emph{liberally domain-expansion (lde)
          safe}, if it is safe and all its attributes are domain-expansion safe.
\end{definition}

A detailed description of liberal safety is beyond the scope of this paper.
However, it is crucial that each liberally domain-expansion safe \hex-program $P$ is finitely restrictable,
i.e., there is a finite subset $P_g$ of $grnd_{C}(P)$ s.t.~$\mathcal{AS}(P_g) = \mathcal{AS}(\grnd_{C}(P))$.
A concrete grounding algorithm
\GroundLiberallyDomainExpansionSafeProgram{} is given in~\cite{eite-etal-14a};
we use $\GroundLiberallyDomainExpansionSafeProgram(P)$ in this article to refer to a finite grounding of $P$ that has the same answer sets.

\fi %

\fi %

\end{document}